\documentclass[12pt, a4paper]{article}
\usepackage[T1]{fontenc}
\usepackage{tgbonum}
\usepackage{times}
\usepackage{amsmath, amssymb, amsfonts}
\usepackage{graphicx}
\usepackage{epstopdf}
\usepackage{caption, subcaption}
\usepackage[margin=01 in]{geometry}
\usepackage[onehalfspacing]{setspace}
\usepackage{hyperref}
\usepackage{xr-hyper}
\usepackage[affil-it]{authblk}
\numberwithin{equation}{section}
\usepackage{amsthm}
\newtheorem{theorem}{Theorem}[subsection]
\newtheorem{lemma}{Lemma}[subsection]

\newtheorem{remark}{Remark}[subsection]
\newtheorem{example}{Example}[subsection]
\newtheorem{assumption}{Assumption}[subsection]
\newtheorem{definition}{Definition}[subsection]
\usepackage{authblk}
\usepackage{url}
\usepackage{cite}
\usepackage[table,xcdraw]{xcolor}
\usepackage{lineno} 
\usepackage{float}
\usepackage{tcolorbox}
\usepackage{algorithm}
\usepackage{algorithmic}
\usepackage{placeins}
\usepackage{tikz}
\usepackage{multirow}
\usepackage{array}
\newcolumntype{C}[1]{>{\centering\let\newline\\\arraybackslash\hspace{0pt}}m{#1}}
\usepackage{siunitx}

\newenvironment{keywords}{%
  \par\noindent\textbf{\textit{Keywords---}}\ }{\par}

\newenvironment{AMS}{%
  \par\noindent\textbf{\textit{AMS subject classifications---}}\ }{\par}

\title{Turning mechanistic models into forecasters by using machine learning} 

\author[1]{Amit K. Chakraborty} 
\author[1,*]{Hao Wang}
\author[2]{Pouria Ramazi}
\affil[1]{Department of Mathematical and Statistical Sciences \& Interdisciplinary Lab for Mathematical Ecology and Epidemiology (ILMEE), University of Alberta, AB, Canada}
\affil[2]{Department of Biological Sciences \& Department of Mathematics and Statistics, University of Calgary, AB, Canada}
\affil[*]{Corresponding author: hao8@ualberta.ca}
\date{}
\date{}

\begin{document}

% \linenumbers
\maketitle

\begin{abstract}
The equations of complex dynamical systems may not be identified by expert knowledge, especially if the underlying mechanisms are unknown. Data-driven discovery methods address this challenge by inferring governing equations from time--series data using a library of functions constructed from the measured variables. However, these methods typically assume time-invariant coefficients, which limits their ability to capture evolving system dynamics. To overcome this limitation, we allow some of the parameters to vary over time, learn their temporal evolution directly from data, and infer a system of equations that incorporates both constant and time-varying parameters. We then transform this framework into a forecasting model by predicting the time-varying parameters and substituting these predictions into the learned equations. The model is validated using datasets for Susceptible--Infected--Recovered, Consumer--Resource, greenhouse gas concentration, and Cyanobacteria cell count. By dynamically adapting to temporal shifts, our proposed model achieved a mean absolute error below $3\%$ for learning a time series and below $6\%$ for forecasting up to a month ahead. We additionally compare forecasting performance against CNN--LSTM and Gradient Boosting Machine (GBM), and show that our model outperforms these methods across most datasets. Our findings demonstrate that integrating time-varying parameters into data-driven discovery of differential equations improves both modeling accuracy and forecasting performance.
\end{abstract} 

% \keywords{Data-driven learning, Time-varying parameters, SINDy, CNN-LSTM, GBM, Forecasting}

% % REQUIRED

% % REQUIRED

\section{Introduction}

Mechanistic models provide the causal and interpretable governing laws of a system, but many natural systems are too complex to be captured by expert knowledge \cite{baker2018mechanistic, lan2025shallow, trappe2023density}. For example, emissions of greenhouse gases, such as methane and carbon dioxide, are influenced by temperature, humidity, wind speed, and other weather variables \cite{small2015emissions, natchimuthu2014influence}. In epidemiology, disease incidence is shaped by environmental factors, population density, social behavior, and policy interventions \cite{lowen2007influenza, shaman2009absolute, chakraborty2024policy, wang2022policy}. The relationships among the variables in such a high-dimensional system are many, time variant, and unknown, making it nearly infeasible to be described by equations that are based on expert knowledge \cite{banga2025mechanistic}. 

A promising alternative is data-driven discovery of the governing equation, often represented as differential equations \cite{song2024towards, brunton2016discovering, rudy2017data, li2019discovering}. Unlike mechanistic modeling that relies on predefined assumptions, data-driven approaches adapt to the intrinsic structure of the system, capturing both linear and nonlinear interactions among the system's variables \cite{rudy2017data}.

A widely used technique for the data-driven discovery of governing equations is the Sparse Identification of Nonlinear Dynamics (SINDy) \cite{brunton2016discovering}. SINDy and its variants have been used to uncover dynamical systems of ordinary and partial differential equations (ODEs and PDEs) in fluid dynamics \cite{brunton2016discovering, fukami2021sparse, callaham2022role, callaham2022empirical, loiseau2020data, loiseau2018constrained}, biological systems \cite{prokop2024biological, mangan2016inferring, kaiser2018sparse, wu2024data, prokop2023data, sandoz2023sindy, prabhu2023data}, mechanical systems \cite{thiele2020system, wulff2024minimal, lee2024energy}, and climate and environmental modeling \cite{guo2024uncertainty, rubio2022modeling, yang2024atmospheric}. It uses sparse regression to identify the minimal set of candidate functions needed to describe the system's dynamics \cite{brunton2016discovering}. Originally developed for systems with constant coefficients, SINDy has since been extended to handle time-varying parameters \cite{li2019discovering, rudy2019data}. Such extensions are needed for systems with evolving dynamics, where temporal changes cannot be accurately represented by fixed parameters \cite{chakraborty2024policy}. For example, the transmission rate and the death rate in an epidemic model are time-varying and can be estimated by calibrating the model to epidemic data \cite{wang2022policy, chakraborty2024policy, wang2023discrete}.

However, for systems with many variables, it remains unclear how many parameters exactly need to be time-varying to capture the data accurately. Identifying only the necessary time-varying parameters can reduce computational cost, prevent overfitting, and improve interpretability.

Another challenge is to use the resulting system of equations equipped with time-varying parameters in forecasting. Finding the time-series of a time-varying parameter (basically a variable) over past data can be done by model fitting. However, forecasting requires running the model with an initial condition, and this requires the time-series of the parameter in the future \cite{chakraborty2024policy}. This is shaped by external forces, interactions among variables, and other dynamic factors \cite{chakraborty2024policy, wang2022policy}. To address this challenge, machine learning (ML) models were employed to forecast time-varying parameters across various fields using their historical time-series and known covariates \cite{wang2023discrete}. In epidemiology, the disease transmission and death rates were forecasted based on policies or weather using epidemic models \cite{bousquet2022deep, long2021identification, wang2022policy, chakraborty2024policy, ji2025hybrid}. In climate science, greenhouse gas drivers were studied by forecasting time-varying parameters \cite{pastpipatkul2024analysis}. Therefore, once a system is modeled with time-varying parameters, they can be forecasted via an ML model using their past values and covariates whose future values are known \cite{wang2023discrete, chakraborty2024early}.

The goal of this study is to develop a data-driven framework for discovering ODEs with fixed and time-varying parameters that capture evolving dynamical systems. We implement this framework using a two-stage application of the SINDy algorithm. In the first stage, SINDy is applied to the full time series to identify the variables that actively contribute to the dynamics and their interactions under constant parameter assumptions. In the second stage, the analysis is repeated on shorter temporal intervals, allowing a subset of parameters--selected based on validation performance--to vary over time, while the remaining parameters are held fixed. This procedure yields an ODE model that preserves the underlying mechanistic structure while accommodating temporal variability in key processes. To enable forecasting, the identified time-varying parameters are predicted using a machine learning model driven by external covariates that influence the time-varying parameters. These predicted parameters are then substituted into the learned ODEs, allowing to run the model and make forecasts that adapt dynamically to changing conditions.

We applied this methodology to two well-known mathematical models--the Susceptible--Infected-- \allowbreak Recovered (SIR) model \cite{kermack1927contribution} and the Consumer--Resource (CR) model \cite{rosenzweig1963graphical}--as well as two real-world datasets: greenhouse gas concentration dynamics in Alberta's oil sands tailings ponds and cyanobacteria cell counts in lakes across Alberta. We predicted the time-varying parameters using a Random Forest (RF) model informed by weather-related drivers, including air temperature, relative humidity, precipitation, and wind speed. To benchmark the forecasting performance of our proposed approach, we compared it against two widely used data-driven baselines: a hybrid Convolutional Neural Network-Long Short-Term Memory (CNN-LSTM) model and a Gradient Boosting Machine (GBM).

\section{Methods} \label{Methods}
\subsection{Problem formulation} \label{sparse-regression-problem}

Consider the evolution of a state vector \( \mathbf{x}(t) \in \mathbb{R}^n \) for a natural number $n\in\mathbb{N}$, described by the dynamical system
\begin{equation}
    \frac{d\mathbf{x}(t)}{dt} = \mathbf{f}(\mathbf{x}(t), \mathbf{u}(t)),
    \label{dynamical-system-equn}
\end{equation}
where 
\( \mathbf{u}(t) \in \mathbb{R}^q \), $q\in\mathbb{N}$, denotes the input variables and \( \mathbf{f} : \mathbb{R}^n \times \mathbb{R}^q \rightarrow \mathbb{R}^n \) is an unknown function. Measurements of the state variable are available at $m$ time points $\mathbf{t}=[t_1,\ldots,t_m]$, resulting in the time-series data
\begin{equation}
    \mathbf{X} =
    \begin{bmatrix}
        \mathbf{x}(t_1)^\top \\
        \mathbf{x}(t_2)^\top \\
        \vdots \\
        \mathbf{x}(t_m)^\top
    \end{bmatrix},
\end{equation}
where \( m \) is the number of measurements and the resulting matrix has dimension \( m \times n \). 
The corresponding time derivative data is
\begin{equation}
    \dot{\mathbf{X}} =
    \begin{bmatrix}
        \dot{\mathbf{x}}(t_1)^\top \\
        \dot{\mathbf{x}}(t_2)^\top \\
        \vdots \\
        \dot{\mathbf{x}}(t_m)^\top
    \end{bmatrix}.
    \label{data-derivative-equation}
\end{equation}
Similarly, the time-series data of the inputs is
\[
\mathbf{U} =
\begin{bmatrix}
    \mathbf{u}(t_1)^\top \\
    \mathbf{u}(t_2)^\top \\
    \vdots \\
    \mathbf{u}(t_m)^\top
\end{bmatrix}.
\]
\\
The library of candidate functions  to estimate $\mathbf{f}(\cdot)$ is constructed from \( \mathbf{X} \) and \( \mathbf{U} \) as
\begin{equation}
    \boldsymbol{\Theta}(\mathbf{X}, \mathbf{U}) =
    \big[\, \mathbf{1},\; \mathbf{X},\; \mathbf{U},\; 
    (\mathbf{X} \odot \mathbf{U}),\; \mathbf{X}^2,\; \mathbf{U}^2,\; \dots, \; (\mathbf{X}^{d-1} \odot \mathbf{U}), \; (\mathbf{X}^{d-2} \odot \mathbf{U}^2), \; \dots, \; (\mathbf{X} \odot \mathbf{U}^{d-1}), \;
    \mathbf{X}^d,\; \mathbf{U}^d \big],
    \label{candidate-term-equation}
\end{equation}
where $d \in \mathbb{N}$, and the candidate library contains all monomials in \( \mathbf{x (t)} \) and \( \mathbf{u(t)} \) of total degree at most $d$. In addition to polynomial terms, the candidate function library may also include nonlinear bases such as $\sin(\cdot)$, $\cos(\cdot)$, $\tan(\cdot)$, their powers, and other user-defined nonlinear functions.

% \( \mathbf{X}^k \) and \( \mathbf{U}^k \) denote the matrices whose columns consist of all monomials formed from \( \mathbf{x (t)} \) and \( \mathbf{u(t)} \) up to degree \( k \), and \( \mathbf{X} \odot \mathbf{U} \) denotes the elementwise products of the components of \( \mathbf{X} \) and \( \mathbf{U} \) for all polynomial orders satisfying \( 0 \le r + s \le d \). 

Based on equation \eqref{dynamical-system-equn}, the derivatives $\dot{\mathbf{X}}$ and the candidate function library $\boldsymbol{\Theta}(\mathbf{X}, \mathbf{U})$ can be related through
\begin{equation}
    \dot{\mathbf{X}} = \boldsymbol{\Theta}(\mathbf{X}, \mathbf{U}) \, \boldsymbol{\Xi},
    \label{sparse_regression_equn}
\end{equation}
where 
\[
\boldsymbol{\Xi} = 
\begin{bmatrix}
\boldsymbol{\xi}_{1} &
\boldsymbol{\xi}_{2} &
\dots &
\boldsymbol{\xi}_{n}
\end{bmatrix}
\]
is the matrix of unknown coefficient vectors \(\xi_k\). For a system with \(n\) variables and a candidate library $\boldsymbol{\Theta}(\cdot)$ of \(l\) terms, the coefficient matrix \( \boldsymbol{\Xi}\) has dimension \( {l \times n} \). The nonzero entries of $\boldsymbol{\xi}_{k}$ indicate the active candidate functions contributing to the dynamics of the $k^{\text{th}}$ state. Once $\boldsymbol{\xi}_{k}$ is identified, the nonlinear equation for the $k^{\text{th}}$ state is given by
\begin{equation}
    \dot{x}_{k} = \boldsymbol{\Theta}(\mathbf{x}, \mathbf{u}) \, \boldsymbol{\xi}_{k},
\end{equation}
where $x_{k}$ is the $k^{\text{th}}$ component of $\mathbf{x}$, and $\boldsymbol{\Theta}(\mathbf{x}, \mathbf{u})$ denotes the vector of candidate functions evaluated at $\mathbf{x}$ and $\mathbf{u}$.

\begin{example}

Consider $n=2$ state variables, i.e.,
\( \mathbf{x} = [x_1, x_2]^\top \), and $q=2$ input variables, i.e.,
\( \mathbf{u} = [u_1, u_2]^\top \), and a library with maximum polynomial degree $d=2$. If the variables are measured at time points $m=2$, i.e.,  \( \mathbf{t} = [t_1, t_2] \), then
\[
\mathbf{X}^2 = 
\big[\, x_1^2(\mathbf{t}),\; x_1 x_2(\mathbf{t}),\; x_2^2(\mathbf{t}) \big], 
\qquad
\mathbf{U}^2 = 
\big[\, u_1^2(\mathbf{t}),\; u_1 u_2(\mathbf{t}),\; u_2^2(\mathbf{t}) \big],
\]
and the candidate library has $l=15$ terms
\begin{align*}
\boldsymbol{\Theta}(\mathbf{X}, \mathbf{U}) = [&
\mathbf{1},\;
x_1(\mathbf{t}),\; x_2(\mathbf{t}),\;
u_1(\mathbf{t}),\; u_2(\mathbf{t}),\;
x_1 x_2(\mathbf{t}),\;
u_1 u_2(\mathbf{t}),\;
x_1 u_1(\mathbf{t}),\; x_1 u_2(\mathbf{t}),\;
x_2 u_1(\mathbf{t}),\; x_2 u_2(\mathbf{t}),\\
& x_1^2(\mathbf{t}),\; x_2^2(\mathbf{t}),\;
u_1^2(\mathbf{t}),\; u_2^2(\mathbf{t})
].
\end{align*}

Suppose the true dynamics are
$\dot{x}_1 = 0.6 x_1 + 0.2x_2 u_1$ and $\dot{x}_2 = -0.4 x_2 + 0.1 u_2^2$. Equation \eqref{sparse_regression_equn} seeks $\boldsymbol{\Xi}\in \mathbb{R}^{15\times2}$, whose columns are coefficient vectors
$\boldsymbol{\xi}_k\in \mathbb{R}^{15\times1}$. In this example,
\[
\boldsymbol{\xi}_1 = [\,0,\;0.6,\;0,\;0,\;0,\;0,\;0,\;0,\;0,\;0.2,\;0,\; \ldots,\;0\,]^\top,
\quad
\boldsymbol{\xi}_2 = [\,0,\;0,\;-0.4,\;0,\; \ldots,\;0,\;0.1\,]^\top,
\]
where the nonzero entries identify the active library terms. \end{example}

Equation \eqref{sparse_regression_equn} can be solved as a sparse regression problem to determine the active coefficients $\boldsymbol{\Xi}$ of the candidate nonlinearities $\mathbf{\Theta(X, U)}$, and a Sparse Identification of Nonlinear Dynamics (SINDy) algorithm was introduced by Brunton et al. (2016) \cite{brunton2016discovering}. Early variants of SINDy used $L^1$ regularization term for the regression, which was applied over the entire time domain to uncover systems of equations with constant coefficients \cite{brunton2016discovering, kaiser2018sparse}. Consequently, Li et al. (2019)\cite{li2019discovering} proposed a time-varying variant of SINDy, which uncovers the time-varying coefficients of a dynamical system, and Rudy et al. (2017, 2019)\cite{rudy2017data, rudy2019data} implemented Ridge Regression to incorporate an $L^2$ regularization term to the loss function. 

\subsection{Discovering time-varying dynamics}

The Sequential Threshold Ridge Regression (STRR), which combines Ridge Regression with a sparsity-enforcing thresholding technique described by Rudy et al. (2017, 2019) \cite{rudy2017data, rudy2019data}, was applied to discover time-varying coefficients $\boldsymbol{\Xi}$ in equation \ref{sparse_regression_equn}, following the framework of Li et al.\ (2019) \cite{li2019discovering}.

First, the input variables $\mathbf{u}(t)$ in \eqref{dynamical-system-equn} were selected based on prior knowledge, and time series data for all variables ($\mathbf{x}(t)$ and $\mathbf{u}(t)$) were collected (Figure \ref{model}a). Each variable's time series was rolled over a number of time points to reduce noise. The derivative $\dot{\mathbf{X}}$ of the state variables was then numerically estimated, and the candidate library $\boldsymbol{\Theta}(\mathbf{X}, \mathbf{U})$ was constructed \eqref{candidate-term-equation}.

Active constant coefficients, $\boldsymbol{\xi}_k$ for the $k^\text{th}$ state of \eqref{sparse_regression_equn}, were extracted over the entire training time domain (Figure \ref{model}b) using STRR. The STRR method minimizes error with regularization while enforcing sparsity through an iterative two-step optimization. Ridge Regression was first applied by minimizing the objective function
\begin{equation}
    \arg \min_{\boldsymbol{\xi_k}} \| \Dot{\mathbf{X}}_k - \mathbf{\Theta(X, U)} \boldsymbol{\xi}_k \|_2^2 + \lambda \| \boldsymbol{\xi}_k \|_2^2,
    \label{objective_function_equation}
\end{equation}
where the term \( \lambda \| \boldsymbol{\xi}_k \|_2^2 \) represents the regularization component that penalizes nonzero coefficients, controlled by the hyperparameter \( \lambda \), which helps in avoiding over-fitting. A sparsity-enforcing thresholding step was then applied, where coefficients smaller than a pre-defined threshold $\tau$ are set to zero:
\begin{equation}
    \xi_{i,k} = 
    \begin{cases} 
    0 & \text{if } |\xi_{i, k}| < \tau \\
    \xi_{i, k} & \text{otherwise}
    \end{cases}
    \label{thrusholding_equaiton}
\end{equation}
where \( \xi_{i, k} \) is \( i \)-th entry of $\boldsymbol{\xi}_k$. This step removes small, less significant features. Ridge regression was then refitted on the remaining active features (those whose coefficients have not been zeroed out) $\mathbf{\Theta_{\text{active}}(X, U)}$, using the same loss function
\begin{equation}
    \arg \min_{\boldsymbol{\xi}_k} \| \Dot{\mathbf{X}}_k - \mathbf{\Theta_{\text{active}}(X, U)} \boldsymbol{\xi}_k \|_2^2 + \lambda \| \boldsymbol{\xi}_k \|_2^2.
    \label{objective_function_equation_2}
\end{equation}

This iterative refinement ensures that the model continues to adapt to the most relevant features. The process repeats until the coefficients converge, i.e., the change in the coefficients between two consecutive iterations is smaller than a specified tolerance \( \epsilon \):
\begin{equation}
    \|\boldsymbol{\xi}_k^{(t)} - \boldsymbol{\xi}_k^{(t-1)}\| < \epsilon,
    \label{tollerance_equation}
\end{equation}
where \( \boldsymbol{\xi}_k^{(t)} \) and \( \boldsymbol{\xi}_k^{(t-1)} \) are the coefficients at iterations \( t \) and \( t-1 \), respectively.

The correlations between the active candidate terms ($\mathbf{\Theta_{\text{active}}(X, U)}$) and the target variable $\dot{\mathbf{X}}_k(\mathbf{t})$ for the $k^\text{th}$ state were computed (Figure \ref{model}c). The top $N$ active candidate terms $\mathbf{\Theta_{\text{top N}}}$ showing the highest correlation with $\dot{\mathbf{X}}_k(\mathbf{t})$ were then identified from the candidate set ($\mathbf{\Theta_{\text{active}}(X, U)}$). The coefficients of $\mathbf{\Theta_{\text{top N}}}$ were selected as time-varying coefficients (referred to as time-varying parameters) (Figure \ref{model}c), while the remaining coefficients were kept constant. Furthermore, the bias term was designated as time-varying by default to account for baseline shifts between intervals.

To capture the time-varying dynamics of the system \eqref{dynamical-system-equn}, a window size $w$ was defined and kept fixed over the entire training period. STRR (\eqref{objective_function_equation}-\eqref{tollerance_equation}) was applied across each time window $[t, t+w]$ (Figure \ref{model}d), with each parameter initialized using its value from the preceding interval. This sequential fitting strategy allowed the time-varying parameters to be dynamically adjusted while less impactful parameters remained fixed, capturing essential changes in the model structure and maintaining continuity between intervals. The discovered expression for the $k^\text{th}$ state can be written as
\begin{equation}
    \dot{x}_k (t) = \xi_{k,0}(t) + \sum_{z_i(t) \in \Theta_{\text{top N}}} \xi_{k,i}(t) z_i(t) + \sum_{z_j(t) \notin \Theta_{\text{top N}}} \bar{\xi}_{k,j} z_j(t),
    \label{data-learning-equation}
\end{equation}
where $\xi (t)$ are time-varying parameters and $\bar{\xi}$ are the constant coefficients. Here, $z_i(t) \in \mathbf{\Theta_{\text{top }N}}$ denotes the top $N$ active candidate terms and $z_j(t) \in \mathbf{\Theta_{\text{active}}(X, U)} \setminus  \mathbf{\Theta_{\text{top }N}}$ represents the remaining active terms with fixed coefficients.

\subsection{Forecasting from learned dynamics}

To forecast the future dynamics of the state variables, the discovered representation in equation~\eqref{data-learning-equation}, which incorporates both fixed and time-varying coefficients, was leveraged. The main idea involves predicting the time-varying parameters, $\xi(t)$, while keeping the constant coefficients, $\bar{\xi}$, unchanged. Forecasting is thereby enabled by updating only the components of the model that evolve over time.

A ML approach was employed to model the evolution of the time-varying parameters. During model development, the time-varying parameters $\xi(t)$ obtained from STRR were piecewise constant over each window of length $w$. However, during forecasting, the predicted parameters were not constrained to remain constant within a window; instead, continuous predictions were produced by the ML model based on the input features. Relevant predictors were used to train the models, with each time-varying parameter designated as a target output (Figure~\ref{model}e). By combining the forecasted time-varying parameters with the fixed coefficients and the active candidate terms, the right-hand side of the dynamical model~\eqref{data-learning-equation} was updated, and the forecast of the $k^{\text{th}}$ state variable was obtained by
\begin{equation}
    \dot{\hat{x}}_k (t) = \hat{\xi}_{k,0}(t) 
    + \sum_{z_i(t) \in \Theta_{\text{top N}}} \hat{\xi}_{k,i}(t)\, z_i(t)
    + \sum_{z_j(t) \notin \Theta_{\text{top N}}} \bar{\xi}_{k,j}\, z_j(t),
    \label{forecasting-equation}
\end{equation}
where $\hat{\xi}(t)$ are the forecasted time-varying parameters and $\bar{\xi}$ are the constant coefficients. Here, $z_i(t) \in \mathbf{\Theta_{\text{top }N}}$ denotes the top $N$ active candidate terms, and $z_j(t) \in \mathbf{\Theta_{\text{active}}(X, U)} \setminus \mathbf{\Theta_{\text{top }N}}$ represents the remaining active terms with fixed coefficients.

\begin{figure}[htb!]
    \centering
    \includegraphics[width=\linewidth]{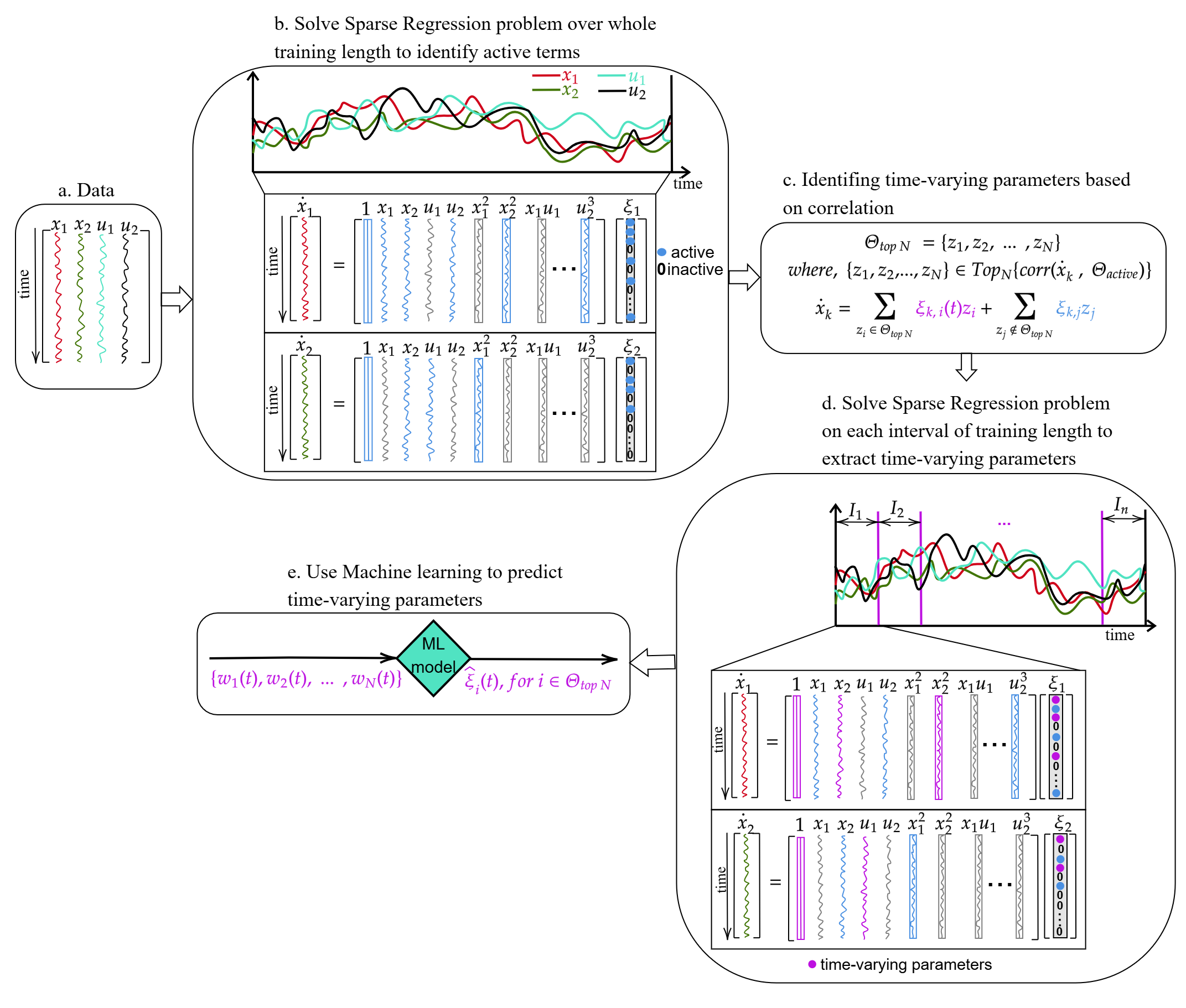}
    \caption{Schematic of data-driven discovery of governing equations and forecasting from the learned system. a) Time-series data of variables are gathered from natural or simulated systems. b) A sparse regression problem is solved over the whole time series to identify active terms and constant coefficients. Variables in gray represent non-active terms, while variables in blue represent active terms with constant coefficients. Variables in boxes represent generated terms. c) Top N candidate terms are chosen based on the correlation of active candidate terms with the derivative of the state variables. The coefficients of these terms are treated as time-varying parameters. By default, the bias term is also considered time-varying. d) Time series data is split into intervals, and sparse regression is performed within each interval to determine the time-varying parameters of the top candidate terms. Purple represents top candidate terms and their time-varying parameters, while blue indicates terms with constant coefficients. e) A ML model is employed to forecast the time-varying parameters. Inputs are relevant predictors, and outputs are the time-varying parameters.}
    \label{model}
\end{figure}

\section{Case studies}
We applied the proposed method (Section \ref{Methods}) to four datasets: two simulated from known ODE systems and two empirical. The simulated datasets were generated from 1) the Susceptible-Infectious-Recovered (SIR) model and 2) a Consumer-Resource (CR) model. The empirical datasets consisted of 1) Carbon dioxide (CO\textsubscript{2}) and methane (CH\textsubscript{4}) concentrations measured in Alberta's oil sands tailings ponds and 2) cyanobacteria cell counts collected from lakes across Alberta.

\subsection{SIR model}

The SIR model \cite{kermack1927contribution} is a foundation for most disease-spreading models. In this model, the population is categorized into three compartments: susceptible ($S$), infectious ($I$), and recovered ($R$). The susceptible compartment represents individuals who are not yet infected but are at risk of contracting the disease through contact with infectious individuals. The infectious compartment consists of individuals who are actively carrying the disease and can transmit it to others. Lastly, the recovered compartment includes individuals who have recovered from the disease and are assumed to have gained immunity, no longer participating in disease transmission. These compartments interact dynamically, with transitions governed by a system of ODEs that model the rates of change in each group over time. The model equations are given by
\begin{equation}
\begin{cases}
\frac{dS(t)}{dt} \;=\; \Lambda \,-\, \beta(t) S(t)I(t) \,-\, \mu S(t), \\
\frac{dI(t)}{dt} \;=\; \beta(t) S(t)I(t) \,-\, \alpha I(t) \,-\, \mu I(t), \\
\frac{dR(t)}{dt} \;=\; \alpha I(t) \,-\, \mu R(t),
\end{cases}
\label{SIR_model_pre}
\end{equation}
where $\Lambda$ is the recruitment rate of susceptible populations, $\beta(t)$ is the time-dependent disease transmission rate, $\mu$ is the natural death rate, and $\alpha$ is the recovery rate from the disease. The first two equations of \eqref{SIR_model_pre} are independent of recovered ($R$), and the reduced equations are given by
\begin{equation}
\begin{cases}
\frac{dS(t)}{dt} \;=\; \Lambda  \,-\, \beta(t) S(t)I(t) \,-\, \mu S(t), \\
\frac{dI(t)}{dt} \;=\; \beta(t) S(t) I(t) \,-\, \alpha I(t) \,-\, \mu I(t).
\end{cases}
\label{SIR_model_reduced}
\end{equation}

When a white noise is added to the equations of the SIR model \eqref{SIR_model_reduced}, then the SIR model with additive white noise has the form \cite{chakraborty2024early}  
\begin{eqnarray}\label{WhiteNoise_Model}
\begin{cases}
\displaystyle{dS(t)} \;=\; \displaystyle{ \Lambda dt \,-\, \beta(t) S(t)I(t) dt \,-\, \mu S(t) dt \,+\, \sigma_1 dW_1(t) }, \\
\displaystyle{dI(t)} \;=\; \displaystyle{\beta(t) S(t)I(t) dt \,-\, \alpha I(t) dt \,-\, \mu I(t) dt \,+\, \sigma_2 dW_2(t)},\\
\end{cases}
\end{eqnarray}
where $\sigma_i,\, i = 1, 2,$  are the intensities of white noise, and $W_i(t),\, i = 1, 2,$ are independent Wiener processes.

We used equation \eqref{WhiteNoise_Model} to simulate the time series of susceptible and infected populations over 1500 days. The model parameters were set as $\Lambda = 100$, $\alpha = 0.1$, $\mu = 0.1$, with initial conditions $S_0 = 500$ and $I_0 = 7$. The transmission rate $\beta(t)$ was defined as a combination of sinusoidal functions to capture irregular seasonal peaks by
\begin{equation}
\beta(t) = A + A_1 \sin(\omega_1 t) + A_2 \sin(\omega_2 t + \pi/4) + A_3 \sin(\omega_3 t + \pi/6)\sin(0.005 t),
\end{equation}
where $A = 0.0003$, $A_1 = 0.00015$, $A_2 = 0.00012$, $A_3 = 0.00010$, and the angular frequencies are $\omega_1 = 2\pi / (365/2)$, $\omega_2 = 2\pi / 365$, and $\omega_3 = 2\pi / (365/2)$. Small random fluctuations were added to $\beta(t)$ at each time step to account for irregular environmental variability, and the values were clipped to remain non-negative. Additive stochasticity was introduced through Wiener processes, $dW_i \sim \mathcal{N}(0, \sqrt{dt}), i = 1,2$, and the Euler--Maruyama method \cite{kloeden1992numerical} was used for numerical integration with $dt = 0.01$. Simulations were repeated with noise intensities $\sigma_1 = \sigma_2 = 0, 0.1, 0.2, 0.3,$ and $0.4$ to examine the effect of increasing stochasticity on model predictions.

The simulated trajectories of $S(t)$ and $I(t)$ were then used to learn the dynamics of the SIR model using time-varying parameters following the proposed method in Section~\ref{Methods}. To discover the governing equations, we expressed the system in the form
\begin{equation}
\begin{cases}
    \frac{dS}{dt} = f(S, I ), \\
    \frac{dI}{dt} = f(S, I ),
\end{cases}
\label{SIR-model-discovered-equation}
\end{equation}
where the time-varying parameters were systematically selected according to the procedure described in Section~\ref{Methods}.

\subsection{CR model}

The CR model \cite{rosenzweig1963graphical} is a fundamental framework used to study the interactions between two populations: a resource population (e.g., prey or a nutrient source) and a consumer population (e.g., predators or organisms that depend on the resource). This model captures the dynamics of population growth and interactions, providing insights into ecological systems. 

In this model, the resource population (\(R\)) grows at a rate proportional to its size, governed by the growth rate \(r\). The growth of \(R\) is constrained by consumption from the consumer population (\(C\)), which follows a saturating functional response. This response is characterized by the consumption rate \(a\) and the half-saturation constant \(b\). Simultaneously, the consumer population (\(C\)) grows through resource consumption but experiences a constant mortality rate \(m\). The dynamics of the model are described by the following equations
\begin{equation}
\begin{cases}
        \frac{dR}{dt} = rR - \frac{aR C}{1 + bR}, \\
    \frac{dC}{dt} = \frac{aR C}{1 + bR} - mC.
\end{cases}
\label{CR-model}
\end{equation}

To incorporate the stochastic fluctuation into the model, we added noise into each of the equations of the system \eqref{CR-model}. The noise-induced CR model has the form

\begin{equation}
\begin{cases}
        \displaystyle{dR(t)} = \displaystyle{rR(t) dt - \frac{aR(t) C(t)}{1 + bR(t)}dt + \sigma_1 dW_1}, \\
    \displaystyle{dC(t)} = \displaystyle{\frac{aR(t) C(t)}{1 + bR(t)} dt - mC(t) dt + \sigma_2 dW_2},
\end{cases}
\label{CR-model-noise}
\end{equation}
where $\sigma_i,\, i = 1, 2,$  are the intensities of white noise, and $W_i(t),\, i = 1, 2,$ are independent Wiener processes.

We simulated the CR models dynamics using the Euler--Maruyama method, which is well suited for SDEs. The model parameters were chosen as the resource growth rate $r = 0.1$, the consumption rate $a = 0.005$, the half-saturation constant $b = 0.001$, and the consumer mortality rate $m = 0.025$. The initial population sizes were set to $R_0 = 20$ for the resource and $C_0 = 10$ for the consumer. Additive environmental stochasticity was incorporated into both equations through Wiener processes sampled from
$
dW_i \sim \mathcal{N}(0, \sqrt{dt}), \, i = 1,2,
$
and simulations were repeated across a range of noise intensities,
$
\sigma_1 = \sigma_2 \in \{0,\, 0.5,\, 1.0,\, 1.5,\, 2.0\}.
$
The system was simulated over a time horizon of 1500 units with a time step of $dt = 0.01$. After each Euler--Maruyama update, any negative population values were prevented by enforcing small positive lower bounds to maintain biological realism.

We used the simulated time series of $C$ and $R$ to learn the dynamics of the CR model using the proposed method described in Section~\ref{Methods}. During the learning phase, we assumed that the system could be represented in the general form
\begin{equation}
\begin{cases}
        \frac{dR}{dt} = f(C, R), \\
    \frac{dC}{dt} = f(C, R),
\end{cases}
\label{CR-model-discovered-equation}
\end{equation}
where the time-varying parameters were systematically selected following the procedure in Section~\ref{Methods}.

\subsection{Gas concentration}

\label{time-varying parameters-method}

CO\textsubscript{2} and CH\textsubscript{4} are two major greenhouse gases, and a significant amount of these gases is found in the oil-sand tailing areas. Our goal is to learn a system of equations that can capture the concentration time-series of CO\textsubscript{2} and CH\textsubscript{4} in Alberta's oil-sand tailing ponds based on relevant weather conditions and diluent inputs. Additionally, we leveraged the learned equation to predict future concentrations based on estimated weather conditions and diluent inputs. 

We used two different air monitoring station datasets that contain CO\textsubscript{2} and CH\textsubscript{4} concentrations, and one nearby meteorological tower dataset by the Wood Buffalo Environmental Association (WBEA) \cite{wbeaNetworkStation}, located in the regional municipality of Wood Buffalo in northeast Alberta, Canada. The air monitoring stations we used are AMS-1: Bertha Ganter--Fort McKay (Latitude: 57.189428, Longitude: -111.640583), and AMS-18: Stony Mountain (Latitude: \allowbreak 55.621408, Longitude: -111.172686). The meteorological tower we used is JP104-Site 1004 (Latitude: 57.11901, Longitude: -111.42542). The parameters used in this study and their sources are described in Table \ref{gas_varibale_list}. Air monitoring station and meteorological tower data can be accessed from \cite{wbeaNetworkStation}. The period of the collected data was from January 2019 to December 2024. Air monitoring stations recorded hourly data with a few missing data. We deleted rows containing missing measures and took the daily average of the recorded data.

Both stations measure CO\textsubscript{2} and CH\textsubscript{4} and some common weather variables (AT, RH, WS, WD, PC, and GR) (Table \ref{gas_varibale_list}). We wanted to keep the same number of weather variables (AT, RH, WS, WD, GR, BP, DEW, PAR, SR, and PC) in each dataset. Therefore, the other weather variables (BP, DEW, PAR, and SR) were considered the same in both datasets which were collected from nearby meteorological tower JP104 (Table \ref{gas_varibale_list}). 

To incorporate the source of the hydrocarbons, which are responsible for producing these gases, we used the total diluent lost data from Alberta Energy Regulator (AER) reports \cite{albertaenergyregulator}. Since it is unknown which monitoring stations capture gases from which exact source tailing pond, we took the summation of monthly reported diluent loss data of all the reported mining companies in that region between January 2019 and December 2024. Then we divided it equally throughout the months to get daily diluent estimations.

\begin{table}[hbt!]
    \caption{Variables and data sources used in this case study.}
    \label{gas_varibale_list}
    \centering
    \footnotesize
    \begin{tabular}{|l|l|l|l|}
    \hline
    \rowcolor{gray!25}
       \textbf{Variable}  & \textbf{Abbreviation}  & \textbf{Unit} & \textbf{Data Sources} \\
       \hline
       Methane & CH\textsubscript{4}  & Parts per million (ppm) & AMS-1, AMS-18 \\
       \hline
       Carbon Dioxide & CO\textsubscript{2} & Parts per million (ppm) & AMS-1, AMS-18 \\
       \hline 
       Ambient Temperature & AT  & Degrees Celsius (\textdegree C) & AMS-1, AMS-18 \\
       \hline
       Relative Humidity & RH & Percentage (\%) & AMS-1, AMS-18 \\
       \hline
       Wind Speed & WS  & $\text{km} \, \text{h}^{-1}$ & AMS-1, AMS-18 \\
       \hline
       Wind Direction  & WD  & Degrees $\in [0, 360)$ & AMS-1, AMS-18 \\
        \hline
       Global radiation & GR  & $\text{W} \, \text{m}^{-2}$  & AMS-1, AMS-18 \\
       \hline
       Precipitation & PC  & millimeters (mm)  & AMS-1, AMS-18 \\
       \hline
       Barometric pressure  & BP  & mB & JP104 \\
       \hline
       Dew point & DEW  & Degrees Celsius (\textdegree C) & JP104 \\
       \hline
       Photosynthetically active radiation at 16m & PAR  & $\mu\text{mol} \, \text{m}^{-2} \, \text{s}^{-1}$ & JP104 \\
       \hline
       Solar radiation & SR  & $\text{W} \, \text{m}^{-2}$ & JP104 \\
       \hline
       Diluent & DIL &  $\text{m}^3$ & AER \\
       \hline
    \end{tabular}
\end{table}

To learn the system of equations, We took up to quadratic terms of the library of candidate nonlinear functions $\mathbf{\Theta(X, U)}$ (equation \eqref{candidate-term-equation}) to fit the data of each of the gas with a quadratic equation. We considered the diluent lost data as a source of CO\textsubscript{2} and CH\textsubscript{4}, and since the gas concentration varies with the weather conditions, we incorporated all the weather variables as input variables in the candidate terms. Additionally, CH\textsubscript{4} can be produced by the chemical reaction of CO\textsubscript{2}. Therefore, based on these considerations, we considered the learned system of equations to have the general form
\begin{equation}
\begin{cases}
    \frac{dCO_2}{dt} = f(DIL, \text{weather variables} ), \\
    \frac{dCH_4}{dt} = f(CO_2, DIL, \text{weather variables} ),
\end{cases}
\label{gases-model}
\end{equation}
where the weather variables are AT, RH, WS, WD, GR, BP, DEW, PAR, SR, and PC, described in Table \ref{gas_varibale_list}. 

\subsection{Cyanobacteria cell count}
Cyanobacterial blooms are commonly found in freshwater ecosystems and can be highly toxic, posing significant risks to aquatic life, wildlife, and human health through the production of harmful toxins that contaminate water sources. Our goal is the data-driven discovery of an equation of cyanobacteria cell counts and prediction from the discovered equation based on nutrient availability, meteorological conditions, lake features, and watershed data. 

The dataset includes observations from 78 lakes and reservoirs across Alberta, Canada, spanning the years from 1986 to 2017 \cite{cheggerud_2023_10109225}. These lakes and reservoirs are situated between latitudes 49.3\textdegree N and 58.8\textdegree N and longitudes 110.0079\textdegree W and 119.21667\textdegree W \cite{heggerud2024predicting}. The lake monitoring data is collected during May to October, which includes phytoplankton data gathered through the Alberta Lake Monitoring Program. The original lake monitoring dataset contains various parameters such as chemical concentrations in lake water, Secchi depth measurements, lake stratification data, cyanobacteria cell counts, etc, provided by Alberta Environment and Parks (AEP). We used the processed datasets from Heggerud et al. \cite{heggerud2024predicting}, which can be accessed from \cite{cheggerud_2023_10109225}. We averaged the data based on each year and month from different lakes. The variables used for this case study are described in Table \ref{cb_variables_table}.

\begin{table}[htb!]
\centering
\footnotesize
\caption{List of variables and their abbreviations used for CB cell count}
\label{cb_variables_table}
\begin{tabular}{|l|l|}
\hline
\rowcolor{gray!25}
\textbf{Variable} & \textbf{Description}                                                                                                                           \\ 
\hline
CB\textsubscript{cell}          & Cell count of cyanobacteria sampled on the monitoring day                                                                                      \\ 
\hline
P                 & Phosphorus concentration sampled on the monitoring day                                                                                         \\ 
\hline
N                 & Nitrogen concentration sampled on the monitoring day                                                                                           \\ 
\hline
SD                & Secchi depth of lake on the monitoring day                                                                                                     \\ 
\hline
ST                & Stratification state of lake on the monitoring day                                                                                             \\ 
\hline
AT                & \begin{tabular}[c]{@{}l@{}}Average daily temperature from the day of monitoring to 2 weeks\\ Before the next available CB sample\end{tabular}  \\ 
\hline
PC                & \begin{tabular}[c]{@{}l@{}}Average daily precipitation from the day of monitoring to 2 weeks\\ Before the next available sample\end{tabular}     \\ 
\hline
WS                & \begin{tabular}[c]{@{}l@{}}Average daily wind speed from the day of monitoring to 2 weeks\\ before the next available sample\end{tabular}      \\ 
\hline
SR                & \begin{tabular}[c]{@{}l@{}}Average daily solar radiation from the day of monitoring to 2 weeks\\ before the next available sample\end{tabular} \\ 
\hline
LD                & Lake depth                                                                                                                                     \\ 
\hline
LE                & Lake elevation                                                                                                                                 \\ 
\hline
PL                & \% pastureland in watershed                                                                                                                    \\ 
\hline
CL                & \% cropland in watershed                                                                                                                       \\ 
\hline
FT                & \% forest in watershed                                                                                                                         \\ 
\hline
WL                & \% wetland in watershed                                                                                                                        \\ 
\hline
LWAR              & Lake-watershed area ratio                                                                                                                      \\ 
\hline
\end{tabular}
\end{table}

We used the method described in section \ref{Methods} to learn a single equation for CB cell count (CB\textsubscript{cell}) based on all the covariates listed in Table \ref{cb_variables_table}. The learned equation has the general form of 
\begin{equation}
    \frac{dCB_{cell}}{dt} = f( CB_{cell}, \text{Lake Monitoring}, \text{Meteorology}, \text{Lake feature}, \text{Watershed features} ),
\end{equation}
where the Lake Monitoring includes P, N, SD, and ST. Meteorology includes AT, PC, WS, and SR.
Lake feature includes LD and LE. Watershed features include PL, CL, FT, WL, and LWAR. The CB\textsubscript{cell} remains in the RHS of the equation to indicate that the new cell count depends on the existing cell count. 

\subsection{Simulated weather data}

Once the system was learned through the time-varying parameters, we used a ML model to forecast these parameters based on weather conditions. In this study, we considered four weather variables-temperature, humidity, wind speed, and precipitation-as predictors of the time-varying parameters. For the gases and CB datasets, these weather variables were available in the dataset and directly used for forecasting. In contrast, for the SIR and CR models, we simulated the corresponding weather variables so that the ML model could be applied in the same way, enabling prediction of the time-varying parameters based on these simulated environmental conditions.

We used the Ornstein-Uhlenbeck (OU) process \cite{uhlenbeck1930theory} to simulate the time series of the weather variables. The OU process is represented by a stochastic differential equation (SDE) commonly used to model mean-reverting phenomena. The process is described by the equation
\begin{equation}
    dX(t) = -\theta_X \left( X(t) - X_{\text{mean}} \right) dt + \sigma_X \, dW(t) , 
    \label{OU-process}
\end{equation}
where \( X(t) \) represents the state variable (e.g., temperature) at time \( t \), and \( X_{\text{mean}} \) is the long-term mean around which the process fluctuates. The parameter \( \theta_X \) quantifies the rate of mean reversion, indicating how quickly deviations from the mean are corrected. The term \( \sigma_X \) represents the intensity of stochastic noise, and \( dW(t) \) is the Wiener process, which introduces randomness.

To simulate data from the OU process, we discretized the continuous-time equation using the Euler-Maruyama method. The discretized version of the SDE is given by

\[
X_{i+1} = X_i - \theta_X \left( X_i - X_{\text{mean}} \right) \Delta t + \sigma_X \sqrt{\Delta t} \, \xi_i ,
\]
where \( \Delta t \) is the time increment, \( \xi_i \) are independent standard normal random variables. 

The weather simulation model was designed with parameter values that closely mimic realistic dynamics for temperature, humidity, wind speed, and precipitation over the simulated period. We used a time step of \( \Delta t = 0.01 \), which provides sufficient granularity to capture variations in weather conditions over a total duration of 1500 days. The initial temperature was \( 20^\circ \mathrm{C} \), the initial humidity was \( 60\% \), the initial wind speed was \( 10 \, \mathrm{m/s} \), and the initial precipitation was \( 0.2 \, \mathrm{mm/day} \). The long-term mean values for each parameter were defined as \( T_{\text{mean}} = 0^\circ \mathrm{C} \), \( H_{\text{mean}} = 50\% \), \( W_{\text{mean}} = 5 \, \mathrm{m/s} \), and \( P_{\text{mean}} = 0.1 \, \mathrm{mm/day} \). The rates of mean reversion (\( \theta \)) were chosen as \( \theta_T = 0.1 \), \( \theta_H = 0.05 \), \( \theta_W = 0.2 \), and \( \theta_P = 0.1 \). The noise intensities were set as \( \sigma_T = 10 \), \( \sigma_H = 5 \), \( \sigma_W = 1 \), and \( \sigma_P = 0.05 \), corresponding to fluctuations in temperature, humidity, wind speed, and precipitation, respectively. Higher noise intensities for temperature and humidity reflect their typically greater variability compared to wind speed and precipitation. Additionally, boundaries were specified for each variable to prevent unrealistic values. Temperature was limited to the range \([-20, 35]^\circ \mathrm{C}\), humidity to \([0, 100]\%\), wind speed to \([0, 30] \, \mathrm{m/s}\), and precipitation to \([0, 100] \, \mathrm{mm/day}\).

\subsection{ML models, error metrics, and cross-validation}

Each variable's time series was first smoothed using a 30-point rolling window to reduce high-frequency noise. To account for differences in scale across variables, all series were normalized to the range $[0,1]$ using min--max scaling,
\[
X' = \frac{X - X_{\min}}{X_{\max} - X_{\min}},
\]
where $X'$ denotes the normalized variable and $X_{\min}$ and $X_{\max}$ are the minimum and maximum values of $X$, respectively.

As a case study, we employed a Random Forest (RF) regression model \cite{breiman2001random} to predict each time-varying parameter independently from four weather-related covariates: air temperature, relative humidity, wind speed, and precipitation. The RF model was implemented using the \texttt{RandomForest\allowbreak Regressor} class from \texttt{sklearn.ensemble} \cite{pedregosa2011scikit}. To ensure ensemble stability, we used 5000 trees (\texttt{n\_estimators} = 5000). Model complexity was controlled by restricting the maximum tree depth to 5 (\texttt{max\_depth}=5), requiring a minimum of 10 samples to split an internal node (\texttt{min\_samples\_\allowbreak split}=10), and enforcing at least 5 samples per leaf (\texttt{min\_samples\_leaf}=5). The weather variables served as predictors, and each time-varying parameter was treated as a separate target (Figure~\ref{model}e).

For each dataset, the time series was partitioned into training, validation, and testing segments. The testing period was divided into five folds of 30 days each, except for the CB dataset, which consisted of five monthly observations. Temporal ordering was strictly preserved. For each test fold, a validation fold of equal duration was selected immediately preceding the test period (Figure~\ref{train-test-valid}). The governing equations were fitted to the training data using equation~\eqref{data-learning-equation}, and forecasts were generated using equation~\eqref{forecasting-equation}.

Model selection was performed using an expanding-window time-series cross-validation strategy (Figure~\ref{train-test-valid}). The window length $w$ was varied from 7 to 28 days in increments of 7 (from 5 to 25 in increments of 5 for the CB dataset). Simultaneously, the number of time-varying parameters was varied from zero (fixed-parameter model) up to the full set of active parameters, with the bias term always included. Starting from an initial group of validation folds (e.g., folds 1-5), the validation window was progressively expanded (1-6, 1-7, \ldots, 1-9) (Figure~\ref{train-test-valid}). For each configuration, performance was assessed using the mean absolute error (MAE),
\[
\mathrm{MAE} = \frac{1}{n} \sum_{i=1}^{n} \left| \hat{y}_i - y_i \right|,
\]
where $y_i$ and $\hat{y}_i$ denote the observed and predicted values, respectively. The configuration minimizing the average validation MAE was selected and used to forecast the subsequent test fold.

To examine the influence of interval length and the number of time-varying parameters, simulations were performed on each test fold and the resulting MAEs were computed. Within each fold, the configuration yielding the lowest MAE was identified as the \emph{optimal configuration} and compared with the cross-validation-selected model. For all datasets, we additionally compared the proposed time-varying parameter model with a fixed-parameter baseline, in which the entire time series is represented by constant coefficients (Figure~\ref{model}b) \cite{rudy2017data}. In both cases, the regularization parameter was set to $\lambda=0.01$, the sparsity threshold to $\tau=0.001$, and the maximum number of iterations to 10,000.

To further benchmark our framework, we compared its forecasting performance against two standard machine learning baselines: a hybrid CNN--LSTM model \cite{shi2015convolutional, hochreiter1997long, bai2018empirical} and a Gradient Boosting Machine (GBM) \cite{friedman2001greedy}.

For the CNN--LSTM model, inputs were constructed using sliding windows of length $w$ days. For each target variable, the predictors consisted of the exogenous covariates over the same $w$-day interval, while the response comprised the corresponding $w$-day sequence of the target state, aligned in time. For the SIR and CR datasets, the covariates included air temperature (AT), wind speed (WS), precipitation (PC), and relative humidity (RH). For the gas concentration and CB datasets, all weather variables listed in Tables~\ref{gas_varibale_list} and~\ref{cb_variables_table} were used.

The CNN--LSTM architecture combines convolutional layers for feature extraction 
\cite{bai2018empirical} with recurrent memory units \cite{hochreiter1997long}, 
and was implemented in \texttt{TensorFlow/Keras} \cite{abadi2016tensorflow,chollet2015keras}. The architecture consisted of a one-dimensional causal convolutional layer,
$\mathrm{Conv1D}(\text{filters}=F,\ \text{kernel size}=K,\ \text{activation}=\mathrm{ReLU})$,
followed by a max-pooling layer with pool size 2, an LSTM layer with $U$ hidden units, a dropout layer with rate $d$, and two fully connected layers: a ReLU-activated dense layer with $D$ units and a final linear output layer producing $w$ outputs (with $w=30$ in our experiments). Hyperparameters $(F, K, U, D, d)$ and the learning rate were selected using a \texttt{KerasTuner} \cite{omalley2019kerastuner} random search to minimize the validation mean squared error (MSE). Training employed the Adam optimizer \cite{kingma2014adam} with early stopping \cite{prechelt1998early} (patience of 10 epochs and restoration of the best-performing weights), using a non-shuffled time-series cross-validation scheme. Within each fold, the most recent block of the training data, matching the test length, was reserved for validation.

As an additional baseline, we implemented a GBM model using the same set of predictors as in the CNN--LSTM setup. The GBM input consisted of the exogenous covariates at each time step, and the response was the corresponding target value. The model was trained using 1000 boosting trees, a learning rate of 0.01, a maximum tree depth of 30, and a subsampling ratio of 0.9 to mitigate overfitting. Each terminal node was constrained to contain at least 10 samples. The squared error loss function was used, corresponding to a Gaussian regression framework.

\begin{figure}[hbt!]
    \centering
    \includegraphics[width=1\linewidth]{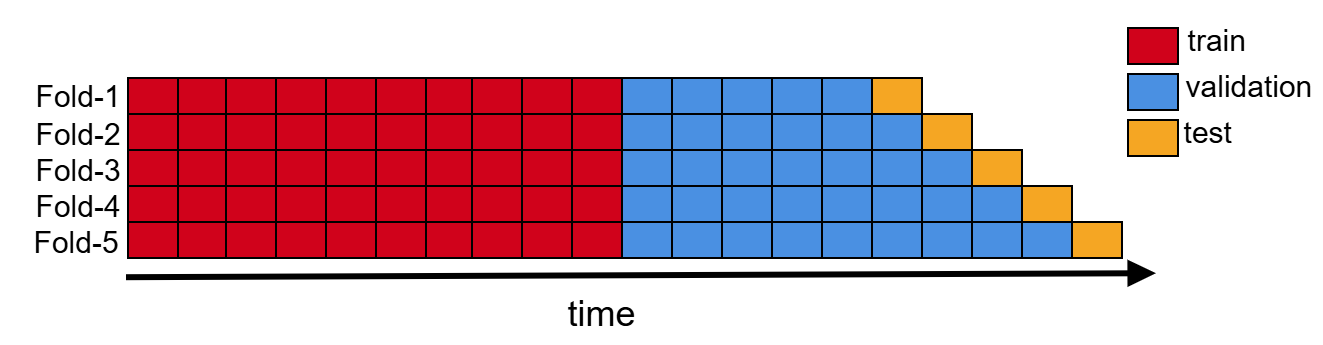}
    \caption{Illustration of the sequential, non-overlapping training, testing, and validation scheme used in this study. The training period (red) covers the majority of the time series. The validation window blocks (blue), equal in length to the test window, are positioned immediately before the test period (yellow). For each fold of the test, only the validation window is expanded forward while the training set remains fixed, producing multiple validation folds without expanding or overlapping the training data.}
    \label{train-test-valid}
\end{figure}

\section{Results}

\subsection{Theoretical guarantees}
\label{subsec:theory}

In this subsection, we establish theoretical guarantees for three key components of our
framework: (i) the propagation
of errors from the learned equations and the forecasted time-varying parameters to the
forecasted states, (ii) error bounds to the split structure of learned and forecasted equations, and (iii) conditions under which the time-varying parameterisation admits
strictly tighter finite-horizon forecast error bounds than any fixed-parameter model built on
the same library. 

Our analysis builds on classical ODE stability theory and sparse equation discovery results \cite{Teschl2012, ZhangSchaeffer2019, brunton2016discovering}. We extend these foundations by deriving finite-horizon forecast error bounds for time-varying parameter models and quantifying error propagation from ML-forecasted parameters to state trajectories.

\subsubsection{Finite-horizon forecast error bounds}

We now quantify how errors in the learned governing equations and in the forecasted
time-varying parameters propagate to errors in the forecasted trajectories. Let
$\mathbf{f}:[0, H]\times\mathbb{R}^n\to\mathbb{R}^n$ denotes the (unknown) right-hand side of the
true system,
\[
\dot{\mathbf{x}}(t) = \mathbf{f}(t,\mathbf{x}(t)), \qquad t\in[0,H],
\]
where $H>0$ is a finite time horizon. We assume that $\mathbf{f}(\cdot)$ can be uniformly approximated on a compact region by a representation
constructed from our candidate library $\boldsymbol{\Theta}(\cdot)$.

\begin{assumption}[Lipschitz dynamics and library approximation \cite{Teschl2012, ZhangSchaeffer2019, brunton2016discovering}]
\label{ass:lipschitz}
Let $D\subset\mathbb{R}^n$ be a compact set containing the trajectories of interest.
Assume:
\begin{enumerate}
  \item[(a)] (Lipschitz dynamics) For each $t\in[0,H]$, the map $x\mapsto f(t,x)$ is
  globally Lipschitz on $D$ with constant $L>0$:
  \[
  \|f(t,x)-f(t,y)\| \le L \|x-y\| \quad\text{for all } x,y\in D.
  \]
  \item[(b)] (Library approximation) There exist coefficient functions
  $\xi^\star:[0,H]\to\mathbb{R}^p$ and a library vector
  $\Theta:\,D\times[0,H]\to\mathbb{R}^p$ such that
  \[
  \bigl\|f(t,x) - \Theta(x,u(t))\,\xi^\star(t)\bigr\| \le \varepsilon_{\mathrm{lib}}
  \quad\text{for all } (t,x)\in[0,H]\times D,
  \]
  for some $\varepsilon_{\mathrm{lib}}\ge 0$.
  \item[(c)] (Bounded library) There exists $B_\Theta>0$ such that
  $\|\Theta(x,u(t))\|_{\mathrm{op}} \le B_\Theta$ for all $(t,x)\in[0,H]\times D$,
  where $\|\cdot\|_{\mathrm{op}}$ denotes the denotes the spectral norm.
\end{enumerate}
\end{assumption}

Our forecasting pipeline produces estimated time-varying coefficients
$\hat\xi(t)$ (via STRR on sliding windows, followed by ML model forecasts), and uses the model
\[
\hat f(t,x) := \Theta(x,u(t))\,\hat\xi(t).
\]
We assume that the ML model forecasts approximate the ideal coefficient functions $\xi^\star(t)$
with a uniform bound on the forecast horizon.

\begin{assumption}[Accuracy of coefficient forecasts]
\label{ass:param-err}
There exists $\varepsilon_{\mathrm{par}}\ge 0$ such that
\[
\sup_{t\in[0,H]} \bigl\|\hat\xi(t) - \xi^\star(t)\bigr\|
\;\le\; \varepsilon_{\mathrm{par}}.
\]
\end{assumption}

Under Assumptions~\ref{ass:lipschitz}-\ref{ass:param-err}, the difference between the true
vector field $f$ and the learned-forecasted vector field $\hat f$ is uniformly bounded:
\[
\bigl\|f(t,x) - \hat f(t,x)\bigr\|
\le \varepsilon_{\mathrm{lib}} + B_\Theta \varepsilon_{\mathrm{par}}
=: \delta
\quad\text{for all } (t,x)\in[0,H]\times D.
\]

We now quantify the impact of this bound on the state trajectories.

\begin{theorem}[Finite-horizon forecast error bound]
\label{thm:ode-error}
Suppose Assumptions~\ref{ass:lipschitz} and~\ref{ass:param-err} hold and let
$\delta = \varepsilon_{\mathrm{lib}} + B_{\Theta}\varepsilon_{\mathrm{par}}$.
Let $\mathbf{x}:[0,H]\to D$ solve the true system
$\dot{\mathbf{x}}(t)=\mathbf{f}(t,\mathbf{x}(t))$ with initial condition $\mathbf{x}(0)=\mathbf{x}_0$, and let
$\hat{\mathbf{x}}:[0,H]\to D$ solve the learned system
$\dot{\hat{\mathbf{x}}}(t)=\hat f(t,\hat{\mathbf{x}}(t))$ with the same initial condition
$\hat{\mathbf{x}}(0)=\mathbf{x}_0$.
Then, for all $t\in[0,H]$,
\[
\|x(t)-\hat x(t)\|
\;\le\; \frac{\delta}{L}\bigl(e^{Lt}-1\bigr)
\;\le\; \frac{\delta}{L}\bigl(e^{LH}-1\bigr).
\]
\end{theorem}

\begin{proof}
Define the error $e(t):=x(t)-\hat x(t)$. Differentiating gives
\[
\dot e(t) = f(t,x(t)) - \hat f(t,\hat x(t))
= \bigl(f(t,x(t)) - f(t,\hat x(t))\bigr)
+ \bigl(f(t,\hat x(t)) - \hat f(t,\hat x(t))\bigr).
\]
Taking norms and applying Assumption~\ref{ass:lipschitz}(a) and the definition of $\delta$
yields
\[
\|\dot e(t)\|
\le L \|e(t)\| + \delta.
\]
Let $y(t):=\|e(t)\|$. Then $y(0)=0$ and
\[
\dot y(t) \le L y(t) + \delta
\quad\text{for a.e. } t\in[0,H].
\]
By the differential form of Gr\"onwall's inequality \cite{gronwall1919note, Teschl2012},
\[
y(t) \le \frac{\delta}{L}\bigl(e^{Lt}-1\bigr)
\quad\text{for all } t\in[0,H],
\]
which proves the stated bound.
\end{proof}

Theorem~\ref{thm:ode-error} is a standard ODE stability result specialised to our
library-based representation. It shows that all sources of error in our pipeline enter the
state forecast through the single quantity $\delta = \varepsilon_{\mathrm{lib}} +
B_\Theta \varepsilon_{\mathrm{par}}$, combining the approximation error of the library
representation and the error of the time-varying parameter forecasts.

%%%%%%%%%%%%%%%%%%%%%%%%%%%%%%%%%%%%%%%%%%%%%%%%%%%%%%%

\subsubsection{Bounds to the split structure of learned and forecasted equations}

We now specialize the baseline bound to the learned forecasting model in this paper.
For reference, we restate the split form of the learned dynamics \eqref{data-learning-equation} and its forecasting counterpart \eqref{forecasting-equation}.
Let
\begin{equation}
\label{eq:tv_model}
\dot{\mathbf{x}}(t)=\mathbf{f}_{\mathrm{tv}}(t,\mathbf{x}(t)):= \Theta_{\mathrm{top}N}(x(t),u(t))\,\xi_{\mathrm{tv}}(t)
+ \Theta_{\mathrm{fix}}(x(t),u(t))\,\bar\xi,
\end{equation}
where $\Theta_{\mathrm{fix}}:=\Theta_{\mathrm{active}}\setminus \Theta_{\mathrm{top}N}$ and $\bar\xi$ denotes the
fixed coefficient vector for the non-top-$n$ active terms. The forecast model is
\begin{equation}
\label{eq:tv_forecast}
\dot{\hat{\mathbf{x}}}(t)=\hat{\mathbf{f}}_{\mathrm{tv}}(t,\hat {\mathbf{x}}(t))
:= \Theta_{\mathrm{top}N}(\hat x(t),u(t))\,\hat\xi_{\mathrm{tv}}(t)
+ \Theta_{\mathrm{fix}}(\hat x(t),u(t))\,\bar\xi.
\end{equation}

\begin{assumption}[Uniform boundedness of the top-$n$ library block]
\label{ass:topn_bounded}
There exists $B_{\mathrm{top}N}>0$ such that
\[
\sup_{t\in[0,H]}\ \sup_{x\in D}\ \|\Theta_{\mathrm{top}N}(x,u(t))\|_{\mathrm{op}}\le B_{\mathrm{top}N}.
\]
\end{assumption}

\begin{assumption}[Uniform accuracy of time-varying coefficient forecasts]
\label{ass:topn_forecast}
There exists $\varepsilon_{\mathrm{top}N}\ge 0$ such that
\[
\sup_{t\in[0,H]}\ \|\hat\xi_{\mathrm{tv}}(t)-\xi_{\mathrm{tv}}(t)\|\le \varepsilon_{\mathrm{top}N}.
\]
\end{assumption}

\begin{theorem}[Split-model forecast bound: only $\Theta_{\mathrm{top}N}$ forecast errors enter]
\label{thm:split_only_topn}
Assume Assumption~\ref{ass:lipschitz}a holds for $f_{\mathrm{tv}}$ in \eqref{eq:tv_model}, and
Assumptions~\ref{ass:topn_bounded}--\ref{ass:topn_forecast} hold. Let $\dot {\mathbf{x}}$ solve \eqref{eq:tv_model}
and $\widehat{\dot {\mathbf{x}}}$ solve \eqref{eq:tv_forecast} with the same initial condition $\mathbf{x}(0)=\hat{\mathbf{x}}(0)$.
If $x(t),\hat x(t)\in D$ for all $t\in[0,H]$, then for all $t\in[0,H]$,
\[
\|x(t)-\hat x(t)\|
\le \frac{B_{\mathrm{top}N}\varepsilon_{\mathrm{top}N}}{L}\bigl(e^{Lt}-1\bigr)
\le \frac{B_{\mathrm{top}N}\varepsilon_{\mathrm{top}N}}{L}\bigl(e^{LH}-1\bigr).
\]
\end{theorem}

\begin{proof}
For any $(t,x)\in[0,H]\times D$, the fixed block cancels exactly:
\[
f_{\mathrm{tv}}(t,x)-\hat f_{\mathrm{tv}}(t,x)
=\Theta_{\mathrm{top}N}(x,u(t))\bigl(\xi_{\mathrm{tv}}(t)-\hat\xi_{\mathrm{tv}}(t)\bigr).
\]
Hence, by Assumptions~\ref{ass:topn_bounded}--\ref{ass:topn_forecast},
\[
\sup_{(t,x)\in[0,H]\times D}\|f_{\mathrm{tv}}(t,x)-\hat f_{\mathrm{tv}}(t,x)\|
\le B_{\mathrm{top}N}\varepsilon_{\mathrm{top}N}.
\]
Apply Theorem~\ref{thm:ode-error} with $\delta=B_{\mathrm{top}N}\varepsilon_{\mathrm{top}N}$ to obtain the stated bound.
\end{proof}

\begin{remark}
\label{rem:only_topn}
Theorem~\ref{thm:split_only_topn} is specific to the split structure of \eqref{eq:tv_model}--\eqref{eq:tv_forecast}:
the fixed coefficients $\bar\xi$ do not appear in the forcing term of the error bound. Consequently, improving the
forecast accuracy of $\hat\xi_{\mathrm{tv}}$ for the selected $\Theta_{\mathrm{top}N}$ terms directly tightens the
finite-horizon state forecast bound, while the fixed block affects the bound only indirectly through the Lipschitz
constant $L$ of the resulting vector field.
\end{remark}

%%%%%%%%%%%%%%%%%%%%%%%%%%%%%%%%%%%%%%%%%%%%%%%%%%%%%%%

\subsubsection{When time-varying parameters outperform fixed parameters}

Finally, we compare our time-varying parameterisation with a fixed-parameter model. Let $\mathcal{C}_{\mathrm{const}}$ denote the class of constant coefficient vectors
\[
\mathcal{C}_{\mathrm{const}}:=\{\xi:[0,H]\to\mathbb{R}^p:\ \xi(t)\equiv c\},\]
and let $\mathcal{C}_{\mathrm{tv}}^{(m)}$ denote the class of piecewise
constant coefficient functions with at most $m$ subintervals:

\[\mathcal{C}_{\mathrm{tv}}^{(m)}
:=\Bigl\{\xi:[0,H]\to\mathbb{R}^p:\ \xi(t)=c_r\ \text{for }t\in I_r,\ r=1,\dots,m\Bigr\},
\]
where $\{I_r\}_{r=1}^m$ is a partition of $[0,H]$ into $m$ intervals. We quantify how well these classes can approximate the ideal coefficient trajectory
$\xi^\star(t)$.

\begin{definition}[Best uniform coefficient approximation errors]
Let
\[
E_{\mathrm{const}}
:= \inf_{\xi\in\mathcal{C}_{\mathrm{const}}}
\sup_{t\in[0,H]} \|\xi^\star(t)-\xi(t)\|,
\qquad
E_{\mathrm{tv}}^{(m)}
:= \inf_{\xi\in\mathcal{C}_{\mathrm{tv}}^{(m)}}
\sup_{t\in[0,H]} \|\xi^\star(t)-\xi(t)\|.
\]
\end{definition}

Clearly $\mathcal{C}_{\mathrm{const}}\subset\mathcal{C}_{\mathrm{tv}}^{(m)}$, so
$E_{\mathrm{tv}}^{(m)}\le E_{\mathrm{const}}$ for every $m$. The next lemma shows that, under
mild regularity, the inequality is strict for sufficiently large $m$ whenever
$\xi^\star(t)$ is truly time-varying. It is a constructive instance of the fact that
step functions are dense in $C([0,H])$ with respect to the supremum norm
\cite[Chapter~2]{Rudin1987}.

\begin{assumption}[Regularity and non-stationarity of the true coefficients]
\label{ass:coeff-cont}
Each component of $\xi^\star(t)$ is continuous on $[0,H]$, and $\xi^\star$ is not constant
on $[0,H]$ (that is, there exist $t_1,t_2$ with $\xi^\star(t_1)\neq \xi^\star(t_2)$).
\end{assumption}

\begin{lemma}[Step-function approximation beats constant approximation]
\label{lem:piecewise-better}
Suppose Assumption~\ref{ass:coeff-cont} holds. Then $E_{\mathrm{const}}>0$, and for every
$0<\varepsilon<E_{\mathrm{const}}$ there exists an integer $m_\varepsilon\ge 1$ such that,
for all $m\ge m_\varepsilon$,
\[
E_{\mathrm{tv}}^{(m)} \le \varepsilon < E_{\mathrm{const}}.
\]
In particular, for all sufficiently large $m$, we have
$E_{\mathrm{tv}}^{(m)} < E_{\mathrm{const}}$.
\end{lemma}

\begin{proof}
Because $\xi^\star$ is continuous on the compact interval $[0,H]$, it is uniformly
continuous. Hence, for any $\varepsilon>0$, there exists $\delta>0$ such that
$\|\xi^\star(t)-\xi^\star(s)\|<\varepsilon$ whenever $|t-s|<\delta$.

Given such a $\delta$, choose $m_\varepsilon$ large enough that a partition of $[0,H]$
into $m_\varepsilon$ subintervals has each interval of length at most $\delta$. Define a
piecewise constant function $\tilde\xi\in\mathcal{C}_{\mathrm{tv}}^{(m_\varepsilon)}$ by
setting $\tilde\xi(t)$ equal to $\xi^\star$ evaluated at a reference point in each
subinterval. By construction,
\[
\sup_{t\in[0,H]} \|\xi^\star(t)-\tilde\xi(t)\| \le \varepsilon,
\]
so $E_{\mathrm{tv}}^{(m_\varepsilon)}\le\varepsilon$ and the same bound holds for all
$m\ge m_\varepsilon$.

To show $E_{\mathrm{const}}>0$, note that by Assumption~\ref{ass:coeff-cont} there exist
$t_1,t_2$ such that $\xi^\star(t_1)\neq \xi^\star(t_2)$. Let
\[
V := \|\xi^\star(t_1)-\xi^\star(t_2)\| > 0.
\]
For any constant vector $c\in\mathbb{R}^p$,
\[
V = \|\xi^\star(t_1)-\xi^\star(t_2)\|
\le \|\xi^\star(t_1)-c\| + \|\xi^\star(t_2)-c\|
\le 2 \sup_{t\in[0,H]} \|\xi^\star(t)-c\|,
\]
so $\sup_{t\in[0,H]} \|\xi^\star(t)-c\| \ge V/2$ for every $c$. Taking the infimum over
$c$ yields $E_{\mathrm{const}}\ge V/2>0$. Choosing any
$0<\varepsilon<E_{\mathrm{const}}$ then gives the claim.
\end{proof}

We now combine Lemma~\ref{lem:piecewise-better} with Theorem~\ref{thm:ode-error} to obtain
a sufficient condition under which the time-varying parameterisation yields a strictly
smaller finite-horizon forecast error bound than any fixed-parameter model.

\begin{theorem}[When time-varying parameters dominate fixed parameters]
\label{thm:tv-vs-const}
Suppose Assumptions~\ref{ass:lipschitz}, \ref{ass:param-err}, and
\ref{ass:coeff-cont} hold. Consider two idealised models built on the same library
$\Theta$:
\begin{itemize}
  \item a fixed-parameter model with coefficient trajectory
  $\xi_{\mathrm{const}}^\star\in\mathcal{C}_{\mathrm{const}}$ achieving the error
  $E_{\mathrm{const}}$,
  \item a time-varying model with $m$ intervals and coefficient trajectory
  $\xi_{\mathrm{tv}}^\star\in\mathcal{C}_{\mathrm{tv}}^{(m)}$ achieving the error
  $E_{\mathrm{tv}}^{(m)}$.
\end{itemize}
let $\hat{\mathbf{x}}_{\mathrm{tv}}^{(m)}(t)$ denote the trajectory of the time-varying model
\[
\dot{\hat{\mathbf{x}}}(t)
= \boldsymbol{\Theta}(\hat{\mathbf{x}}(t),\mathbf{u}(t))\,\boldsymbol{\xi}_{\mathrm{tv}}^\star(t),
\qquad \hat{\mathbf{x}}_{\mathrm{tv}}^{(m)}(0)=\mathbf{x}_0,
\]
with $\boldsymbol{\xi}_{\mathrm{tv}}^\star \in \mathcal{C}_{\mathrm{tv}}^{(m)}$. For these two models, define
\[
\delta_{\mathrm{const}} := \varepsilon_{\mathrm{lib}} + B_\Theta E_{\mathrm{const}},\qquad
\delta_{\mathrm{tv}}^{(m)} := \varepsilon_{\mathrm{lib}} + B_\Theta E_{\mathrm{tv}}^{(m)}.
\]
Then, for any $H>0$ and any $m\ge m_\varepsilon$ with
$E_{\mathrm{tv}}^{(m)} < E_{\mathrm{const}}$, the finite-horizon forecast error bound of the time-varying model satisfies
\[
\sup_{t\in[0,H]} \|x(t)-\hat x_{\mathrm{tv}}^{(m)}(t)\|
\;\le\;
\frac{\delta_{\mathrm{tv}}^{(m)}}{L}\bigl(e^{LH}-1\bigr)
<
\frac{\delta_{\mathrm{const}}}{L}\bigl(e^{LH}-1\bigr)
\;\ge\;
\sup_{t\in[0,H]} \|x(t)-\hat x_{\mathrm{const}}(t)\|.
\]
In particular, for all sufficiently large $m$, the time-varying parameterisation admits a
strictly smaller worst-case finite-horizon forecast error bound than any fixed-parameter
model based on the same library.
\end{theorem}

\begin{proof}
By Lemma~\ref{lem:piecewise-better}, for all sufficiently large $m$ there exists a
time-varying coefficient trajectory $\xi_{\mathrm{tv}}^\star\in\mathcal{C}_{\mathrm{tv}}^{(m)}$
with $E_{\mathrm{tv}}^{(m)} < E_{\mathrm{const}}$. For this choice of $m$ we have
$\delta_{\mathrm{tv}}^{(m)} < \delta_{\mathrm{const}}$. Applying
Theorem~\ref{thm:ode-error} separately to the fixed-parameter and time-varying models
gives
\[
\sup_{t\in[0,H]} \|x(t)-\hat x_{\mathrm{const}}(t)\|
\le \frac{\delta_{\mathrm{const}}}{L}(e^{LH}-1),
\quad
\sup_{t\in[0,H]} \|x(t)-\hat x_{\mathrm{tv}}^{(m)}(t)\|
\le \frac{\delta_{\mathrm{tv}}^{(m)}}{L}(e^{LH}-1),
\]
and the strict inequality $\delta_{\mathrm{tv}}^{(m)}<\delta_{\mathrm{const}}$ yields the
claimed ordering of the bounds.
\end{proof}

Theorem~\ref{thm:tv-vs-const} formalises the intuition that, when the true parameters are
non-stationary in time, a sufficiently rich time-varying parameterisation can represent the coefficient trajectory more accurately than any fixed-parameter model and therefore
admits strictly tighter worst-case bounds on finite-horizon forecast errors. In practice,
our STRR+RF pipeline searches over window lengths and numbers of time-varying parameters, providing a data-driven approximation to the theoretically favourable regime
identified above.

\subsection{Experiment results}

\subsubsection{Learning and reproducing system dynamics}

In the SIR dataset, the time-varying parameter model \allowbreak demonstrated strong capability of capturing the dynamics, achieving a MAE of 0.9\% for the susceptible ($S$) population and 1.7\% for the infected ($I$) population (Figure~\ref{train_mae_comparison}a). In contrast, the fixed-parameter model yielded significantly higher errors--29.7\% for $S$ and 29.9\% for $I$ (Figure~\ref{train_mae_comparison}a). Across all noise levels, the time-varying model consistently maintained an average MAE below 2\%, whereas the fixed-parameter model exhibited considerably poorer accuracy, with MAEs ranging from 28\% to 33\% (SI Appendix Figure~\ref{train_mae_noise_comparison}). 

For the CR dataset, the time-varying parameter model effectively reproduced the time series dynamics, with an average MAE of 1.9\% for the consumer ($C$) variable and 2.5\% for the resource ($R$) variable (Figure~\ref{train_mae_comparison}b). In comparison, the fixed-parameter model performed substantially worse, producing MAEs of 21.5\% for $C$ and 31.1\% for $R$ (Figure~\ref{train_mae_comparison}b). When examined across varying noise levels, the time-varying model maintained a stable MAE range of 1.6--2.7\%, while the fixed-parameter model fluctuated between 11--38\% (SI Appendix Figure~\ref{train_mae_noise_comparison}).

In the gases dataset, the time-varying parameter approach also exhibited superior performance, capturing the temporal behavior of $\mathrm{CO_2}$ and $\mathrm{CH_4}$ with MAEs of 1.4\% and 1.9\%, respectively (Figure~\ref{train_mae_comparison}c). Conversely, the fixed-parameter model achieved notably higher MAEs of 16.9\% for $\mathrm{CO_2}$ and 9.1\% for $\mathrm{CH_4}$ (Figure~\ref{train_mae_comparison}c). Averaged across multiple station datasets, the time-varying model maintained an overall MAE below 3\%, while the fixed-parameter counterpart ranged between 11--17\% (SI Appendix Figure~\ref{train_mae_noise_comparison}).

In the CB dataset, both models performed comparably, with the time-varying and fixed-parameter models producing nearly identical MAEs of 3.51\% and 3.53\%, respectively (Figure~\ref{train_mae_comparison}d).

Sample plots illustrate that the fixed-parameter model generally fails to capture the evolving trends in the time series across most variables in our case studies, with the exception of $\mathrm{CH_4}$ and CB (Figure~\ref{sample_plot_learning}). The model performs well on the CB dataset, which exhibits only a single peak, and moderately on $\mathrm{CH_4}$, which has multiple peaks (SI Appendix Figure~\ref{sample_plot_learning}). In contrast, the time-varying parameter model accurately captures the temporal dynamics for all variables (SI Appendix Figure~\ref{sample_plot_learning}).

\begin{figure}[hbt!]
    \centering
    % First row
    \begin{subfigure}{0.49\textwidth}
        \includegraphics[width=\linewidth]{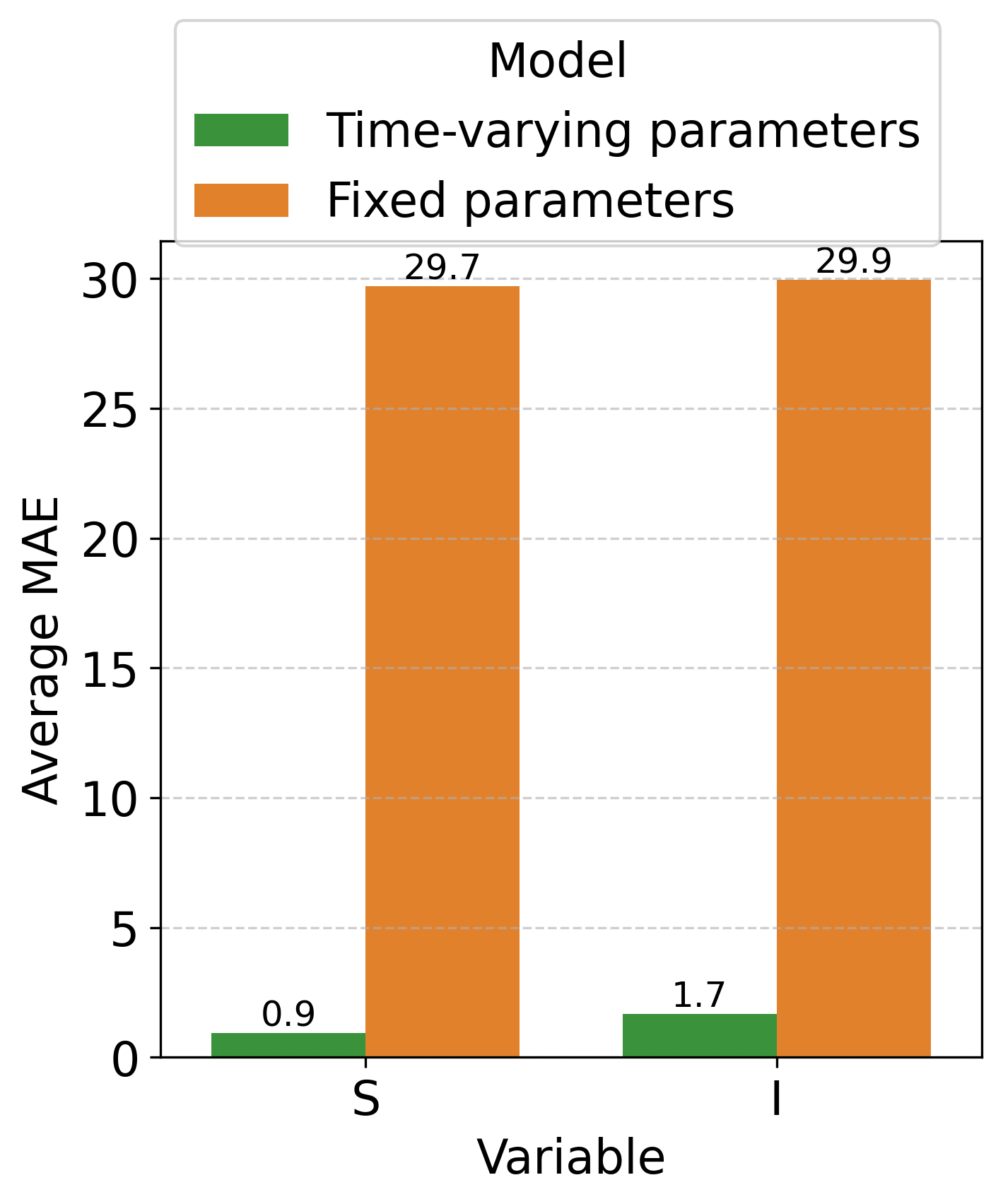}
        \put(-225,210){\textbf{(a)}}
    \end{subfigure}
    \hfill
    \begin{subfigure}{0.49\textwidth}
        \centering
        \includegraphics[width=\linewidth]{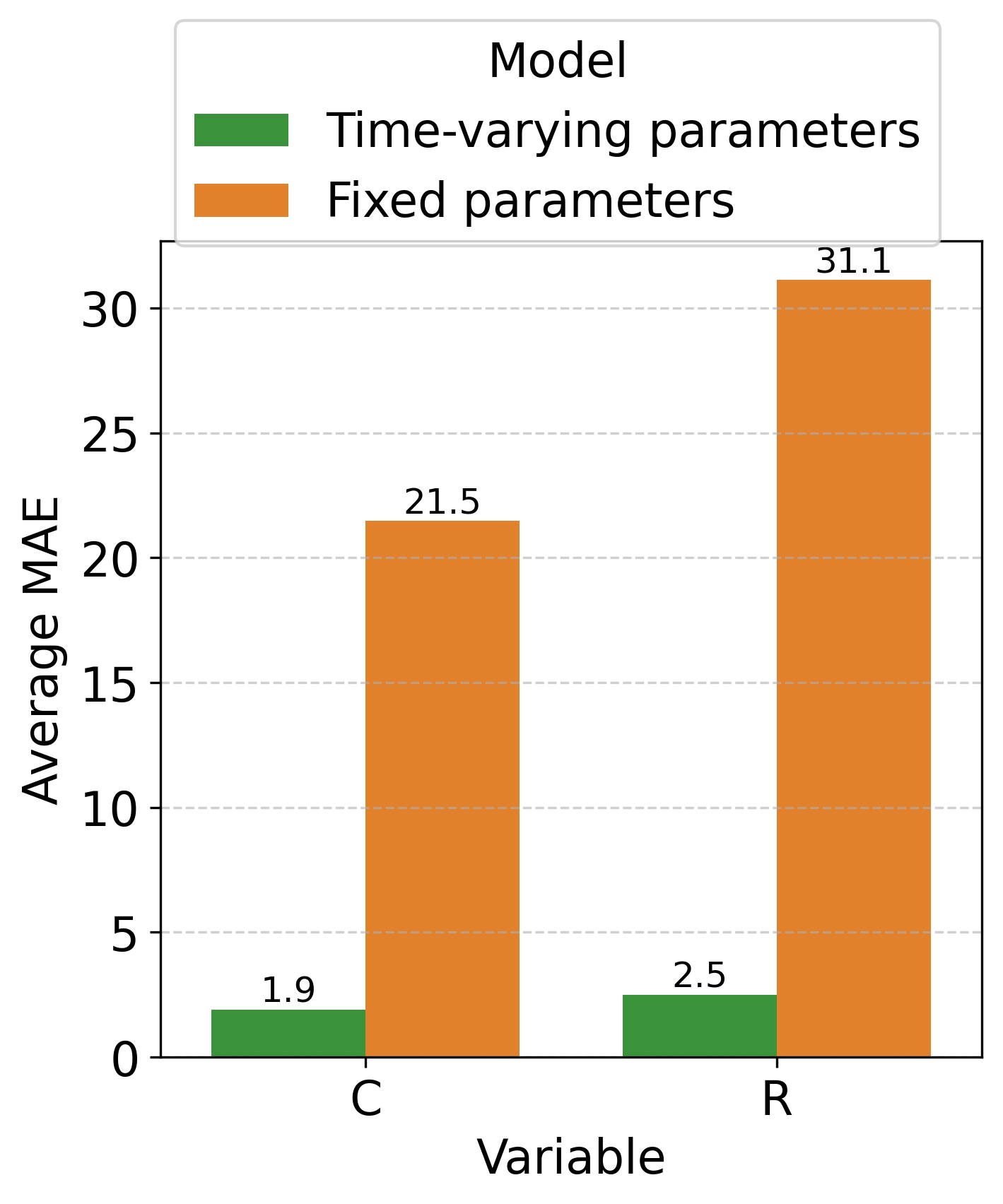}
        \put(-225,210){\textbf{(b)}}
    \end{subfigure}

    \vspace{0.0cm}

    % Second row
    \begin{subfigure}{0.49\textwidth}
        \centering
        \includegraphics[width=\linewidth]{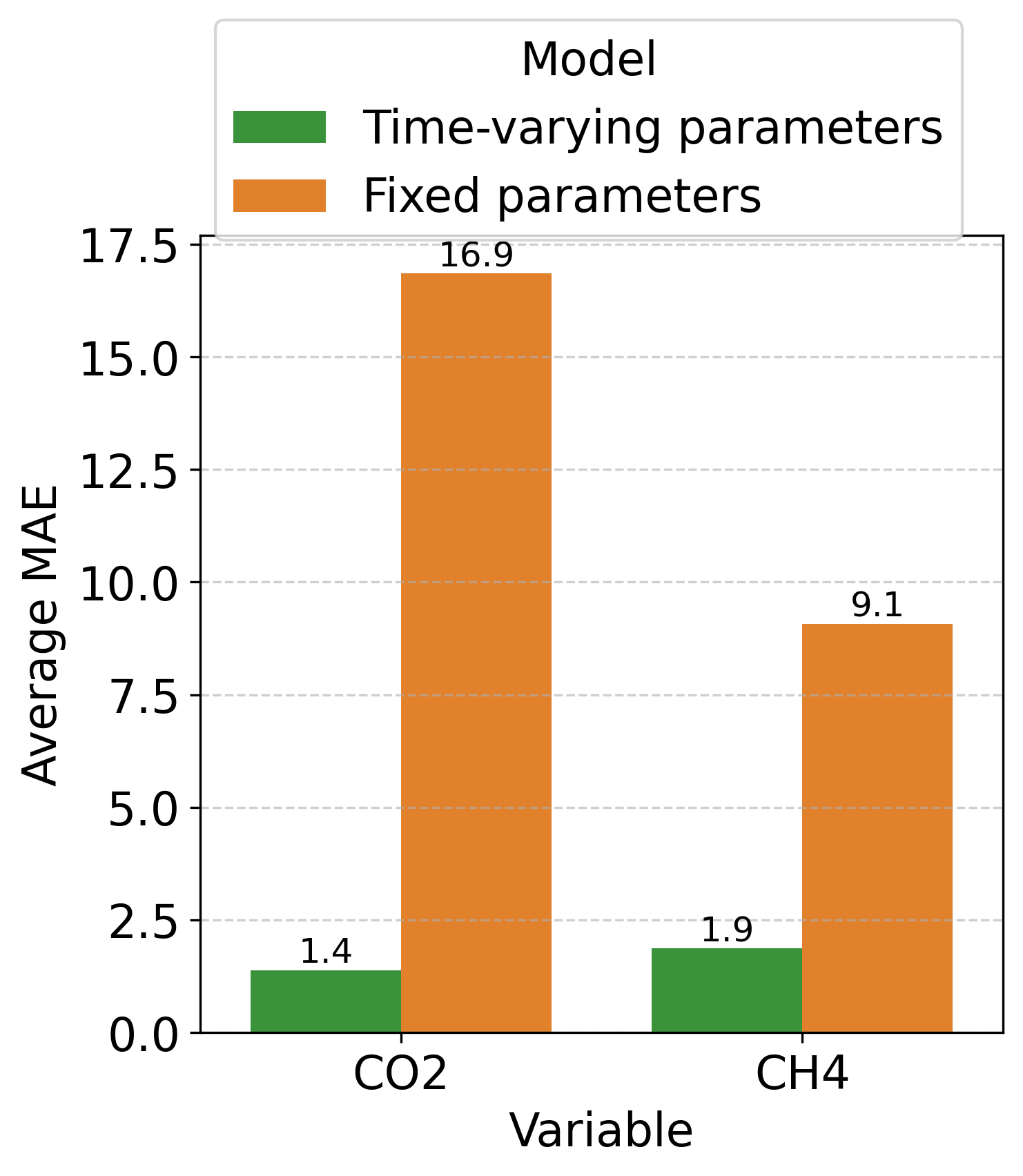}
        \put(-225,210){\textbf{(c)}}
    \end{subfigure}
    \hfill
    \begin{subfigure}{0.49\textwidth}
        \centering
        \includegraphics[width=\linewidth]{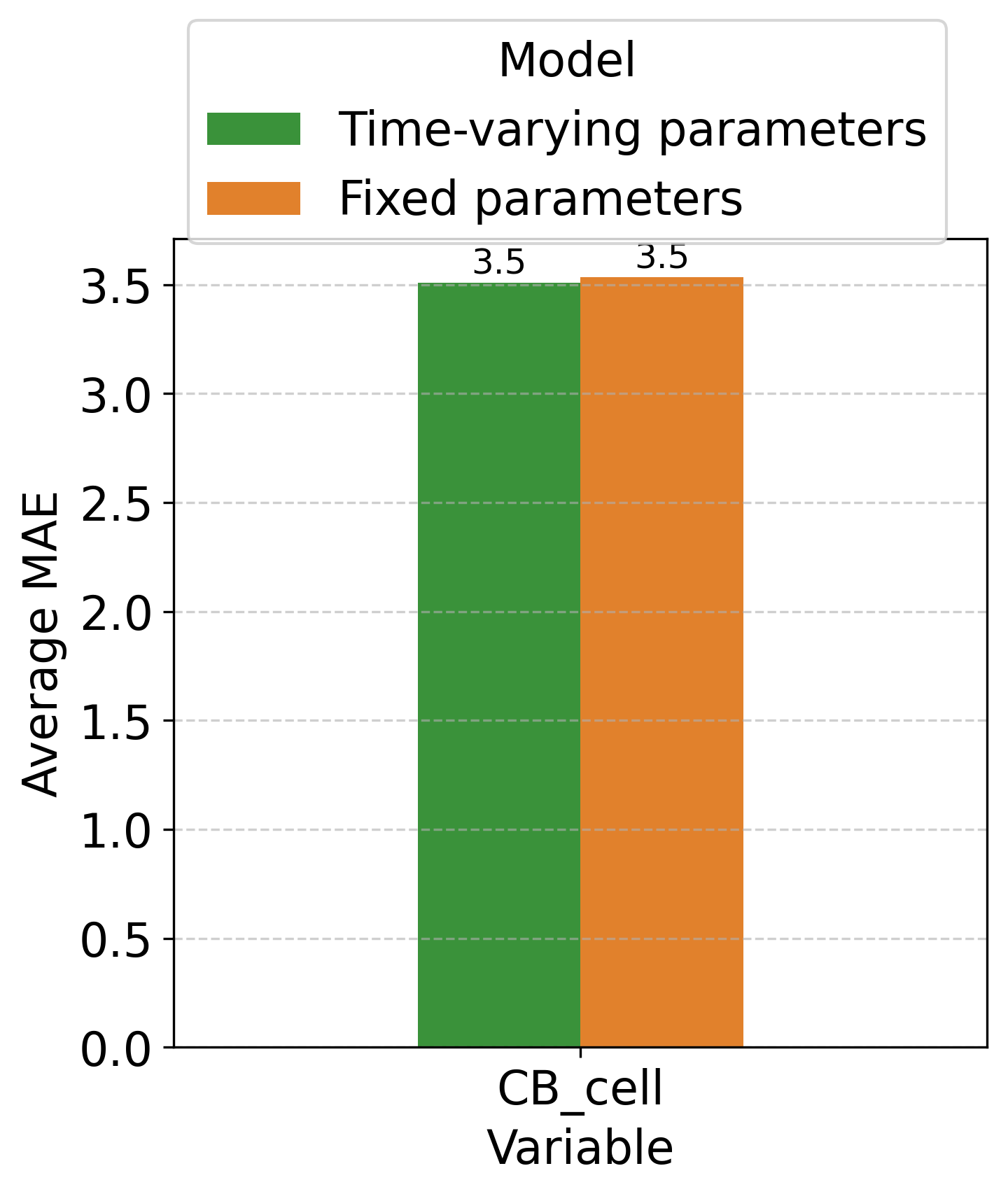}
        \put(-225,210){\textbf{(d)}}
    \end{subfigure}

    \caption{Average learning MAEs for the (a) SIR dataset, (b) CR dataset, (c) gases dataset, and (d) CB dataset. Green bars represent the time-varying parameter model, and orange bars represent the fixed-parameter model. MAEs were computed for each variable across different noise levels or monitoring stations within each dataset and then averaged.}
    \label{train_mae_comparison}
\end{figure}

\subsubsection{Forecasting system dynamics}

For the SIR dataset, the time-varying parameter model demonstrated strong forecasting skill for both the susceptible ($S$) and infected ($I$) populations, achieving average MAEs of 10.0\% and 9.9\%, respectively (Figure~\ref{test_mae_comparison}a). In contrast, the fixed-parameter model yielded substantially higher errors, with MAEs of 13.8\% for $S$ and 14.0\% for $I$ (Figure~\ref{test_mae_comparison}a). Across all noise levels, the time-varying model consistently maintained MAEs between 9\% and 13\%, whereas the fixed-parameter model exhibited degraded performance, with errors ranging from 12\% to 15\% (SI Appendix Figure~\ref{test_mae_noise_comparison}). 

Purely data-driven baselines performed notably worse in forecasting. The CNN--LSTM model produced MAEs exceeding 21\% for both $S$ and $I$, while the GBM achieved slightly lower but still elevated errors of about 21\% for both variables (Figure~\ref{test_mae_comparison}a). This performance gap persisted across increasing noise levels, where CNN--LSTM and GBM errors rose sharply relative to the mechanistic models (SI Appendix Figure~\ref{test_mae_noise_comparison}).

In the CR dataset, the time-varying parameter model accurately forecasted the dynamics of both the consumer ($C$) and resource ($R$) variables, achieving average MAEs of 3.9\% and 5.4\%, respectively (Figure~\ref{test_mae_comparison}b). The fixed-parameter model achieved comparable accuracy, with MAEs of 3.9\% for $C$ and 5.1\% for $R$ (Figure~\ref{test_mae_comparison}b). Across varying noise conditions, both mechanistic models maintained relatively low forecasting errors, generally between 2\% and 7\% (SI Appendix Figure~\ref{test_mae_noise_comparison}).

In contrast, the CNN--LSTM and GBM baselines showed substantially higher errors. The CNN--LSTM yielded MAEs of approximately 18.5\% for $C$ and 18.8\% for $R$, while the GBM achieved MAEs of about 18.0\% and 15.9\%, respectively (Figure~\ref{test_mae_comparison}b). At every noise level, CNN-LSTM and GBM errors remained roughly three to five times larger than those of the mechanistic models (SI Appendix Figure~\ref{test_mae_noise_comparison}).

In the gases dataset, the time-varying model effectively forecasted the temporal behavior of $\mathrm{CO_2}$ and $\mathrm{CH_4}$, achieving MAEs of 10.4\% and 12.0\%, respectively (Figure~\ref{test_mae_comparison}c). In contrast, the fixed-parameter model showed higher errors for $\mathrm{CH_4}$, with MAEs 24.5\% (Figure~\ref{test_mae_comparison}c). Averaged across multiple stations, the time-varying approach maintained an overall MAE around 11\%, whereas the fixed-parameter model ranged from 6\% to 27\% (SI Appendix Figure~\ref{test_mae_noise_comparison}).

The CNN--LSTM and GBM models performed substantially worse, with CNN--LSTM MAEs exceeding 25\% for $\mathrm{CH_4}$ and 35\% for $\mathrm{CO_2}$, and GBM errors remaining above 16\% and 34\%, respectively (Figure~\ref{test_mae_comparison}c). Across different stations, the mechanistic time-varying model remained comparatively stable, whereas CNN--LSTM and GBM errors varied between 23\% to 34\% (SI Appendix Figure~\ref{test_mae_noise_comparison}).

For the CB dataset, the time-varying parameter and fixed-parameter models achieved comparable forecasting performance, with MAEs of 3.0\% and 3.3\%, respectively (Figure~\ref{test_mae_comparison}d). The CNN-LSTM baseline yielded a MAE of approximately 2.2\%, outperforming the mechanistic models, whereas GBM achived MAE of around 10.8\%.

\begin{figure}[htb!]
    \centering
    % First row
    \begin{subfigure}{0.49\textwidth}
        \includegraphics[width=\linewidth]{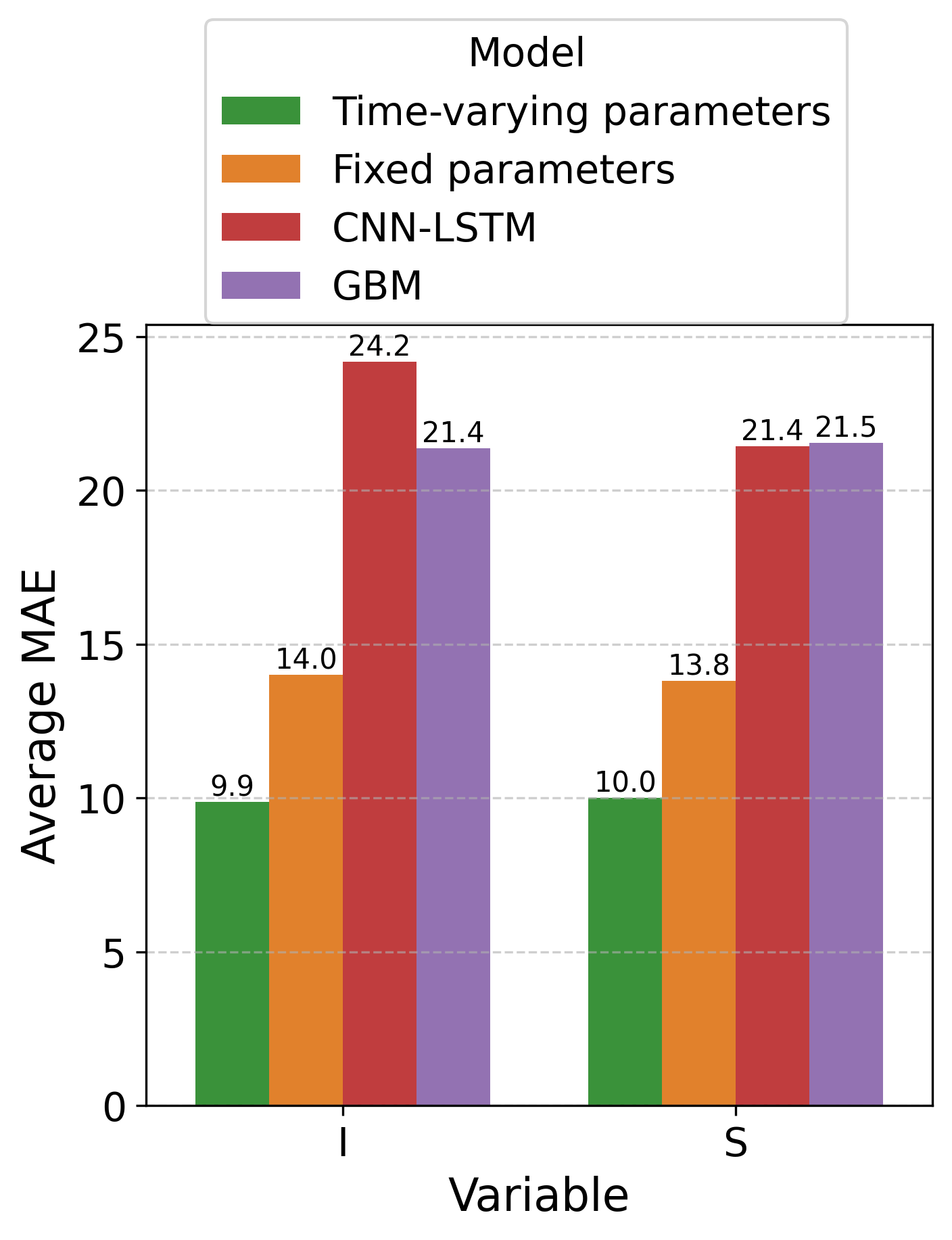}
        \put(-225,210){\textbf{(a)}}
    \end{subfigure}
    \hfill
    \begin{subfigure}{0.49\textwidth}
        \centering
        \includegraphics[width=\linewidth]{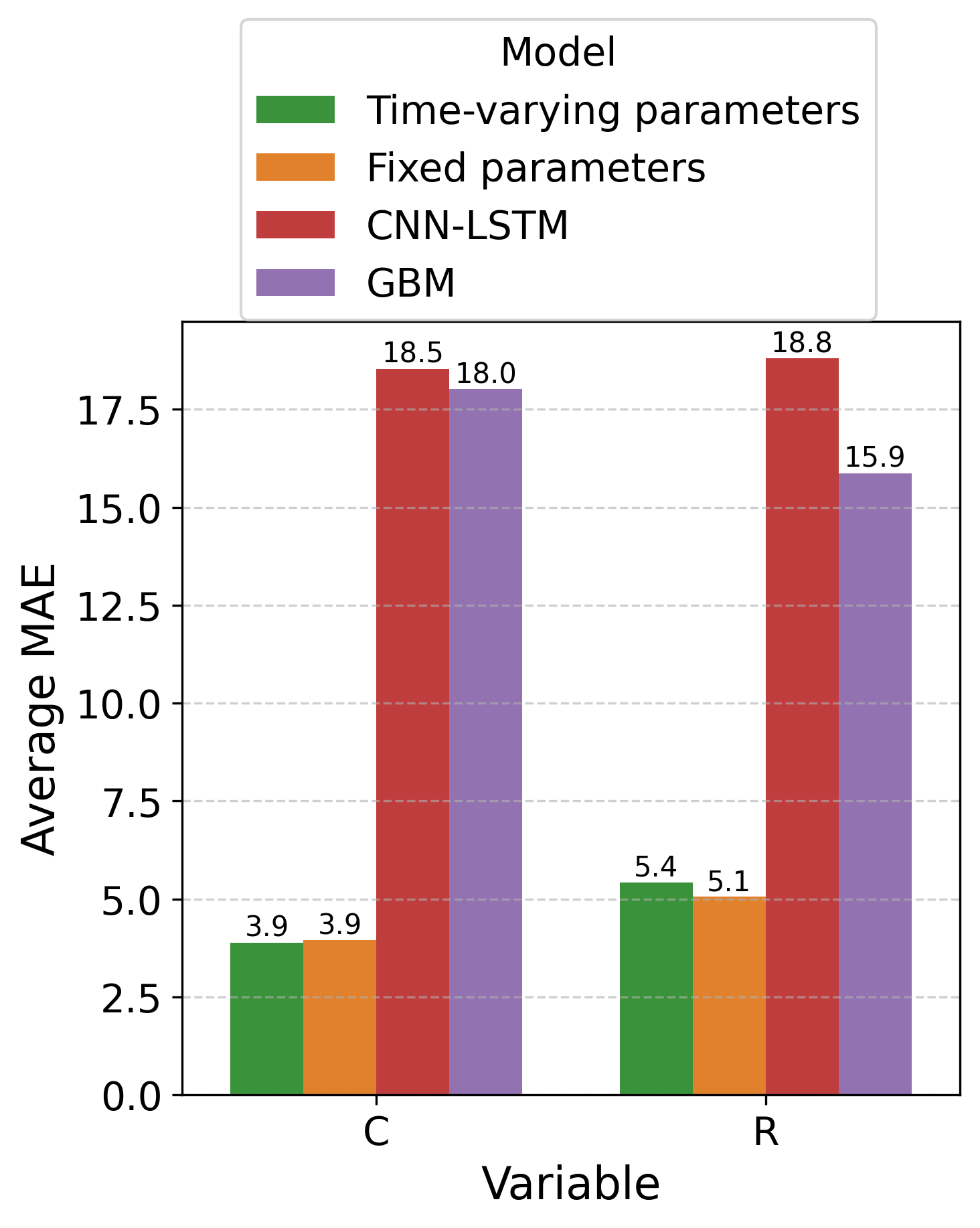}
        \put(-225,210){\textbf{(b)}}
    \end{subfigure}

    \vspace{0.0cm}

    % Second row
    \begin{subfigure}{0.49\textwidth}
        \centering
        \includegraphics[width=\linewidth]{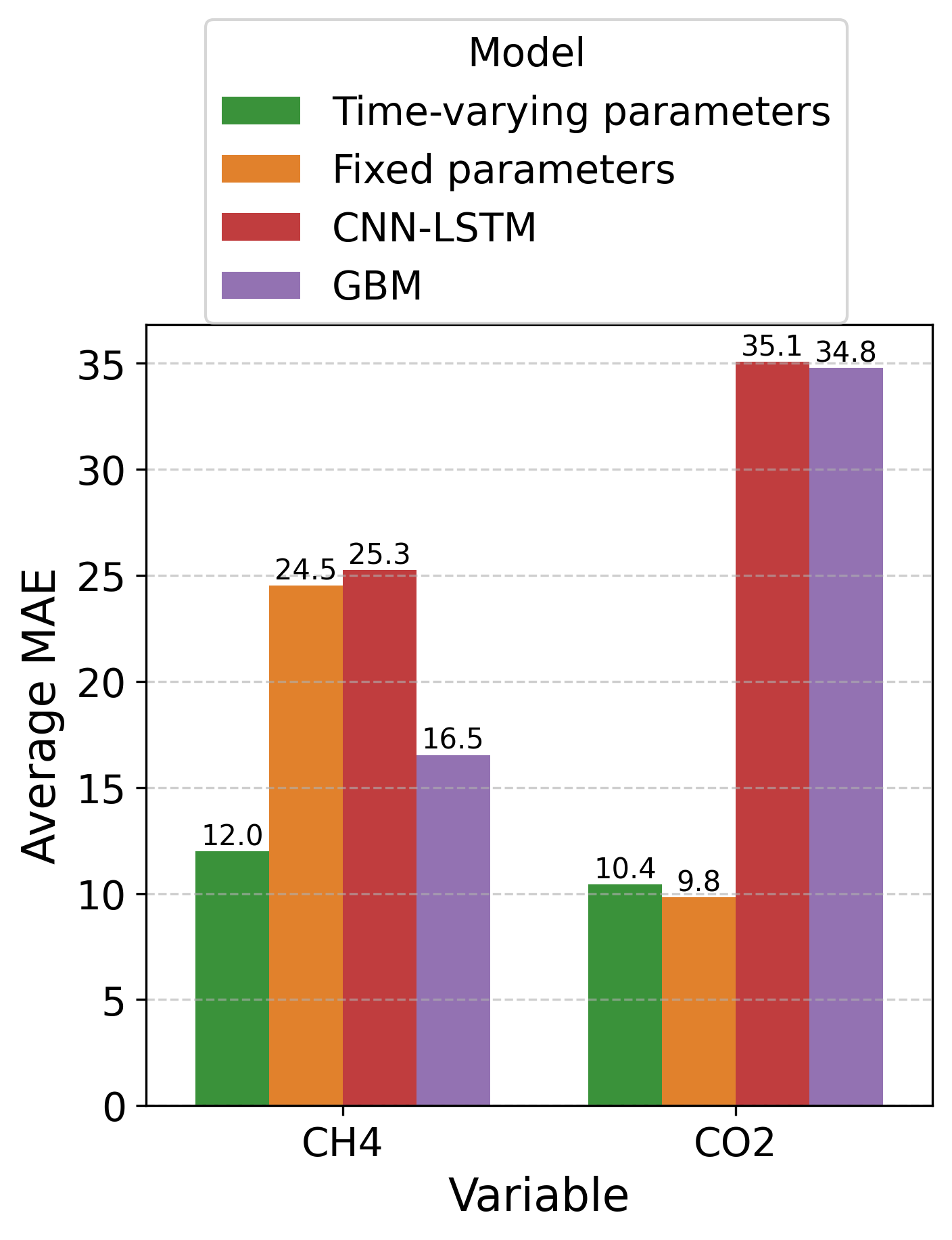}
        \put(-225,210){\textbf{(c)}}
    \end{subfigure}
    \hfill
    \begin{subfigure}{0.49\textwidth}
        \centering
        \includegraphics[width=\linewidth]{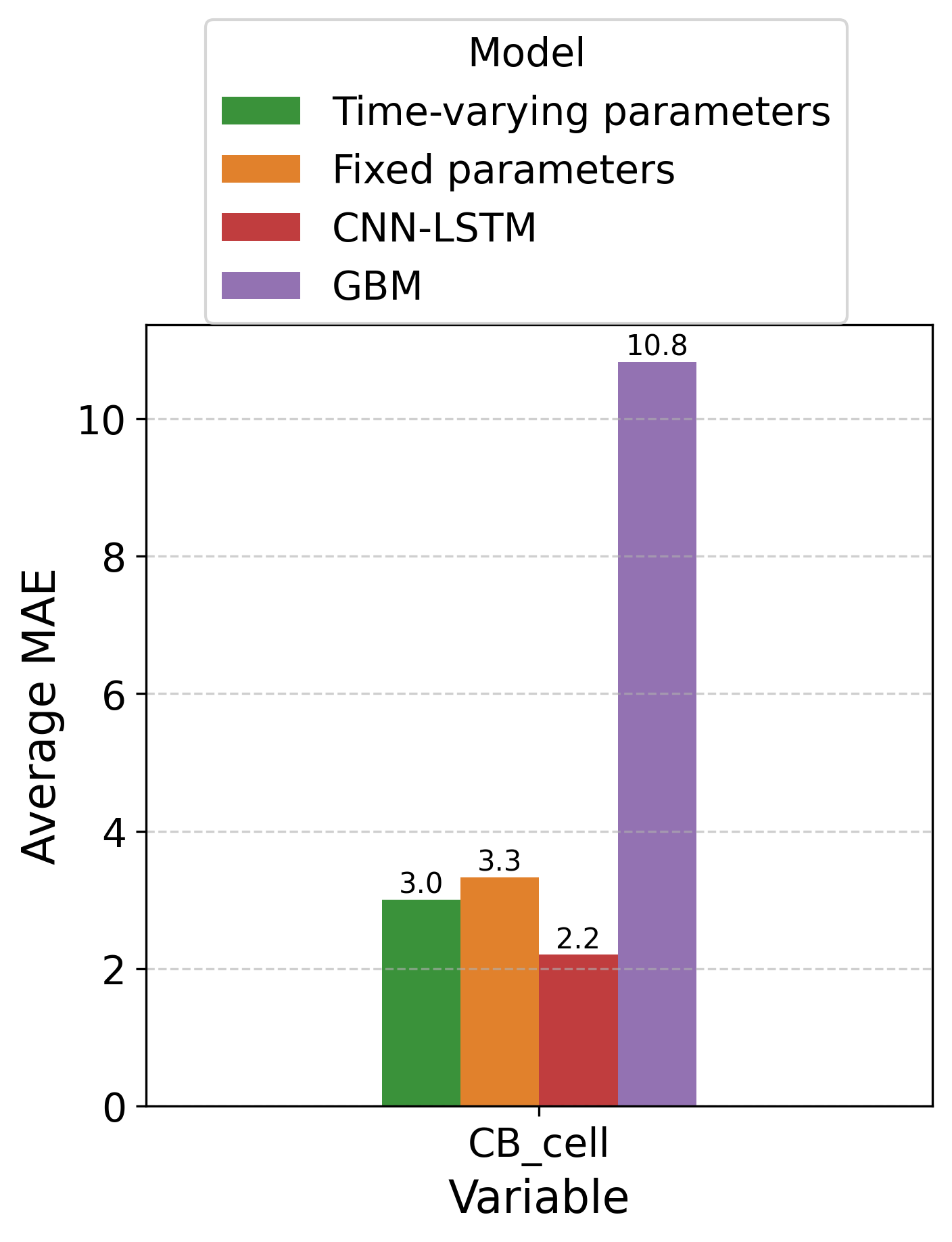}
        \put(-225,210){\textbf{(d)}}
    \end{subfigure}

    \caption{Average forecasting MAEs using the expanding-window cross-validation approach for the (a) SIR dataset, (b) CR dataset, (c) gases dataset, and (d) CB dataset. Green bars represent the time-varying parameter model, orange bars represent the fixed-parameter model, red bars represent the CNN-LSTM model, and purple bars represent the GBM model. MAEs were computed for each variable across different noise levels or monitoring stations within each dataset and then averaged.}
    \label{test_mae_comparison}
\end{figure}

\subsubsection{Optimal forecasting beyond cross-validation}

The expanding-window cross-validation approach did not always capture the most effective set of hyperparameters, namely the interval length and number of time-varying parameters, for the time-varying parameter model. For each dataset, there existed distinct configurations that produced superior predictive performance (SI Appendix Table~\ref{tab:CR}, \ref{tab:SIR}, \ref{tab:gases}, \ref{tab:CB}). 

In the SIR dataset, the optimal configuration achieved a MAE of 2.5\% for the infected ($I$) population and 2.2\% for the susceptible ($S$) population~(Figure \ref{test_optimal_mae_comparison}a)--an improvement of approximately 8\% compared to the hyperparameters of the expanding-window cross-validation setup. In the CR dataset, the model yielded MAEs of 1.3\% for the consumer ($C$) and 2.7\% for the resource ($R$) dynamics~(Figure \ref{test_optimal_mae_comparison}b), reflecting an enhancement of around 2\%. 

Similarly, for the gases dataset, the optimal configuration produced MAEs of 5.7\% for $\mathrm{CH_4}$ and 2.6\% for $\mathrm{CO_2}$~(Figure \ref{test_optimal_mae_comparison}c), marking an improvement of roughly 5\% over the expanding-window cross-validation approach. Finally, in the CB dataset, the best-performing configuration resulted in an MAE of 0.5\%~(Figure \ref{test_optimal_mae_comparison}d), which represents a 2.5\% gain in predictive accuracy relative to the expanding-window cross-validation method. 

Collectively, these results indicate that while the expanding-window cross-validation framework provides a robust baseline, fine-tuning the interval length and number of time-varying parameters can lead to noticeable improvements in predictive precision across diverse dynamical systems.

Sample forecasting plots demonstrate that the time-varying parameter model with optimal configuration consistently tracks the observed trajectories across all variables in the case studies, resulting in markedly improved forecasting performance~(SI Appendix Figure \ref{sample_plot_forecasting}). In contrast, the fixed-parameter model produces forecasts that only moderately match the time series of $R$, and $\mathrm{CH_4}$, and fails to capture the evolving dynamics in the remaining variables~(SI Appendix Figure \ref{sample_plot_forecasting}).

\begin{figure}[hbt!]
    \centering
    % First row
    \begin{subfigure}{0.49\textwidth}
        \includegraphics[width=\linewidth]{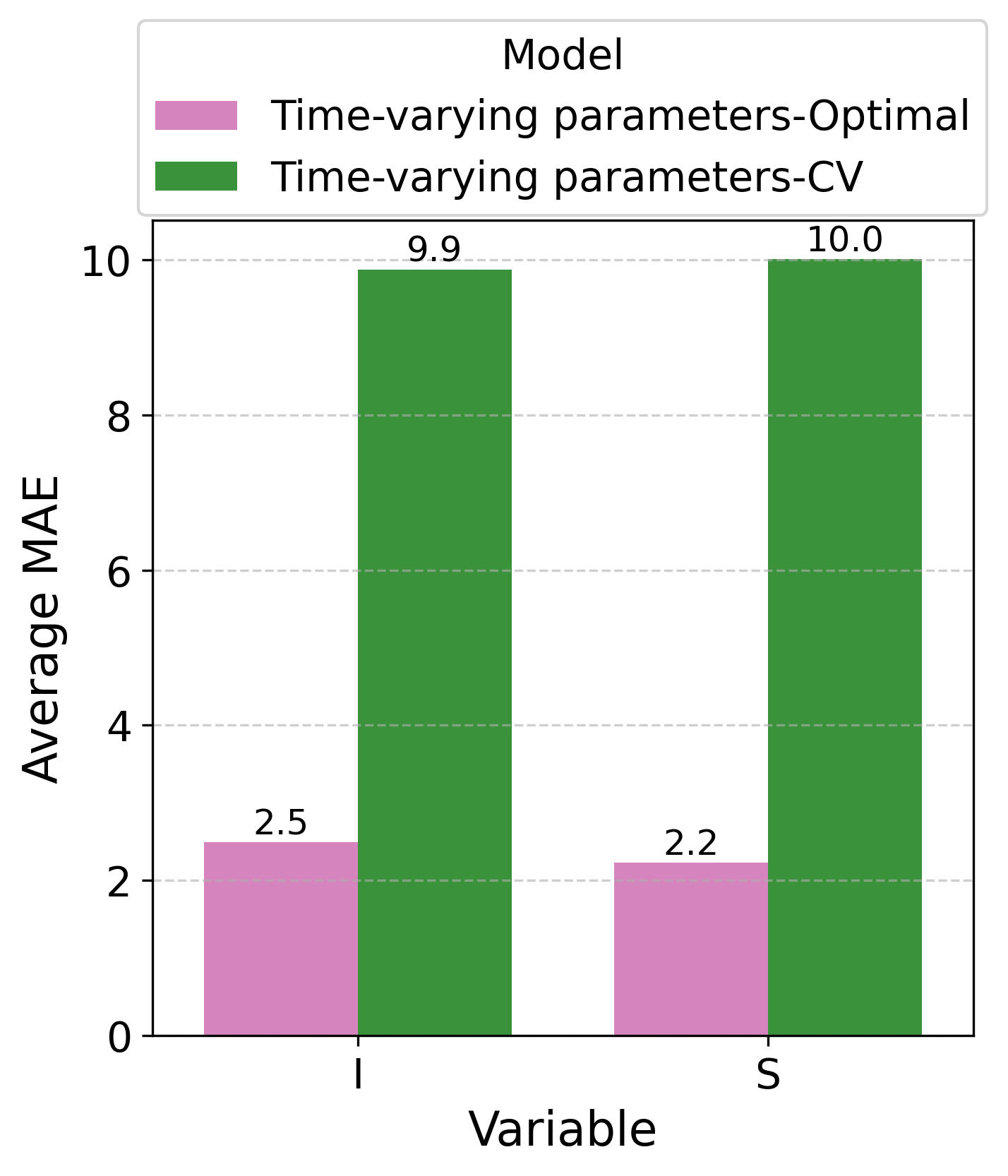}
        \put(-225,210){\textbf{(a)}}
    \end{subfigure}
    \hfill
    \begin{subfigure}{0.49\textwidth}
        \centering
        \includegraphics[width=\linewidth]{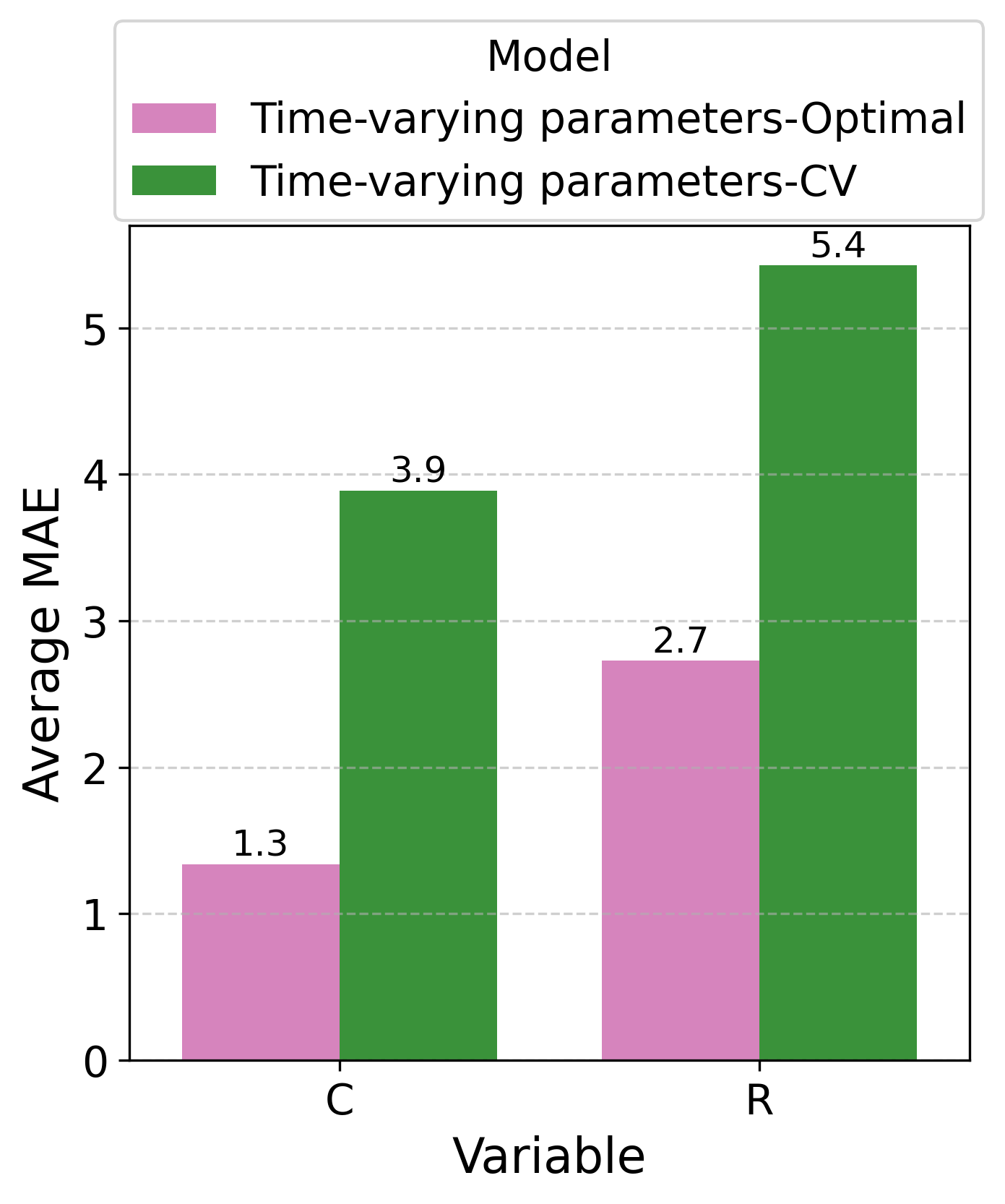}
        \put(-225,210){\textbf{(b)}}
    \end{subfigure}

    \vspace{0.0cm}

    % Second row
    \begin{subfigure}{0.49\textwidth}
        \centering
        \includegraphics[width=\linewidth]{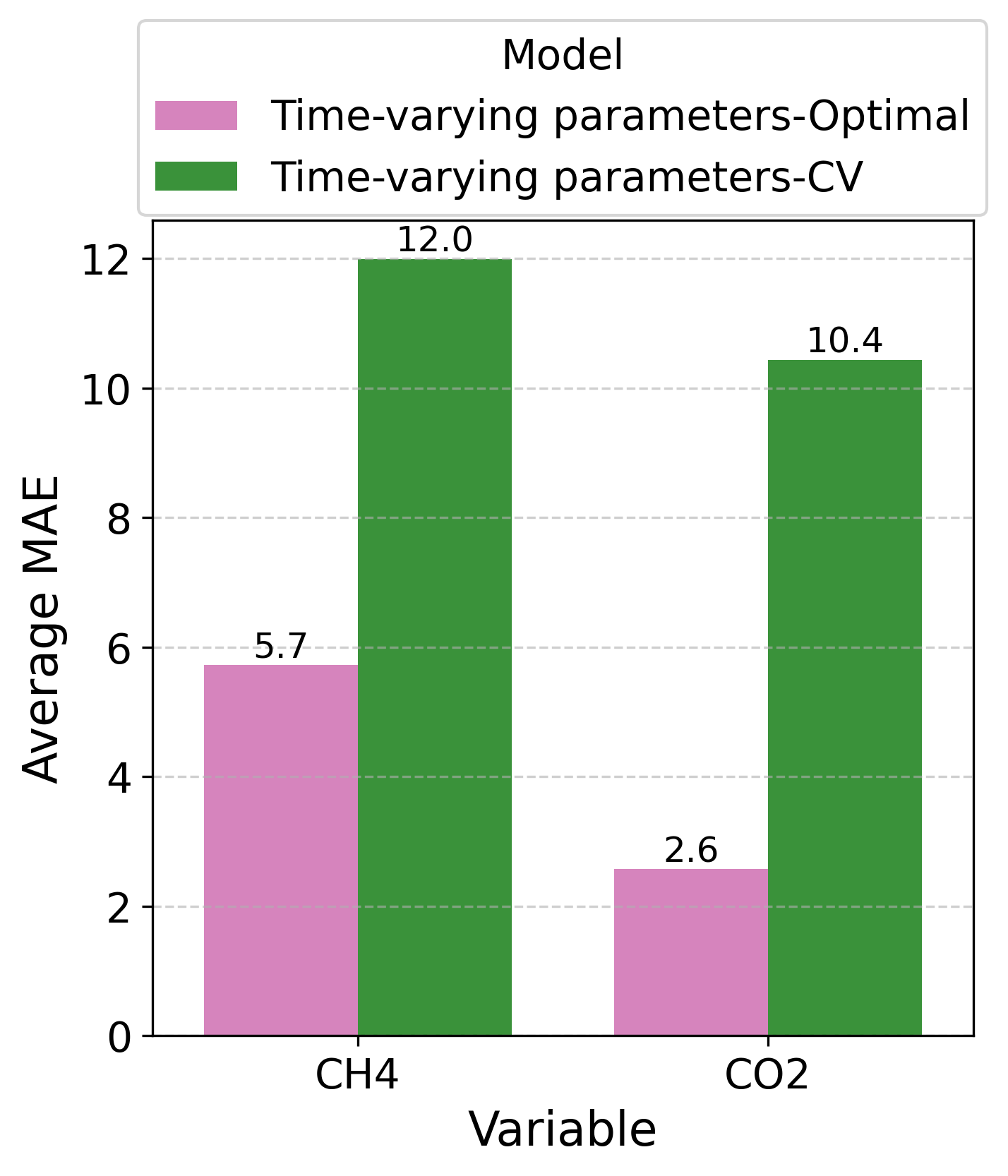}
        \put(-225,210){\textbf{(c)}}
    \end{subfigure}
    \hfill
    \begin{subfigure}{0.49\textwidth}
        \centering
        \includegraphics[width=\linewidth]{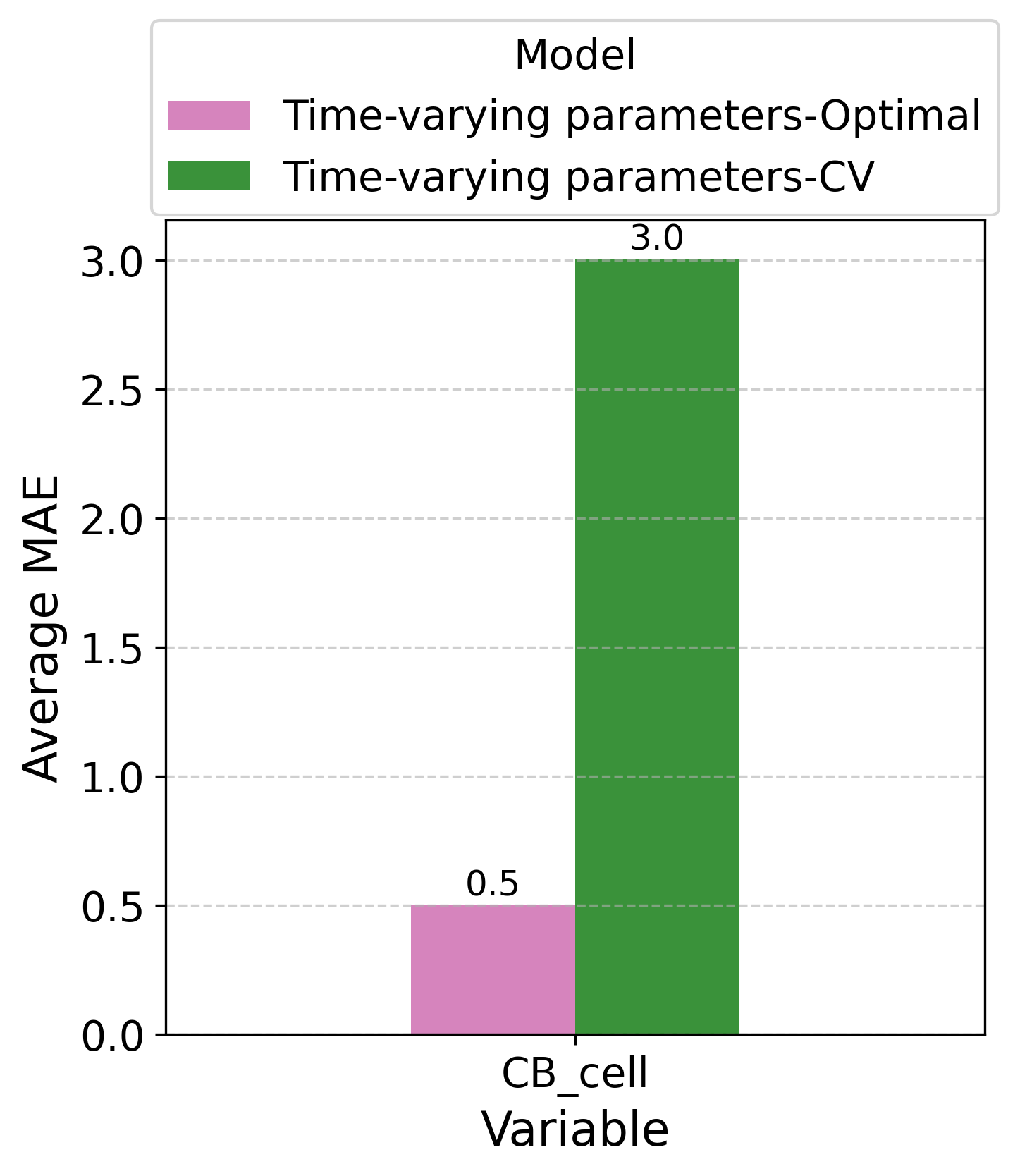}
        \put(-225,210){\textbf{(d)}}
    \end{subfigure}

    \caption{Average forecasting MAEs by the time-varying parameter model based on optimal configuration and cross-validation for the (a) SIR dataset, (b) CR dataset, (c) gases dataset, and (d) CB dataset. Green bars represent the error by the model based on the cross-validation (CV) method, and pink bars represent the optimal configuration forecasting error of the model. MAEs were computed for each variable across different noise levels or monitoring stations within each dataset and then averaged.}
    \label{test_optimal_mae_comparison}
\end{figure}

\FloatBarrier
\section{Discussion}
We present a data-driven framework that couples sparse equation discovery with machine-learned parameter evolution to enable adaptive learning and forecasting of dynamical systems. The results show that allowing a subset of coefficients to vary over time improves both learning fidelity and short-term predictive accuracy, particularly for non-stationary processes such as epidemic transmission and environmental gas emissions. Compared with fixed-coefficient sparse regression, the proposed model reduced forecasting MAE by 2-7\% across most datasets and maintained training errors below 3\%, indicating that temporal parameterization can enhance model expressiveness without sacrificing interpretability.

Traditional ML models are widely used in epidemiological and environmental prediction \cite{hamrani2020machine, guo2017developing, saha2025dispersion}. While these models often achieve strong point forecasts, their black-box structure limits mechanistic interpretation \cite{rubio2022modeling, li2019discovering}. In our experiments, two common data-driven forecasters--LSTM-CNN and GBM--produced substantially higher forecasting MAEs and showed greater sensitivity to noise, indicating that such models alone may struggle to generalize when system mechanisms evolve, or external drivers are only partially informative. Rather than viewing this as a weakness of ML itself, this outcome highlights the importance of inductive structure: embedding ML-predicted evolving coefficients into discovered differential equations can improve both interpretability and forecasting stability \cite{chakraborty2024policy, wang2022policy}. Our framework operationalizes this by coupling sparse equation learners with RF-predicted time-varying parameters, enabling adaptive forecasts while preserving dynamical transparency. This integration is especially valuable for evolving natural systems-including influenza transmission, climate-forced ecological interactions, atmospheric gas emissions, and cyanobacteria dynamics-where fixed-coefficient or unconstrained ML forecasters may fail to capture transient or non-stationary responses \cite{lowen2014roles, laws2017climate, small2015emissions, tseng2025apply}.

The simulated SIR and CR systems provided a controlled comparison between intrinsically non-stationary and approximately stationary parameter regimes. While our framework recovered accurate dynamics for both systems, the forecasting benefit of time-varying coefficients was substantially more pronounced in SIR. This outcome is consistent with theoretical expectations: CR trajectories, driven by nearly constant interaction strengths, can be reasonably approximated by fixed coefficients over short horizons, whereas SIR evolution is strongly governed by non-stationary drivers-most critically, a temporally shifting transmission rate. Consequently, the fixed-coefficient sparse regression baseline exhibited large forecasting error for SIR, reinforcing the importance of temporal parameterization \cite{chakraborty2024policy, wang2023discrete}. Similar advantages of time-varying formulations have been reported in epidemiological and ecological models under seasonal forcing and external perturbations \cite{spannaus2022inferring, song2013time}.

In the empirical datasets, the time-varying model achieved substantially lower errors in both learning and forecasting the $\mathrm{CO_2}$ and $\mathrm{CH_4}$. The fixed-coefficient sparse model struggled especially with forecasting $\mathrm{CH_4}$, and the LSTM-CNN and GBM baselines struggled with both, reflecting the influence of seasonal cycles and emission fluctuations on atmospheric gas dynamics. Because environmental systems shift continuously, capturing their behavior requires time-varying representations \cite{ornik2021learning}. In the CB dataset, both models performed similarly during learning, suggesting that some systems can be approximated with fixed parameters; however, the time-varying model consistently outperformed in forecasting. Collectively, these results support a constructive conclusion: equation learners that incorporate time-varying parameters yield more noise-robust and more accurate finite-horizon forecasts than fixed-coefficient sparse regression or standalone ML forecasters

In noise-free settings, both sparse models achieved low forecasting MAE, confirming that fixed-coefficient representations can be adequate when dynamics are nearly stationary and data are clean. However, performance diverged under noise and transient dynamical regimes. The time-varying formulation remained stable across both SIR and CR learning intervals, while the fixed-coefficient model exhibited high variance in CR folds, producing substantially larger learning error. Interestingly, the fixed-coefficient model exhibited high learning MAE at certain CR noise levels but lower short-horizon prediction error, as its coefficients converged to long-run averages that smoothed local fluctuations. While such averaging can appear competitive over small forecast windows, it obscures derivative-level dynamics and limits regime-specific generalization. In contrast, learning coefficients on short intervals improves dynamical expressiveness and reduces cross-fold error, but forecasting accuracy can still degrade if the prediction horizon is not well-aligned with the timescale of coefficient evolution.

Forecast accuracy in our framework is conditioned on the predictive skill of the time-varying parameters, which in this study were estimated using a single RF model. Although RF effectively captures nonlinear feature-parameter relationships, forecasting models with evolving dynamics may benefit from predictors with complementary inductive biases or stronger temporal adaptation \cite{ahmed2010empirical, muter2024future}. Neural networks, Bayesian networks, and hybrid learners have shown promise in modeling complex, noisy, and transient feature-response interactions \cite{chakraborty2024policy, wang2022policy}. Future work should therefore explore multi-baseline or ensemble-based parameter predictors to improve generalization and reduce propagation drift under noise, particularly for rapidly evolving or seasonally forced systems \cite{chakraborty2025deep, chakraborty2024early}.

Although expanding-window cross-validation is standard for hyperparameter selection, it may not identify the best interval length or the most effective number of time-varying parameters, as it optimizes for average fold performance rather than segment-specific predictive skill. Our fold-specific evaluation revealed configurations that consistently outperformed the cross-validation-selected model--most notably in the SIR system, where optimal interval alignment reduced forecasting MAE by approximately 8\% for both $S$ and $I$, with additional 2-5\% gains observed in the CR, gases, and CB datasets. We further observed that cross-validation error increased more sharply in the presence of noise, a known limitation of averaging-based validation that can obscure locally optimal configurations \cite{bates2024cross, bergmeir2018note}. These findings motivate future work on forecast-aware validation objectives, noise-weighted model selection, and adaptive interval tuning for transient or externally forced dynamics.

Candidate terms were initialized using domain knowledge to ensure a well-posed and interpretable library. For simulated systems (SIR and CR), where the governing interactions are known, state variables were used directly as candidate functions. For empirical datasets (gases and cyanobacteria, CB), the library was expanded to include measured source variables and weather covariates anticipated to modulate dynamics. Although this design supports strong performance, further reductions in library size are often possible without degrading forecast skill. Causal selection strategies-such as restricting predictors to the Markov blanket in a Bayesian network \cite{heggerud2024predicting}--can substantially lower dimensionality while preserving short-horizon accuracy. Because the cost of learning and predicting evolving coefficients scales with the number of active terms, library reduction also bounds the number of time-varying parameters, improving computational efficiency and robustness. Our framework readily accommodates such dataset-specific library adaptation, enabling future integration of automated causal or correlation-based term selection to further enhance scalability and generalization.

% noise vs prediction and learning change  

Although the framework achieved strong learning and forecasting accuracy, the analysis also identifies several opportunities for improvement. First, expanding-window cross-validation, while standard, may miss locally optimal choices of interval length and time-varying parameter counts because it prioritizes average fold performance rather than horizon-specific forecast skill. Second, autoregressive ML forecasters exhibit error accumulation over longer horizons, especially under noise or rapid regime shifts, a challenge that can be mitigated through dynamical constraints and noise-aware validation objectives. Third, parameter forecasting depends on exogenous drivers (here, four weather covariates); limited variability or measurement error in these inputs can reduce coefficient-forecast reliability, motivating future work on adaptive or expanded driver sets. Fourth, the selection of time-varying parameters is based on top correlation (with the bias term always treated as time-varying), which may introduce selection bias and can fail when correlations are transient, noisy, or do not reflect true causal influence. Finally, the candidate library was initialized using prior system knowledge, which ensured interpretability but may exclude latent or unobserved drivers; future extensions should investigate automated causal or correlation-based library reduction and selection to enhance scalability and generalization. Collectively, these considerations highlight areas for careful model design and future methodological improvements.

\section*{Data availability}
All the data and codes used in this study are available in a Zenodo repository \cite{chakraborty_2026_18473544}.

\section*{Funding}
Hao Wang's research was partially supported by the Natural Sciences and Engineering Research Council of Canada (Individual Discovery Grant RGPIN-2020-03911 and Discovery Accelerator Supplement Award RGPAS-2020-00090) and the Canada Research Chairs program (Tier 1 Canada Research Chair Award). Pouria Ramazi also acknowledges the support from an NSERC Discovery Grant.

\section*{Acknowledgement}
We would like to thank ILMEE members for providing valuable feedback. We polished spelling, grammar, and general style of this paper using ChatGPT \cite{openai2023chatgpt}. Computational resources were provided by the Digital Research Alliance of Canada (formerly Compute Canada).

\section*{Competing interests} We declare we have no competing interests.

\bibliographystyle{unsrt}
\bibliography{2references}

\begin{thebibliography}{10}

\bibitem{baker2018mechanistic}
Ruth~E Baker, Jose-Maria Pena, Jayaratnam Jayamohan, and Antoine J{\'e}rusalem.
\newblock Mechanistic models versus machine learning, a fight worth fighting for the biological community?
\newblock {\em Biology letters}, 14(5):20170660, 2018.

\bibitem{lan2025shallow}
Yunduo Lan, Sung-Young Shin, and Lan~K Nguyen.
\newblock From shallow to deep: the evolution of machine learning and mechanistic model integration in cancer research.
\newblock {\em Current Opinion in Systems Biology}, 40:100541, 2025.

\bibitem{trappe2023density}
Martin-I Trappe and Ryan~A Chisholm.
\newblock A density functional theory for ecology across scales.
\newblock {\em Nature Communications}, 14(1):1089, 2023.

\bibitem{small2015emissions}
Christina~C Small, Sunny Cho, Zaher Hashisho, and Ania~C Ulrich.
\newblock Emissions from oil sands tailings ponds: Review of tailings pond parameters and emission estimates.
\newblock {\em Journal of Petroleum Science and Engineering}, 127:490--501, 2015.

\bibitem{natchimuthu2014influence}
Sivakiruthika Natchimuthu, Balathandayuthabani Panneer~Selvam, and David Bastviken.
\newblock Influence of weather variables on methane and carbon dioxide flux from a shallow pond.
\newblock {\em Biogeochemistry}, 119:403--413, 2014.

\bibitem{lowen2007influenza}
Anice~C Lowen, Samira Mubareka, John Steel, and Peter Palese.
\newblock Influenza virus transmission is dependent on relative humidity and temperature.
\newblock {\em PLOS pathogens}, 3(10):e151, 2007.

\bibitem{shaman2009absolute}
Jeffrey Shaman and Melvin Kohn.
\newblock Absolute humidity modulates influenza survival, transmission, and seasonality.
\newblock {\em Proceedings of the National Academy of Sciences}, 106(9):3243--3248, 2009.

\bibitem{chakraborty2024policy}
Amit~K Chakraborty, Hao Wang, and Pouria Ramazi.
\newblock From policy to prediction: Assessing forecasting accuracy in an integrated framework with machine learning and disease models.
\newblock {\em Journal of Computational Biology}, 31(11):1104--1117, 2024.

\bibitem{wang2022policy}
Xiunan Wang, Hao Wang, Pouria Ramazi, Kyeongah Nah, and Mark Lewis.
\newblock From policy to prediction: Forecasting covid-19 dynamics under imperfect vaccination.
\newblock {\em Bulletin of Mathematical Biology}, 84(9):90, 2022.

\bibitem{banga2025mechanistic}
Julio~R Banga and Alejandro~F Villaverde.
\newblock Mechanistic dynamic modelling of biological systems: The road ahead.
\newblock {\em Current Opinion in Systems Biology}, 42:100553, 2025.

\bibitem{song2024towards}
Wenxiang Song, Shijie Jiang, Gustau Camps-Valls, Mathew Williams, Lu~Zhang, Markus Reichstein, Harry Vereecken, Leilei He, Xiaolong Hu, and Liangsheng Shi.
\newblock Towards data-driven discovery of governing equations in geosciences.
\newblock {\em Communications Earth \& Environment}, 5(1):589, 2024.

\bibitem{brunton2016discovering}
Steven~L Brunton, Joshua~L Proctor, and J~Nathan Kutz.
\newblock Discovering governing equations from data by sparse identification of nonlinear dynamical systems.
\newblock {\em Proceedings of the National Academy of Sciences}, 113(15):3932--3937, 2016.

\bibitem{rudy2017data}
Samuel~H Rudy, Steven~L Brunton, Joshua~L Proctor, and J~Nathan Kutz.
\newblock Data-driven discovery of partial differential equations.
\newblock {\em Science advances}, 3(4):e1602614, 2017.

\bibitem{li2019discovering}
Shanwu Li, Eurika Kaiser, Shujin Laima, Hui Li, Steven~L Brunton, and J~Nathan Kutz.
\newblock Discovering time-varying aerodynamics of a prototype bridge by sparse identification of nonlinear dynamical systems.
\newblock {\em Physical Review E}, 100(2):022220, 2019.

\bibitem{fukami2021sparse}
Kai Fukami, Takaaki Murata, Kai Zhang, and Koji Fukagata.
\newblock Sparse identification of nonlinear dynamics with low-dimensionalized flow representations.
\newblock {\em Journal of Fluid Mechanics}, 926:A10, 2021.

\bibitem{callaham2022role}
Jared~L Callaham, Steven~L Brunton, and Jean-Christophe Loiseau.
\newblock On the role of nonlinear correlations in reduced-order modelling.
\newblock {\em Journal of Fluid Mechanics}, 938:A1, 2022.

\bibitem{callaham2022empirical}
Jared~L Callaham, Georgios Rigas, Jean-Christophe Loiseau, and Steven~L Brunton.
\newblock An empirical mean-field model of symmetry-breaking in a turbulent wake.
\newblock {\em Science Advances}, 8(19):eabm4786, 2022.

\bibitem{loiseau2020data}
Jean-Christophe Loiseau.
\newblock Data-driven modeling of the chaotic thermal convection in an annular thermosyphon.
\newblock {\em Theoretical and Computational Fluid Dynamics}, 34(4):339--365, 2020.

\bibitem{loiseau2018constrained}
Jean-Christophe Loiseau and Steven~L Brunton.
\newblock Constrained sparse galerkin regression.
\newblock {\em Journal of Fluid Mechanics}, 838:42--67, 2018.

\bibitem{prokop2024biological}
Bartosz Prokop and Lendert Gelens.
\newblock From biological data to oscillator models using sindy.
\newblock {\em Iscience}, 27(4), 2024.

\bibitem{mangan2016inferring}
Niall~M Mangan, Steven~L Brunton, Joshua~L Proctor, and J~Nathan Kutz.
\newblock Inferring biological networks by sparse identification of nonlinear dynamics.
\newblock {\em IEEE Transactions on Molecular, Biological, and Multi-Scale Communications}, 2(1):52--63, 2016.

\bibitem{kaiser2018sparse}
Eurika Kaiser, J~Nathan Kutz, and Steven~L Brunton.
\newblock Sparse identification of nonlinear dynamics for model predictive control in the low-data limit.
\newblock {\em Proceedings of the Royal Society A}, 474(2219):20180335, 2018.

\bibitem{wu2024data}
Xiaojun Wu, MeiLu McDermott, and Adam~L MacLean.
\newblock Data-driven model discovery and model selection for noisy biological systems.
\newblock {\em bioRxiv}, pages 2024--10, 2024.

\bibitem{prokop2023data}
Bartosz Prokop and Lendert Gelens.
\newblock Data-driven discovery of oscillator models using sindy: Towards the application on experimental data in biology.
\newblock {\em bioRxiv}, pages 2023--08, 2023.

\bibitem{sandoz2023sindy}
Antoine Sandoz, Verena Ducret, Georg~A Gottwald, Gilles Vilmart, and Karl Perron.
\newblock Sindy for delay-differential equations: application to model bacterial zinc response.
\newblock {\em Proceedings of the Royal Society A}, 479(2269):20220556, 2023.

\bibitem{prabhu2023data}
Siddharth Prabhu, Srinivas Rangarajan, and Mayuresh Kothare.
\newblock Data-driven discovery of sparse dynamical model of cardiovascular system for model predictive control.
\newblock {\em Computers in Biology and Medicine}, 166:107513, 2023.

\bibitem{thiele2020system}
Gregor Thiele, Arne Fey, David Sommer, and J{\"o}rg Kr{\"u}ger.
\newblock System identification of a hysteresis-controlled pump system using sindy.
\newblock In {\em 2020 24th International Conference on System Theory, Control and Computing (ICSTCC)}, pages 457--464. IEEE, 2020.

\bibitem{wulff2024minimal}
Paul Wulff, Nils Gr{\"a}bner, and Utz von Wagner.
\newblock Minimal model identification of drum brake squeal via sindy.
\newblock {\em Archive of Applied Mechanics}, 94(10):3101--3117, 2024.

\bibitem{lee2024energy}
Hao Lee, Ruoning Ren, Yifei Qian, and Jacob Rosen.
\newblock Energy reduction for wearable pneumatic valve system with sindy and time-variant model predictive control.
\newblock {\em IEEE/ASME Transactions on Mechatronics}, 2024.

\bibitem{guo2024uncertainty}
Lin Guo, Xiaokai Yang, Zhonghua Zheng, Nicole Riemer, and Christopher~W Tessum.
\newblock Uncertainty quantification in reduced-order gas-phase atmospheric chemistry modeling using ensemble sindy.
\newblock {\em arXiv preprint arXiv:2407.09757}, 2024.

\bibitem{rubio2022modeling}
Javier Rubio-Herrero, Carlos~Ortiz Marrero, and Wai-Tong~Louis Fan.
\newblock Modeling atmospheric data and identifying dynamics temporal data-driven modeling of air pollutants.
\newblock {\em Journal of Cleaner Production}, 333:129863, 2022.

\bibitem{yang2024atmospheric}
Xiaokai Yang, Lin Guo, Zhonghua Zheng, Nicole Riemer, and Christopher~W Tessum.
\newblock Atmospheric chemistry surrogate modeling with sparse identification of nonlinear dynamics.
\newblock {\em Journal of Geophysical Research: Machine Learning and Computation}, 1(2):e2024JH000132, 2024.

\bibitem{rudy2019data}
Samuel Rudy, Alessandro Alla, Steven~L Brunton, and J~Nathan Kutz.
\newblock Data-driven identification of parametric partial differential equations.
\newblock {\em SIAM Journal on Applied Dynamical Systems}, 18(2):643--660, 2019.

\bibitem{wang2023discrete}
Xiunan Wang and Hao Wang.
\newblock Discrete inverse method for extracting disease transmission rates from accessible infection data.
\newblock {\em SIAM Journal on Applied Mathematics}, 84(3):S336--S361, 2023.

\bibitem{bousquet2022deep}
Arthur Bousquet, William~H Conrad, Said~Omer Sadat, Nelli Vardanyan, and Youngjoon Hong.
\newblock Deep learning forecasting using time-varying parameters of the sird model for covid-19.
\newblock {\em Scientific Reports}, 12(1):3030, 2022.

\bibitem{long2021identification}
Jie Long, AQM Khaliq, and Khaled~M Furati.
\newblock Identification and prediction of time-varying parameters of covid-19 model: a data-driven deep learning approach.
\newblock {\em International Journal of Computer Mathematics}, 98(8):1617--1632, 2021.

\bibitem{ji2025hybrid}
Juping Ji, Shohel Ahmed, and Hao Wang.
\newblock A hybrid approach to study and forecast climate-sensitive norovirus infections in the usa.
\newblock {\em Journal of Theoretical Biology}, 598:112007, 2025.

\bibitem{pastpipatkul2024analysis}
Pathairat Pastpipatkul and Panicha Subsai.
\newblock Analysis of time-varying coefficients and forecasting effects between greenhouse gas emissions and its determinants in thailand.
\newblock In {\em Applications of Optimal Transport to Economics and Related Topics}, pages 591--604. Springer, 2024.

\bibitem{chakraborty2024early}
Amit~K Chakraborty, Shan Gao, Reza Miry, Pouria Ramazi, Russell Greiner, Mark~A Lewis, and Hao Wang.
\newblock An early warning indicator trained on stochastic disease-spreading models with different noises.
\newblock {\em Journal of the Royal Society Interface}, 21(217):20240199, 2024.

\bibitem{kermack1927contribution}
William~Ogilvy Kermack and Anderson~G McKendrick.
\newblock A contribution to the mathematical theory of epidemics.
\newblock {\em Proceedings of the Royal Society of London. Series A, Containing papers of a Mathematical and Physical character}, 115(772):700--721, 1927.

\bibitem{rosenzweig1963graphical}
Michael~L Rosenzweig and Robert~H MacArthur.
\newblock Graphical representation and stability conditions of predator-prey interactions.
\newblock {\em The American Naturalist}, 97(895):209--223, 1963.

\bibitem{kloeden1992numerical}
Peter~E. Kloeden and Eckhard Platen.
\newblock {\em Numerical Solution of Stochastic Differential Equations}, volume~23 of {\em Stochastic Modelling and Applied Probability}.
\newblock Springer, Berlin, Heidelberg, 1992.

\bibitem{wbeaNetworkStation}
{N}etwork {M}ap \& {S}tation {D}ata - {W}ood {B}uffalo {E}nvironmental {A}ssociation - wbea.org.
\newblock \url{https://wbea.org/data/network-map-station-data/}.
\newblock [Accessed 06-11-2024].

\bibitem{albertaenergyregulator}
{A}lberta {E}nergy {R}egulator. {A}lberta mineable oil sands plant statistics.
\newblock \url{https://www.aer.ca/data-and-performance-reports/statistical-reports/st39}.
\newblock [Accessed 01-23-2025].

\bibitem{cheggerud_2023_10109225}
Christopher~M Heggerud.
\newblock cheggerud/cb-prediction: Nov\_10\_2023.
\newblock \url{https://doi.org/10.5281/zenodo.10109225}, 2023.
\newblock [Accessed 24-01-2024].

\bibitem{heggerud2024predicting}
Christopher~M Heggerud, Jingjing Xu, Hao Wang, Mark~A Lewis, Ron~W Zurawell, Charlie~JG Loewen, Rolf~D Vinebrooke, and Pouria Ramazi.
\newblock Predicting imminent cyanobacterial blooms in lakes using incomplete timely data.
\newblock {\em Water Resources Research}, 60(2):e2023WR035540, 2024.

\bibitem{uhlenbeck1930theory}
George~E. Uhlenbeck and Leonard~S. Ornstein.
\newblock On the theory of the brownian motion.
\newblock {\em Physical Review}, 36(5):823--841, 1930.

\bibitem{breiman2001random}
Leo Breiman.
\newblock Random forests.
\newblock {\em Machine Learning}, 45:5--32, 2001.

\bibitem{pedregosa2011scikit}
Fabian Pedregosa et~al.
\newblock Scikit-learn: Machine learning in python.
\newblock {\em Journal of Machine Learning Research}, 12:2825--2830, 2011.

\bibitem{shi2015convolutional}
Xingjian Shi et~al.
\newblock Convolutional lstm network: A machine learning approach for precipitation nowcasting.
\newblock {\em Advances in Neural Information Processing Systems}, 28, 2015.

\bibitem{hochreiter1997long}
Sepp Hochreiter and J{\"u}rgen Schmidhuber.
\newblock Long short-term memory.
\newblock {\em Neural Computation}, 9:1735--1780, 1997.

\bibitem{bai2018empirical}
Shaojie Bai, J~Zico Kolter, and Vladlen Koltun.
\newblock An empirical evaluation of generic convolutional and recurrent networks for sequence modeling.
\newblock {\em arXiv preprint arXiv:1803.01271}, 2018.

\bibitem{friedman2001greedy}
Jerome~H Friedman.
\newblock Greedy function approximation: A gradient boosting machine.
\newblock {\em Annals of Statistics}, 29:1189--1232, 2001.

\bibitem{abadi2016tensorflow}
Mart{\'\i}n Abadi et~al.
\newblock Tensorflow: A system for large-scale machine learning.
\newblock {\em Proceedings of the 12th USENIX Symposium on Operating Systems Design and Implementation}, pages 265--283, 2016.

\bibitem{chollet2015keras}
Fran{\c{c}}ois Chollet.
\newblock Keras.
\newblock {\em GitHub repository}, 2015.
\newblock \url{https://keras.io}.

\bibitem{omalley2019kerastuner}
Tom O'Malley et~al.
\newblock Keras tuner, 2019.
\newblock \url{https://keras.io/keras_tuner}.

\bibitem{kingma2014adam}
Diederik~P Kingma and Jimmy Ba.
\newblock Adam: A method for stochastic optimization.
\newblock {\em arXiv preprint arXiv:1412.6980}, 2014.

\bibitem{prechelt1998early}
Lutz Prechelt.
\newblock Early stopping-but when?
\newblock {\em Neural Networks: Tricks of the Trade}, pages 55--69, 1998.

\bibitem{Teschl2012}
Gerald Teschl.
\newblock {\em Ordinary Differential Equations and Dynamical Systems}.
\newblock American Mathematical Society, 2012.

\bibitem{ZhangSchaeffer2019}
Zeyu Zhang and Hayden Schaeffer.
\newblock On the convergence of the sindy algorithm.
\newblock {\em Multiscale Modeling \& Simulation}, 17(3):948--972, 2019.

\bibitem{gronwall1919note}
Thomas~H. Gronwall.
\newblock Note on the derivatives with respect to a parameter of the solutions of a system of differential equations.
\newblock {\em Annals of Mathematics}, 20(4):292--296, 1919.

\bibitem{Rudin1987}
Walter Rudin.
\newblock {\em Real and Complex Analysis}.
\newblock McGraw--Hill, 3rd edition, 1987.

\bibitem{hamrani2020machine}
Abderrachid Hamrani, Abdolhamid Akbarzadeh, and Chandra~A Madramootoo.
\newblock Machine learning for predicting greenhouse gas emissions from agricultural soils.
\newblock {\em Science of The Total Environment}, 741:140338, 2020.

\bibitem{guo2017developing}
Pi~Guo, Tao Liu, Qin Zhang, Li~Wang, Jianpeng Xiao, Qingying Zhang, Ganfeng Luo, Zhihao Li, Jianfeng He, Yonghui Zhang, et~al.
\newblock Developing a dengue forecast model using machine learning: A case study in china.
\newblock {\em PLOS neglected tropical diseases}, 11(10):e0005973, 2017.

\bibitem{saha2025dispersion}
Esha Saha, Oscar Wang, Amit~K Chakraborty, Pablo~Venegas Garcia, Russell Milne, and Hao Wang.
\newblock Dispersion based recurrent neural network model for methane monitoring in albertan tailings ponds.
\newblock {\em Journal of Environmental Management}, 395:127748, 2025.

\bibitem{lowen2014roles}
Anice~C Lowen and John Steel.
\newblock Roles of humidity and temperature in shaping influenza seasonality.
\newblock {\em Journal of virology}, 88(14):7692--7695, 2014.

\bibitem{laws2017climate}
Angela~N Laws.
\newblock Climate change effects on predator--prey interactions.
\newblock {\em Current Opinion in Insect Science}, 23:28--34, 2017.

\bibitem{tseng2025apply}
Shih-Hsien Tseng, Pei~Han Feng, and Thi Ha~Trang Duong.
\newblock Apply data science and feature selection techniques to predict carbon dioxide emissions in taiwan.
\newblock {\em Stochastic Environmental Research and Risk Assessment}, pages 1--23, 2025.

\bibitem{spannaus2022inferring}
Adam Spannaus, Theodore Papamarkou, Samantha Erwin, and J~Blair Christian.
\newblock Inferring the spread of covid-19: the role of time-varying reporting rate in epidemiological modelling.
\newblock {\em Scientific Reports}, 12(1):10761, 2022.

\bibitem{song2013time}
Xiaodong Song, Brett~A Bryan, Auro~C Almeida, Keryn~I Paul, Gang Zhao, and Yin Ren.
\newblock Time-dependent sensitivity of a process-based ecological model.
\newblock {\em Ecological Modelling}, 265:114--123, 2013.

\bibitem{ornik2021learning}
Melkior Ornik and Ufuk Topcu.
\newblock Learning and planning for time-varying mdps using maximum likelihood estimation.
\newblock {\em Journal of Machine Learning Research}, 22(35):1--40, 2021.

\bibitem{ahmed2010empirical}
Nesreen~K Ahmed, Amir~F Atiya, Neamat~El Gayar, and Hisham El-Shishiny.
\newblock An empirical comparison of machine learning models for time series forecasting.
\newblock {\em Econometric reviews}, 29(5-6):594--621, 2010.

\bibitem{muter2024future}
Bushra~Majeed Muter and Ayat~Jasim Mohammed.
\newblock The future of ai: Assessing the strengths and limitations of deep learning and machine learning.
\newblock {\em Wisdom Journal For Studies \& Research}, 4(06):348--374, 2024.

\bibitem{chakraborty2025deep}
Amit~K Chakraborty, Reza Miry, Russell Greiner, Mark~A Lewis, Hao Wang, Tianyu Guan, and Pouria Ramazi.
\newblock Deep learning for disease outbreak prediction: a parallel {LSTM}-{CNN} model.
\newblock {\em Journal of the Royal Society Interface}, 22(229):20250046, 2025.

\bibitem{bates2024cross}
Stephen Bates, Trevor Hastie, and Robert Tibshirani.
\newblock Cross-validation: what does it estimate and how well does it do it?
\newblock {\em Journal of the American Statistical Association}, 119(546):1434--1445, 2024.

\bibitem{bergmeir2018note}
Christoph Bergmeir, Rob~J Hyndman, and Bonsoo Koo.
\newblock A note on the validity of cross-validation for evaluating autoregressive time series prediction.
\newblock {\em Computational Statistics \& Data Analysis}, 120:70--83, 2018.

\bibitem{chakraborty_2026_18473544}
Amit~K. Chakraborty.
\newblock Data and codes: Turning mechanistic models into forecasters by using machine learning, February 2026.

\bibitem{openai2023chatgpt}
{OpenAI}.
\newblock Chatgpt, 2023.
\newblock Large language model used for proofreading assistance.

\end{thebibliography}

% fix figure captions 

\section*{SI Appendix}

\begin{figure}[htbp]
    \centering
    % First row
    \begin{subfigure}{0.8\textwidth}
        \includegraphics[width=0.9\linewidth]{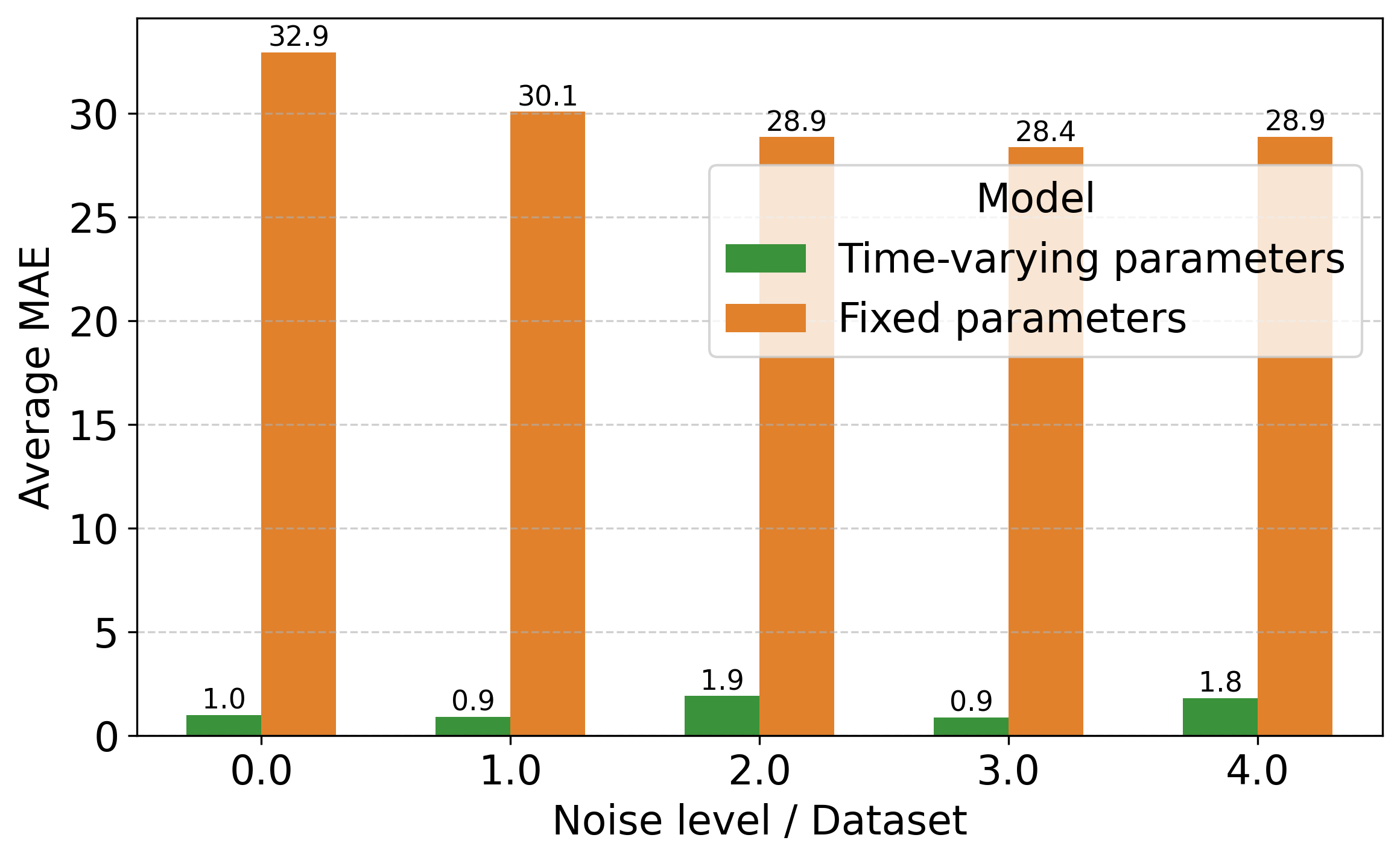}
        \put(-320,200){\textbf{(a)}}
    \end{subfigure}

    \vspace{0.0cm}

    % Second row
    \begin{subfigure}{0.8\textwidth}
        \includegraphics[width=0.9\linewidth]{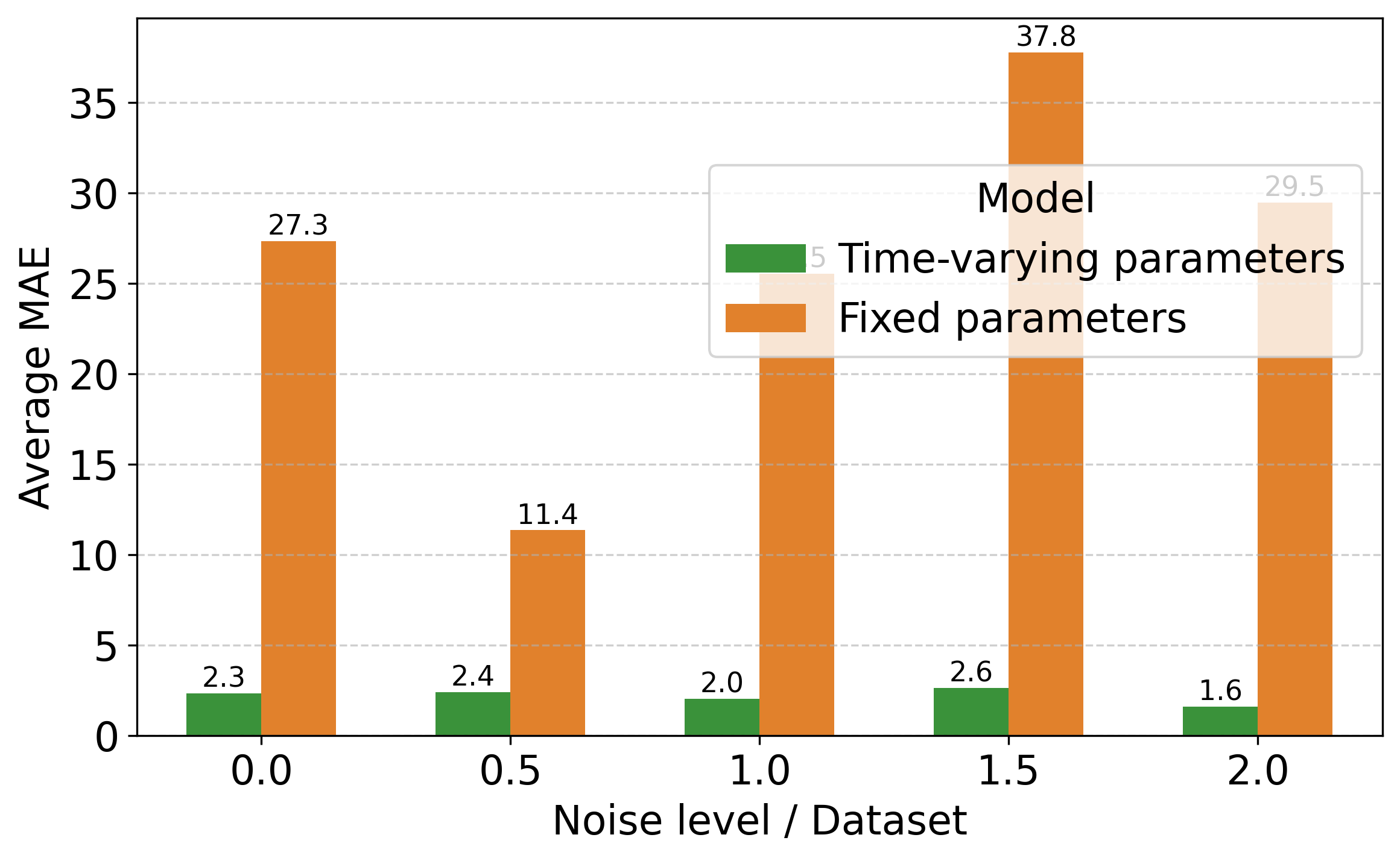}
        \put(-320,200){\textbf{(b)}}
    \end{subfigure}

        \vspace{0.0cm}

    % Second row
    \begin{subfigure}{0.8\textwidth}
        \includegraphics[width=0.9\linewidth]{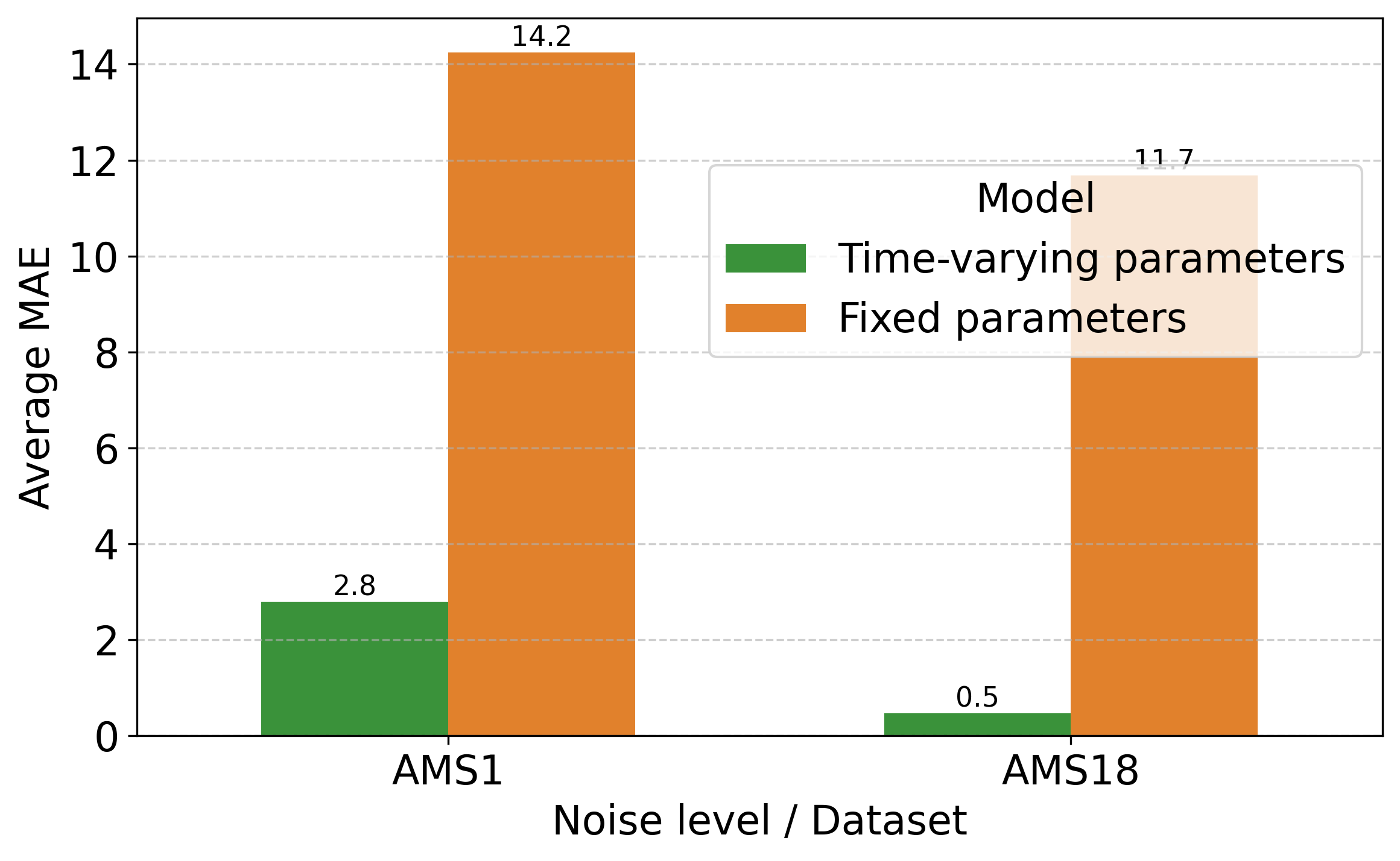}
        \put(-320,200){\textbf{(c)}}
    \end{subfigure}

    \caption{Average learning MAE across different noise in (a) SIR dataset, (b) CR dataset, and (c) gases dataset.}
    \label{train_mae_noise_comparison}
\end{figure}

\begin{figure}[htbp]
    \centering
    % First row
    \begin{subfigure}{0.8\textwidth}
        \includegraphics[width=0.9\linewidth]{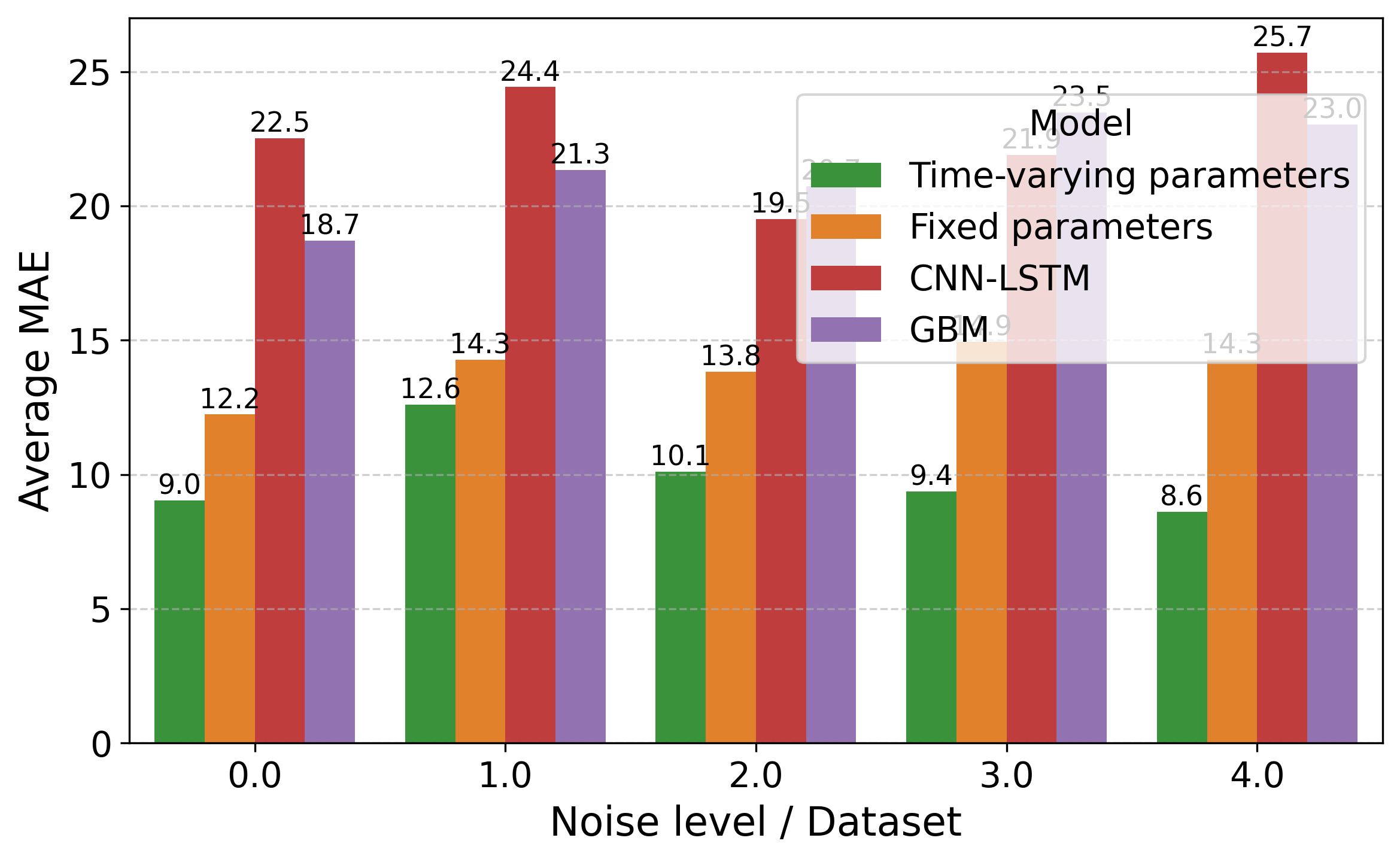}
        \put(-320,200){\textbf{(a)}}
    \end{subfigure}

    \vspace{0.0cm}

    % Second row
    \begin{subfigure}{0.8\textwidth}
        \includegraphics[width=0.9\linewidth]{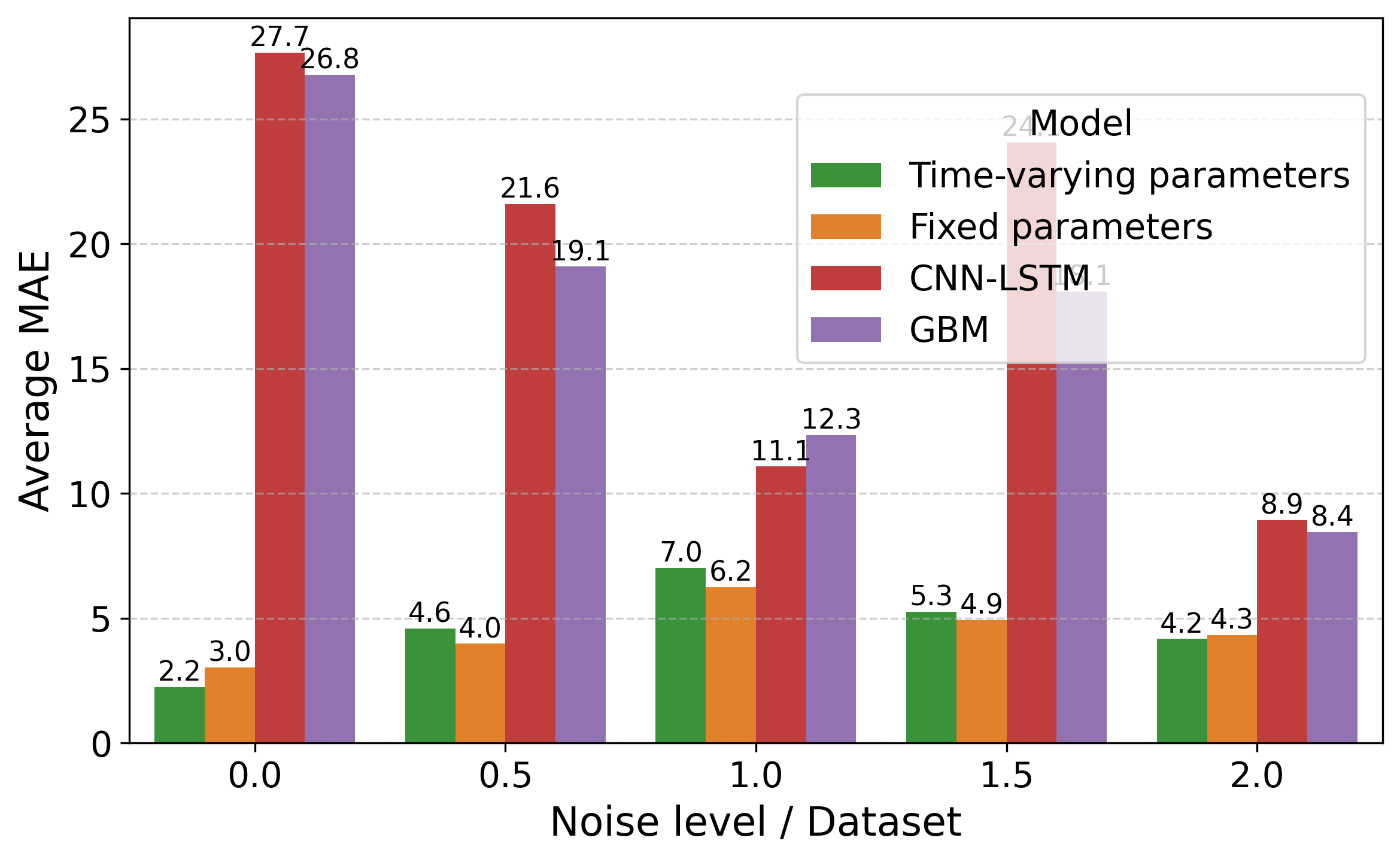}
        \put(-320,200){\textbf{(b)}}
    \end{subfigure}

        \vspace{0.0cm}

    % Second row
    \begin{subfigure}{0.8\textwidth}
        \includegraphics[width=0.9\linewidth]{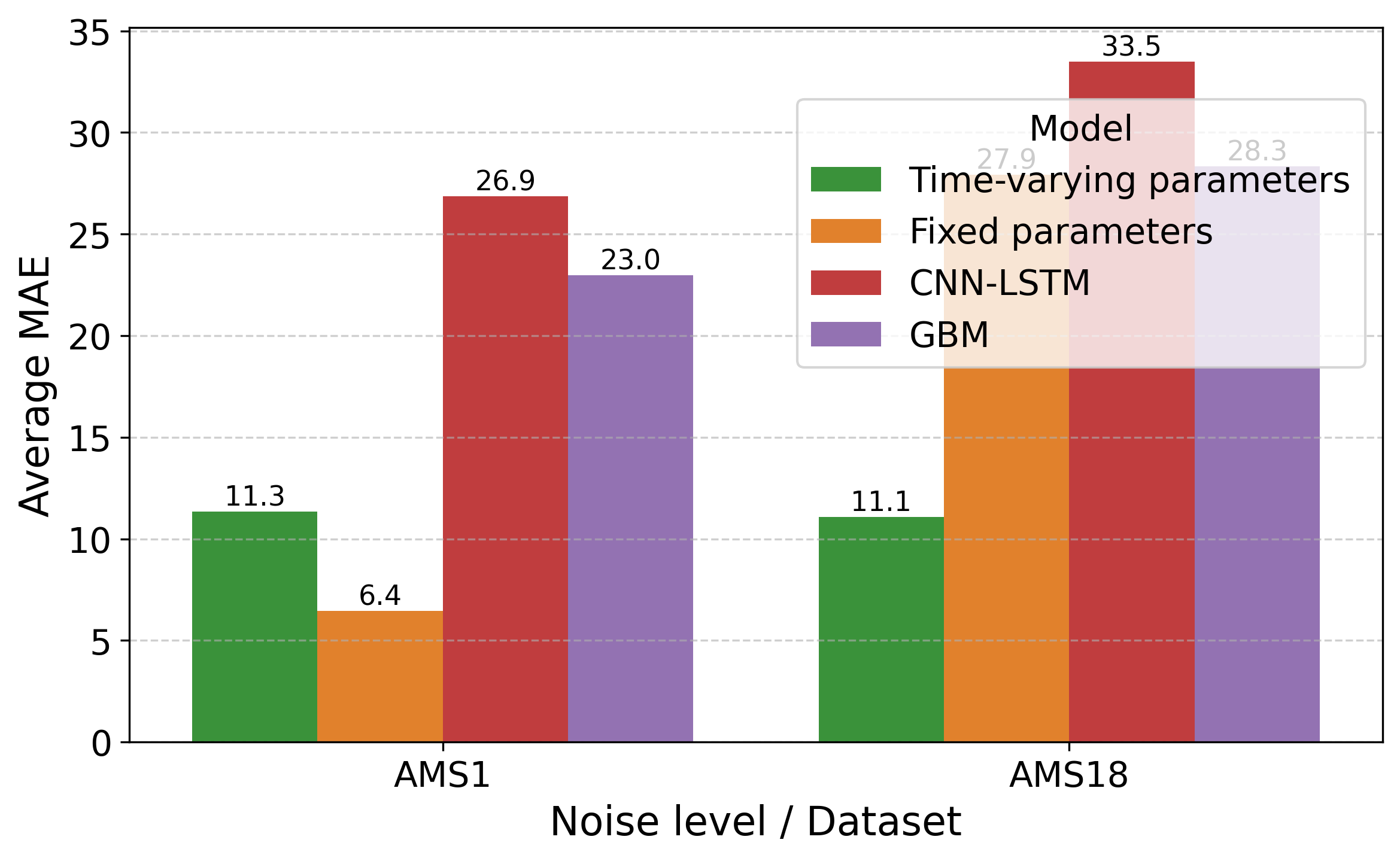}
        \put(-320,200){\textbf{(c)}}
    \end{subfigure}

    \caption{Average forecasting MAE across different noise in (a) SIR dataset, (b) CR dataset, and (c) gases dataset.}
    \label{test_mae_noise_comparison}
\end{figure}

\begin{figure}[htbp]
    \centering
    % First row
    \begin{subfigure}{0.8\textwidth}
        \includegraphics[width=0.9\linewidth]{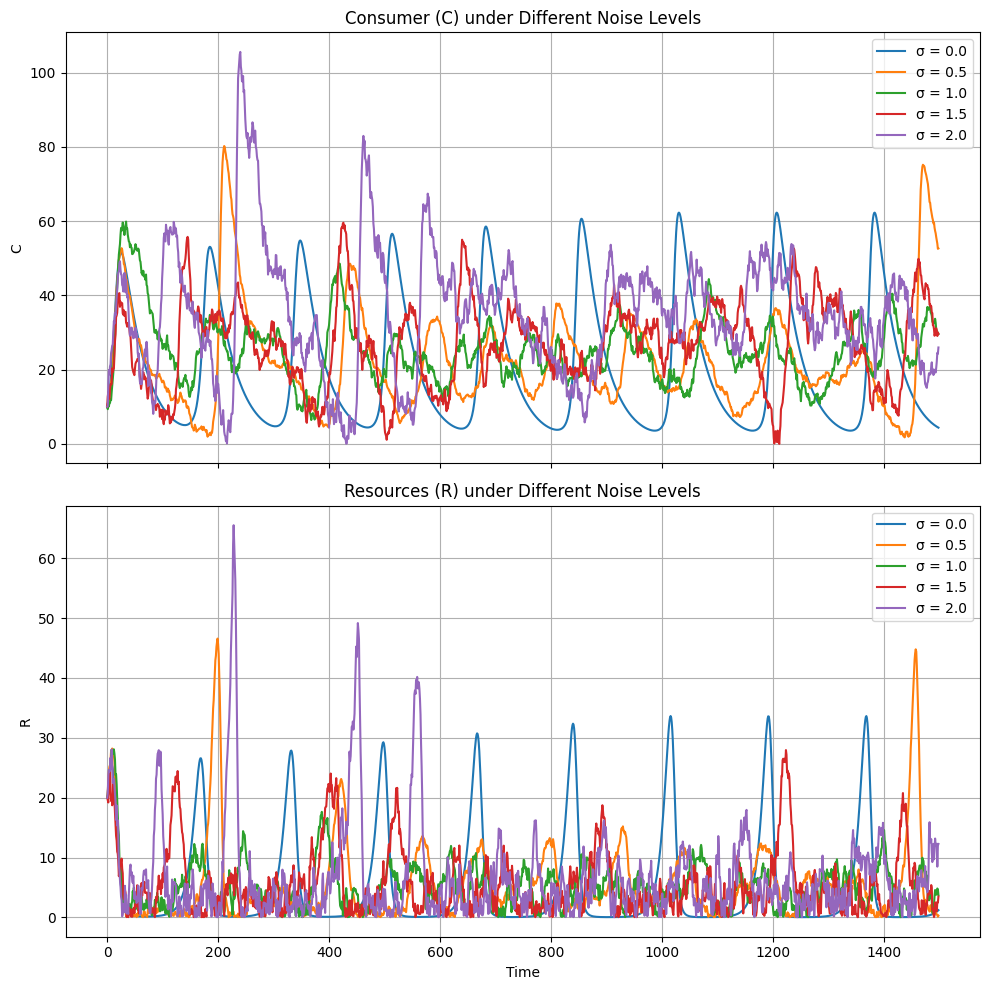}
        \put(-320,320){\textbf{(a)}}
    \end{subfigure}

    \vspace{0.0cm}

    % Second row
    \begin{subfigure}{0.8\textwidth}
        \includegraphics[width=0.9\linewidth]{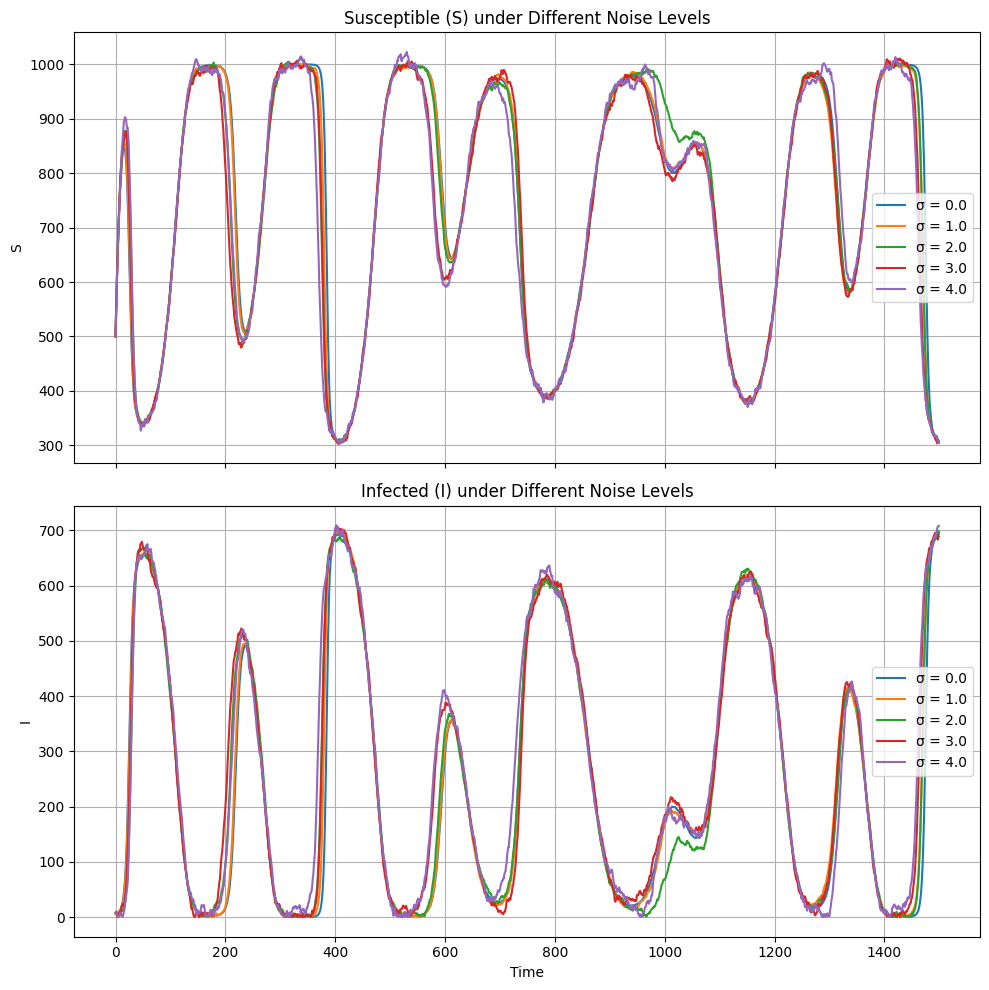}
        \put(-320,320){\textbf{(b)}}
    \end{subfigure}

    \caption{Simulations of a) CR model and b) SIR model.}
    \label{simulations}
\end{figure}

\begin{figure}[htbp]
    \centering
    % First row
    \begin{subfigure}{0.8\textwidth}
        \includegraphics[width=0.9\linewidth]{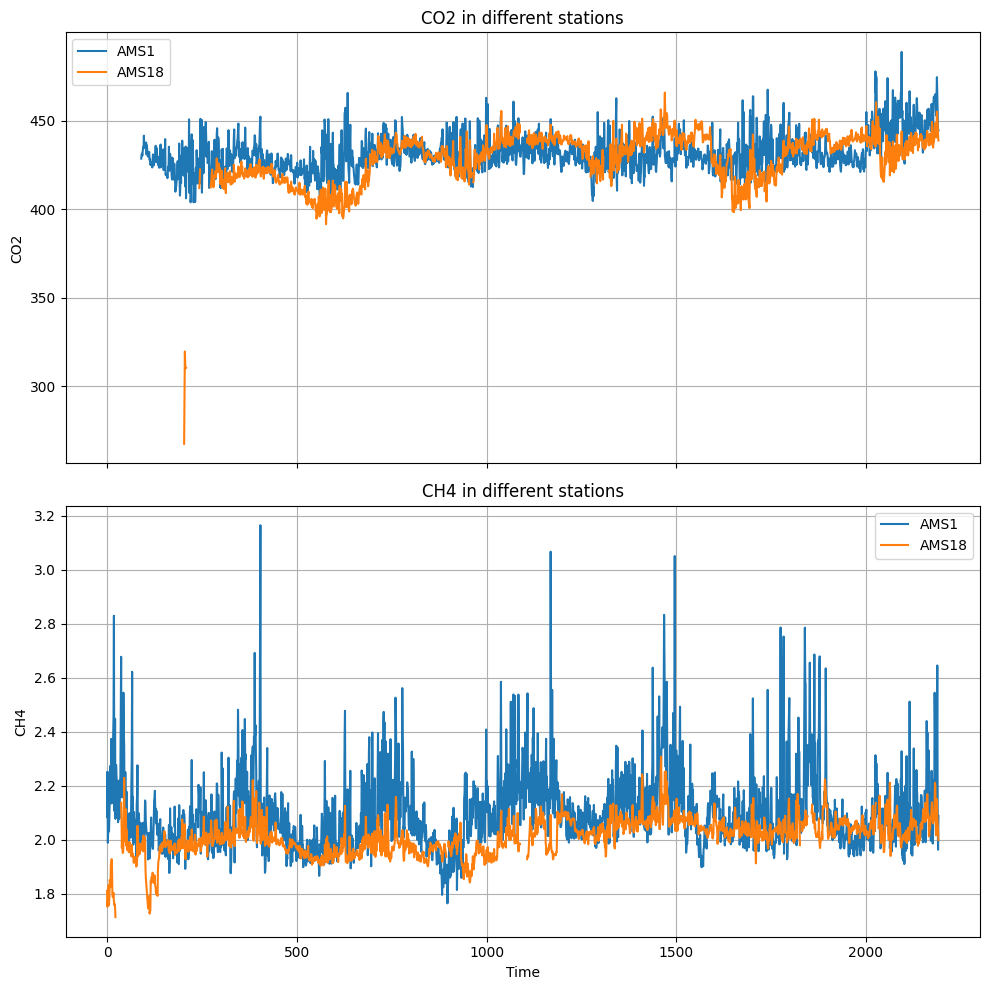}
        \put(-320,320){\textbf{(a)}}
    \end{subfigure}

    \vspace{0.0cm}

    % Second row
    \begin{subfigure}{0.8\textwidth}
        \includegraphics[width=0.9\linewidth]{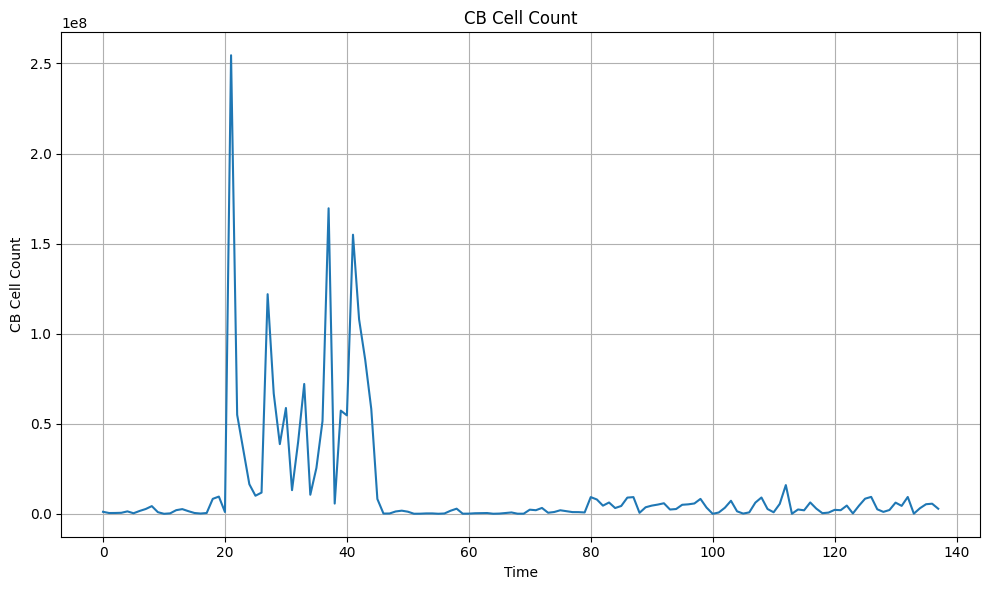}
        \put(-320,195){\textbf{(b)}}
    \end{subfigure}

    \caption{Time series of a) gases dataset and b) CB dataset.}
    \label{time-series}
\end{figure}

\begin{figure}[htbp]
\centering
\scriptsize

% ======== ROW 1: MAIN MODELS ========
\begin{subfigure}{0.24\textwidth}
    \includegraphics[width=\linewidth]{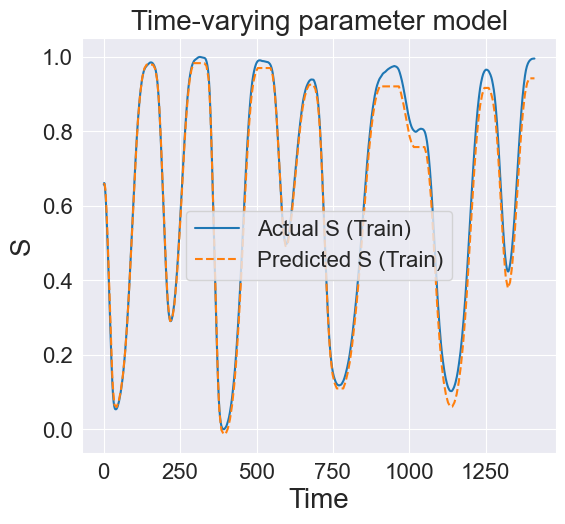}
\end{subfigure}\hfill
\begin{subfigure}{0.24\textwidth}
    \includegraphics[width=\linewidth]{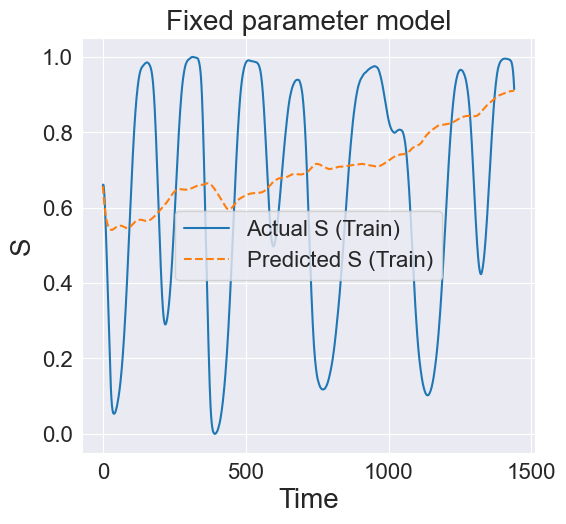}
\end{subfigure}\hfill
\begin{subfigure}{0.24\textwidth}
    \includegraphics[width=\linewidth]{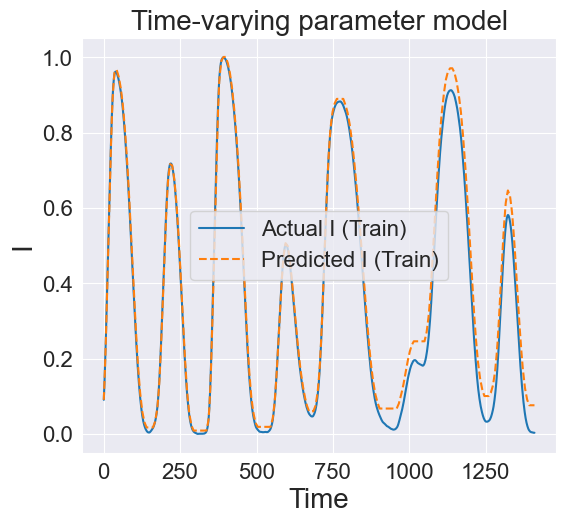}
\end{subfigure}
\begin{subfigure}{0.24\textwidth}
    \includegraphics[width=\linewidth]{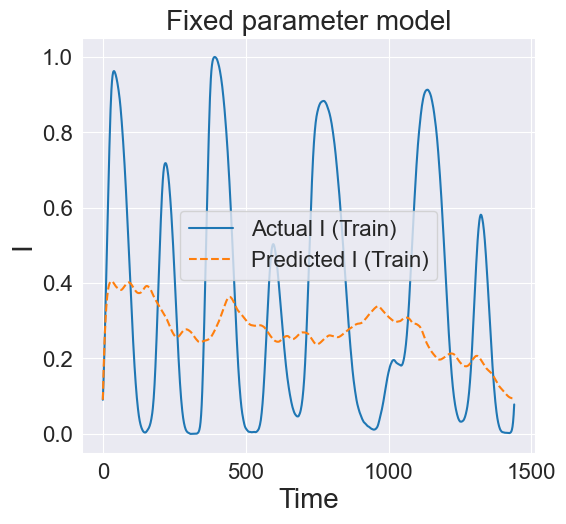}
\end{subfigure}

\vspace{0cm}

% ======== ROW 2: FIXED MODELS ========
\begin{subfigure}{0.24\textwidth}
    \includegraphics[width=\linewidth]{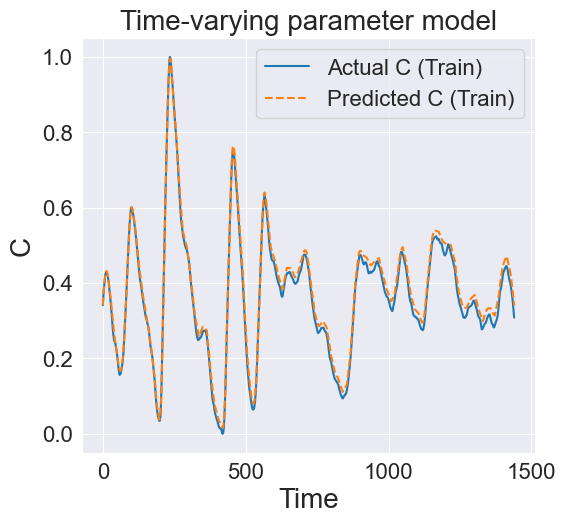}
\end{subfigure}\hfill
\begin{subfigure}{0.24\textwidth}
    \includegraphics[width=\linewidth]{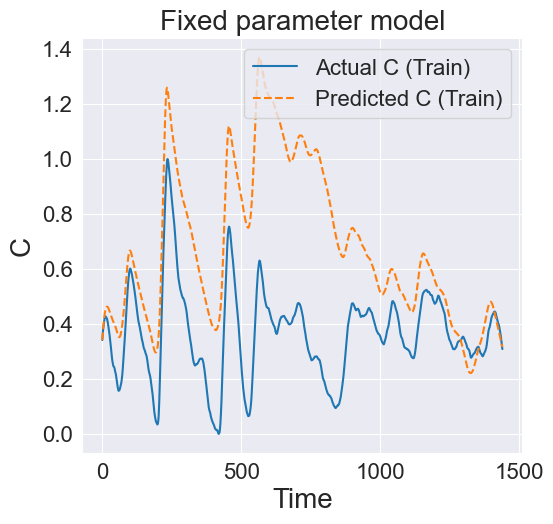}
\end{subfigure}\hfill
\begin{subfigure}{0.24\textwidth}
    \includegraphics[width=\linewidth]{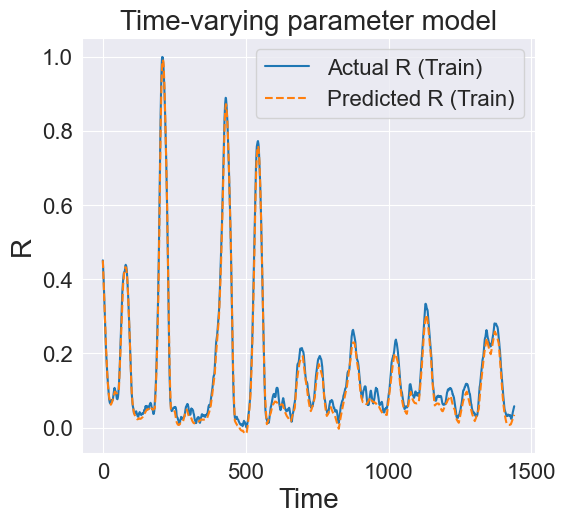}
\end{subfigure}
\begin{subfigure}{0.24\textwidth}
    \includegraphics[width=\linewidth]{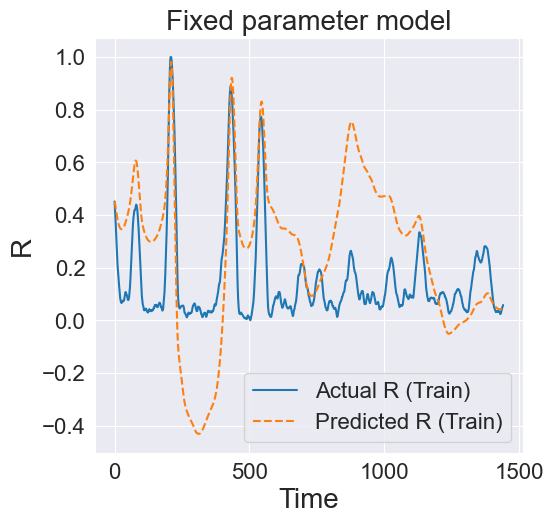}
\end{subfigure}

\vspace{0cm}

% ======== ROW 3: MAIN MODELS ========
\begin{subfigure}{0.24\textwidth}
    \includegraphics[width=\linewidth]{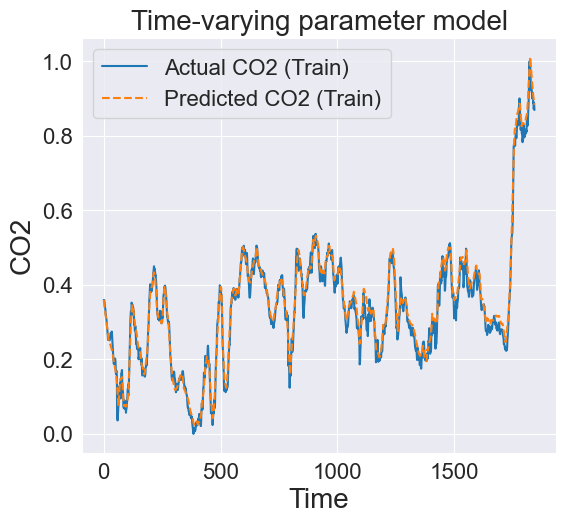}
\end{subfigure}\hfill
\begin{subfigure}{0.24\textwidth}
    \includegraphics[width=\linewidth]{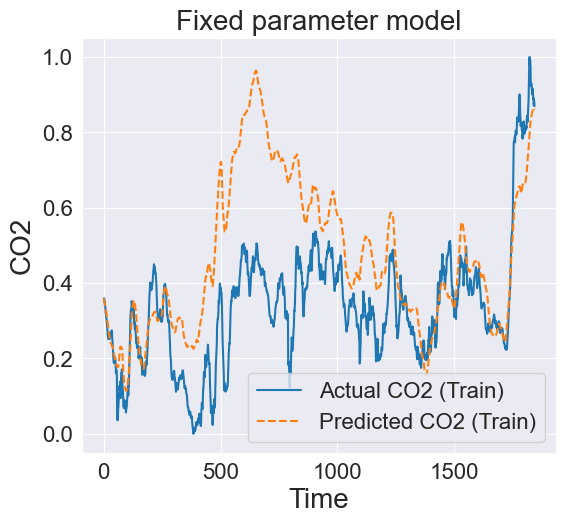}
\end{subfigure}\hfill
\begin{subfigure}{0.24\textwidth}
    \includegraphics[width=\linewidth]{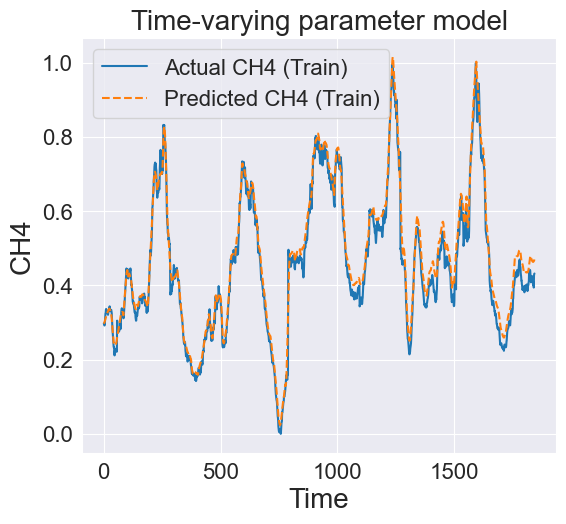}
\end{subfigure}
\begin{subfigure}{0.24\textwidth}
    \includegraphics[width=\linewidth]{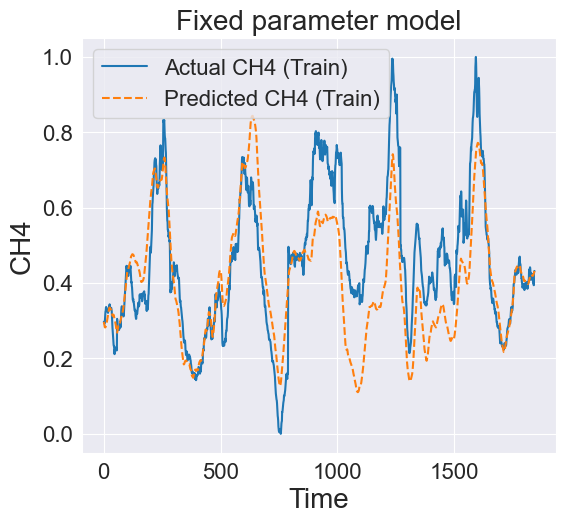}
\end{subfigure}

\vspace{0cm}

% ======== ROW 4: FIXED MODELS ========
\begin{subfigure}{0.24\textwidth}
    \includegraphics[width=\linewidth]{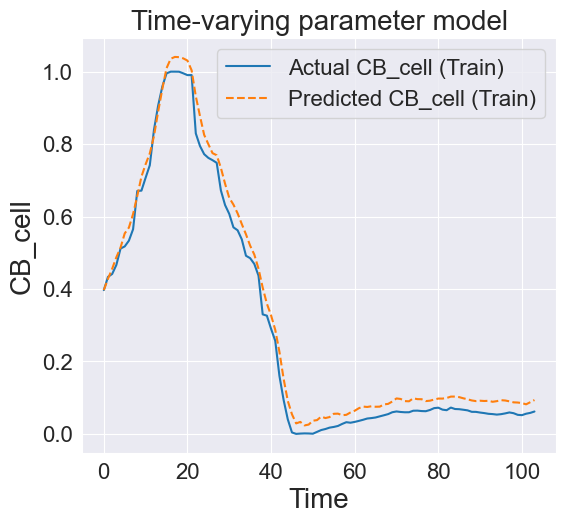}
\end{subfigure}\hfill
\begin{subfigure}{0.24\textwidth}
    \includegraphics[width=\linewidth]{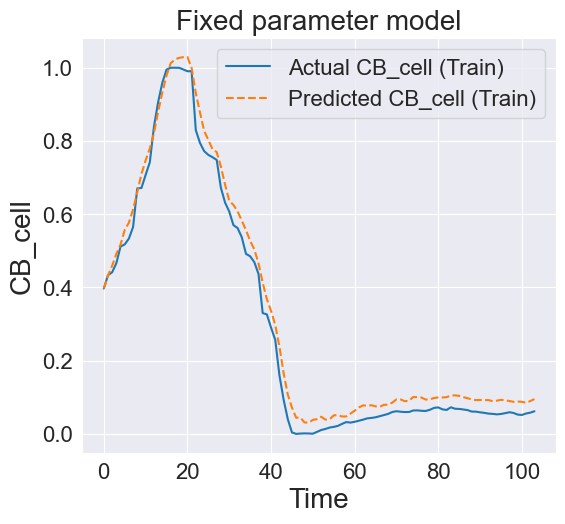}
\end{subfigure}

\caption{Learning comparison between the time-varying parameter and fixed parameter model across four datasets. First and third columns are learning made by the time-varying parameter model, while the second and fourth columns are learning made by the fixed parameter model.}
\label{sample_plot_learning}
\end{figure}

\begin{figure}[htbp]
\centering
\scriptsize

% ======== ROW 1: MAIN MODELS ========
\begin{subfigure}{0.24\textwidth}
    \includegraphics[width=\linewidth]{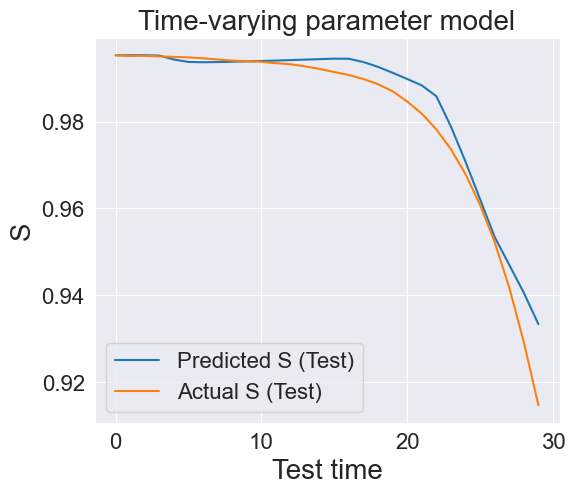}
\end{subfigure}\hfill
\begin{subfigure}{0.24\textwidth}
    \includegraphics[width=\linewidth]{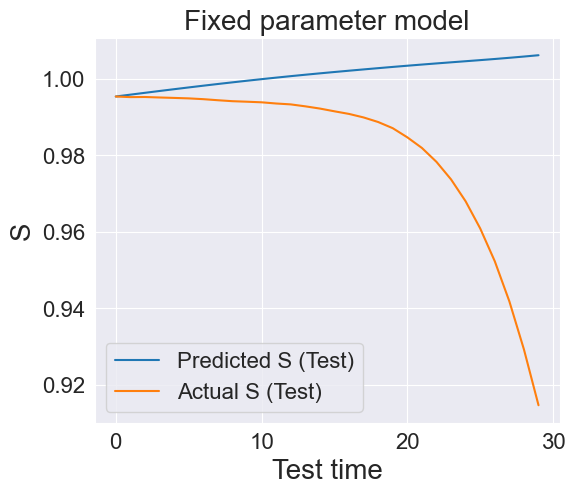}
\end{subfigure}\hfill
\begin{subfigure}{0.24\textwidth}
    \includegraphics[width=\linewidth]{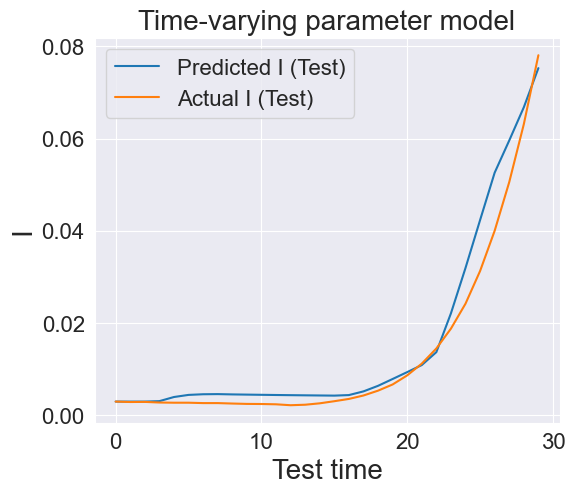}
\end{subfigure}
\begin{subfigure}{0.24\textwidth}
    \includegraphics[width=\linewidth]{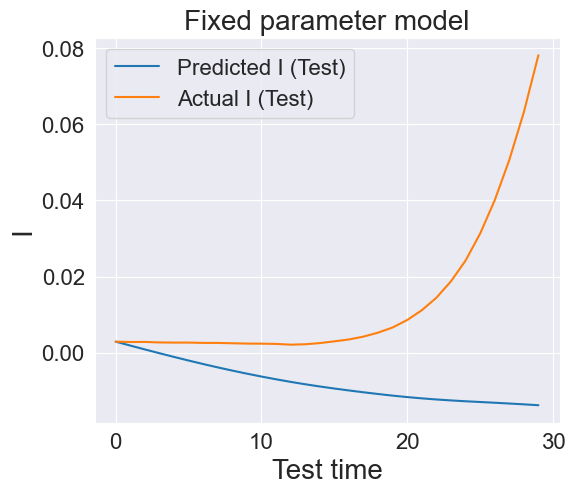}
\end{subfigure}

\vspace{0cm}

% ======== ROW 2: FIXED MODELS ========
\begin{subfigure}{0.24\textwidth}
    \includegraphics[width=\linewidth]{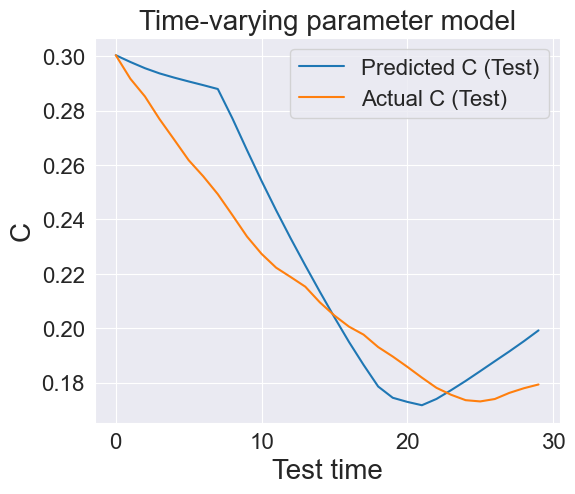}
\end{subfigure}\hfill
\begin{subfigure}{0.24\textwidth}
    \includegraphics[width=\linewidth]{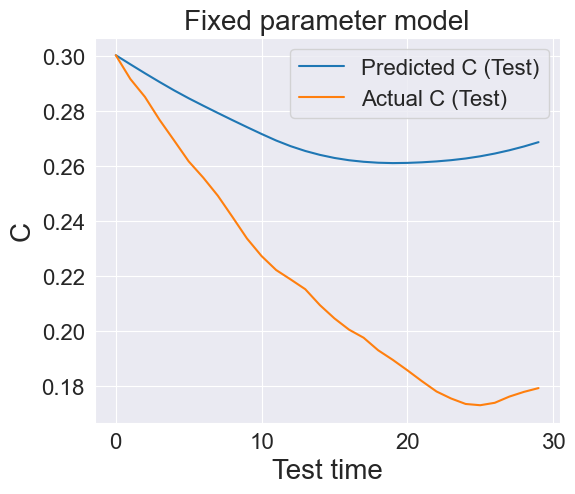}
\end{subfigure}\hfill
\begin{subfigure}{0.24\textwidth}
    \includegraphics[width=\linewidth]{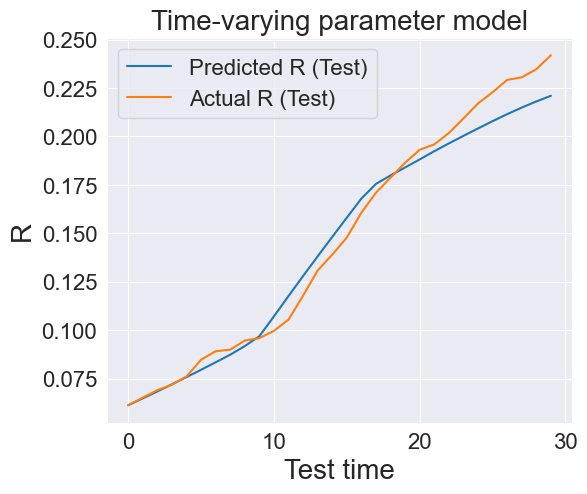}
\end{subfigure}
\begin{subfigure}{0.24\textwidth}
    \includegraphics[width=\linewidth]{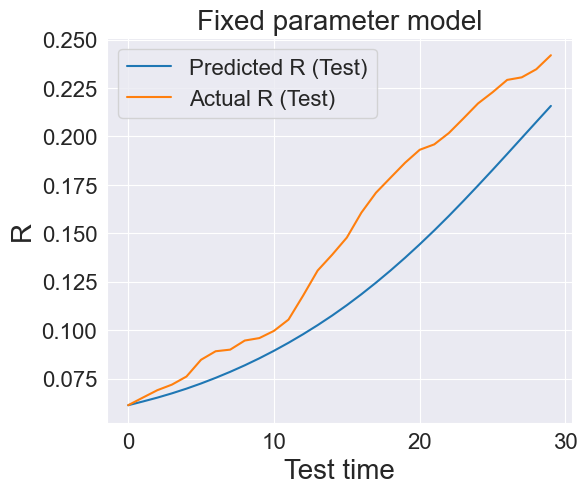}
\end{subfigure}

\vspace{0cm}

% ======== ROW 3: MAIN MODELS ========
\begin{subfigure}{0.24\textwidth}
    \includegraphics[width=\linewidth]{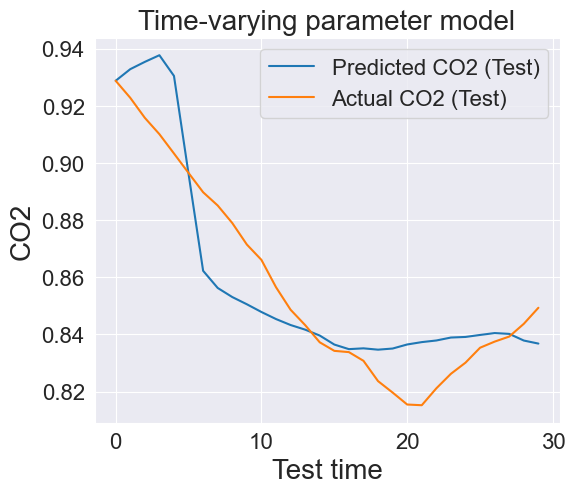}
\end{subfigure}\hfill
\begin{subfigure}{0.24\textwidth}
    \includegraphics[width=\linewidth]{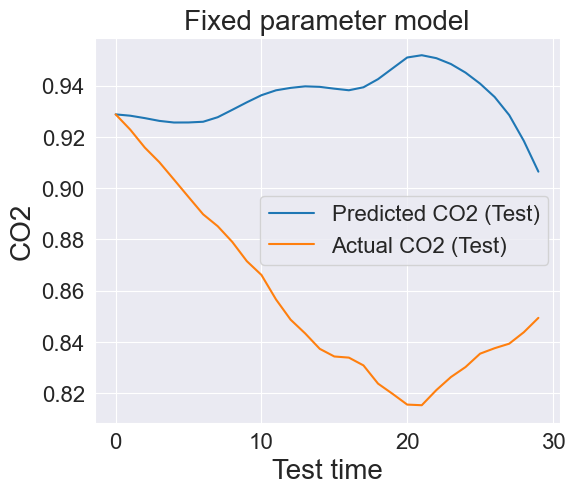}
\end{subfigure}\hfill
\begin{subfigure}{0.24\textwidth}
    \includegraphics[width=\linewidth]{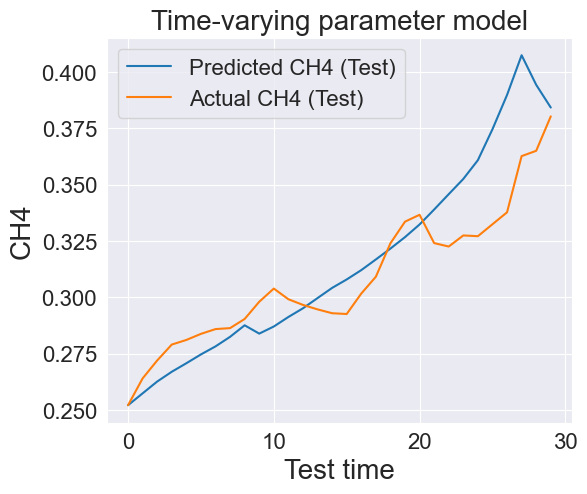}
\end{subfigure}
\begin{subfigure}{0.24\textwidth}
    \includegraphics[width=\linewidth]{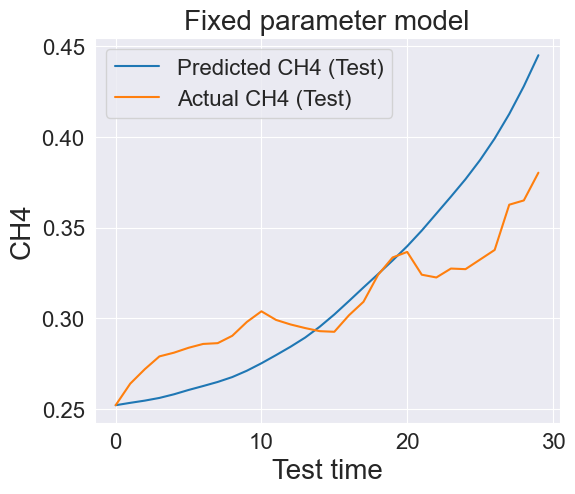}
\end{subfigure}

\vspace{0cm}

% ======== ROW 4: FIXED MODELS ========
\begin{subfigure}{0.24\textwidth}
    \includegraphics[width=\linewidth]{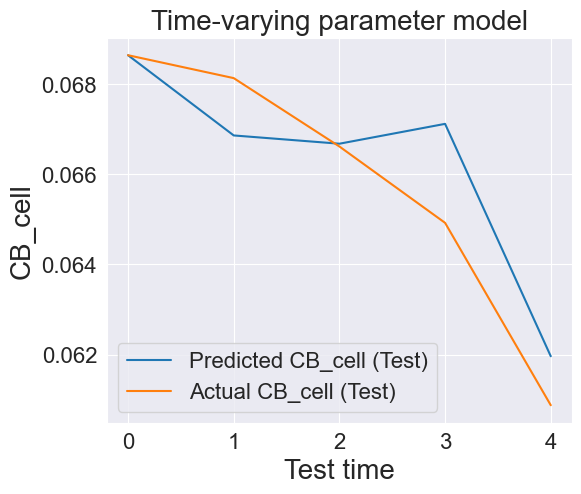}
\end{subfigure}\hfill
\begin{subfigure}{0.24\textwidth}
    \includegraphics[width=\linewidth]{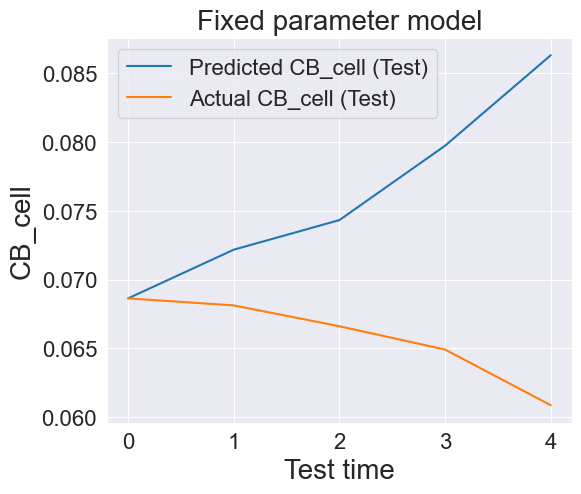}
\end{subfigure}

\caption{ Forecasting comparison between the time-varying parameter and fixed parameter model across four datasets. First and third columns are predictions made by the time-varying parameter model, while the second and fourth columns are predictions made by the fixed parameter model.}
\label{sample_plot_forecasting}
\end{figure}

\begin{table}[htb!]
\scriptsize
\caption{Number of time-varying parameters, interval lengths, and MAEs across different test folds using the expanding-window cross-validation (CV) method and the optimal configuration (OC) for each noise setting in the CR model simulations.}
\label{tab:CR}
\begin{tabular}{|m{1cm}|m{1cm}|m{1cm}|m{1cm}|m{1cm}|m{1cm}|m{1cm}|m{1cm}|m{1cm}|m{1cm}|}
\hline
\multicolumn{1}{|l|}{\textbf{Dataset}} & \multicolumn{1}{l|}{\textbf{Noise}} & \textbf{Variable}           & \textbf{Fold} & \textbf{CV: Interval length} & \textbf{CV: Num. of time varying params.} & \textbf{CV: MAE test} & \textbf{OC: Interval length} & \textbf{OC: Num. of time varying params.} & \textbf{OC: MAE test} \\ \hline
\multirow{50}{*}{\textbf{CR}}          & \multirow{10}{*}{\textbf{0.0}}      & \multirow{5}{*}{\textbf{C}} & 1             & 28                       & 1                                   & 1                 & 21                       & 2                                   & 0.29              \\ \cline{4-10} 
                                       &                                     &                             & 2             & 28                       & 1                                   & 2.02              & 14                       & 2                                   & 0.59              \\ \cline{4-10} 
                                       &                                     &                             & 3             & 28                       & 1                                   & 0.93              & 7                        & 5                                   & 0.46              \\ \cline{4-10} 
                                       &                                     &                             & 4             & 28                       & 1                                   & 0.1               & 7                        & 1                                   & 0.07              \\ \cline{4-10} 
                                       &                                     &                             & 5             & 28                       & 1                                   & 0.21              & 21                       & 1                                   & 0.08              \\ \cline{3-10} 
                                       &                                     & \multirow{5}{*}{\textbf{R}} & 1             & 28                       & 1                                   & 5.22              & 28                       & 2                                   & 1.72              \\ \cline{4-10} 
                                       &                                     &                             & 2             & 28                       & 1                                   & 3.79              & 21                       & 2                                   & 2.16              \\ \cline{4-10} 
                                       &                                     &                             & 3             & 28                       & 1                                   & 6                 & 14                       & 4                                   & 1.07              \\ \cline{4-10} 
                                       &                                     &                             & 4             & 28                       & 1                                   & 1.56              & 7                        & 1                                   & 0.16              \\ \cline{4-10} 
                                       &                                     &                             & 5             & 28                       & 1                                   & 1.62              & 14                       & 1                                   & 0.16              \\ \cline{2-10} 
                                       & \multirow{10}{*}{\textbf{0.5}}      & \multirow{5}{*}{\textbf{C}} & 1             & 21                       & 1                                   & 2.75              & 28                       & 1                                   & 2.42              \\ \cline{4-10} 
                                       &                                     &                             & 2             & 21                       & 1                                   & 5.54              & 14                       & 2                                   & 3.86              \\ \cline{4-10} 
                                       &                                     &                             & 3             & 21                       & 1                                   & 1.39              & 21                       & 6                                   & 0.72              \\ \cline{4-10} 
                                       &                                     &                             & 4             & 21                       & 1                                   & 2.97              & 28                       & 1                                   & 2.77              \\ \cline{4-10} 
                                       &                                     &                             & 5             & 14                       & 1                                   & 4.06              & 28                       & 2                                   & 2.85              \\ \cline{3-10} 
                                       &                                     & \multirow{5}{*}{\textbf{R}} & 1             & 21                       & 1                                   & 1.55              & 28                       & 1                                   & 1.45              \\ \cline{4-10} 
                                       &                                     &                             & 2             & 21                       & 1                                   & 5.53              & 14                       & 4                                   & 4.01              \\ \cline{4-10} 
                                       &                                     &                             & 3             & 21                       & 1                                   & 1.85              & 28                       & 4                                   & 0.72              \\ \cline{4-10} 
                                       &                                     &                             & 4             & 21                       & 1                                   & 15.28             & 7                        & 2                                   & 8.01              \\ \cline{4-10} 
                                       &                                     &                             & 5             & 14                       & 1                                   & 4.91              & 7                        & 2                                   & 3.63              \\ \cline{2-10} 
                                       & \multirow{10}{*}{\textbf{1.0}}      & \multirow{5}{*}{\textbf{C}} & 1             & 28                       & 2                                   & 3.1               & 28                       & 6                                   & 1.55              \\ \cline{4-10} 
                                       &                                     &                             & 2             & 28                       & 2                                   & 4.48              & 21                       & 5                                   & 2.9               \\ \cline{4-10} 
                                       &                                     &                             & 3             & 28                       & 2                                   & 2.38              & 21                       & 2                                   & 1.54              \\ \cline{4-10} 
                                       &                                     &                             & 4             & 28                       & 2                                   & 1.06              & 14                       & 6                                   & 0.55              \\ \cline{4-10} 
                                       &                                     &                             & 5             & 28                       & 2                                   & 19.81             & 14                       & 4                                   & 0.95              \\ \cline{3-10} 
                                       &                                     & \multirow{5}{*}{\textbf{R}} & 1             & 28                       & 2                                   & 2.77              & 28                       & 1                                   & 2.69              \\ \cline{4-10} 
                                       &                                     &                             & 2             & 28                       & 2                                   & 16.36             & 21                       & 6                                   & 10.66             \\ \cline{4-10} 
                                       &                                     &                             & 3             & 28                       & 2                                   & 2.3               & 28                       & 2                                   & 2.3               \\ \cline{4-10} 
                                       &                                     &                             & 4             & 28                       & 2                                   & 3.5               & 28                       & 6                                   & 1.23              \\ \cline{4-10} 
                                       &                                     &                             & 5             & 28                       & 2                                   & 14.45             & 21                       & 4                                   & 3.06              \\ \cline{2-10} 
                                       & \multirow{10}{*}{\textbf{1.5}}      & \multirow{5}{*}{\textbf{C}} & 1             & 14                       & 1                                   & 3.94              & 7                        & 5                                   & 1.18              \\ \cline{4-10} 
                                       &                                     &                             & 2             & 14                       & 1                                   & 3.46              & 21                       & 3                                   & 1.04              \\ \cline{4-10} 
                                       &                                     &                             & 3             & 14                       & 1                                   & 9.69              & 7                        & 6                                   & 1.45              \\ \cline{4-10} 
                                       &                                     &                             & 4             & 14                       & 1                                   & 6.17              & 28                       & 2                                   & 2.69              \\ \cline{4-10} 
                                       &                                     &                             & 5             & 28                       & 1                                   & 5.04              & 14                       & 3                                   & 1.08              \\ \cline{3-10} 
                                       &                                     & \multirow{5}{*}{\textbf{R}} & 1             & 14                       & 1                                   & 2.84              & 21                       & 4                                   & 1.54              \\ \cline{4-10} 
                                       &                                     &                             & 2             & 14                       & 1                                   & 4.03              & 14                       & 6                                   & 1.62              \\ \cline{4-10} 
                                       &                                     &                             & 3             & 14                       & 1                                   & 5.8               & 21                       & 6                                   & 5.39              \\ \cline{4-10} 
                                       &                                     &                             & 4             & 14                       & 1                                   & 8.27              & 14                       & 6                                   & 3.62              \\ \cline{4-10} 
                                       &                                     &                             & 5             & 28                       & 1                                   & 3.35              & 28                       & 3                                   & 1.21              \\ \cline{2-10} 
                                       & \multirow{10}{*}{\textbf{2.0}}      & \multirow{5}{*}{\textbf{C}} & 1             & 14                       & 2                                   & 1.34              & 21                       & 1                                   & 0.94              \\ \cline{4-10} 
                                       &                                     &                             & 2             & 14                       & 2                                   & 4.72              & 7                        & 5                                   & 0.62              \\ \cline{4-10} 
                                       &                                     &                             & 3             & 14                       & 2                                   & 3.87              & 28                       & 2                                   & 0.59              \\ \cline{4-10} 
                                       &                                     &                             & 4             & 14                       & 2                                   & 1.42              & 7                        & 1                                   & 0.83              \\ \cline{4-10} 
                                       &                                     &                             & 5             & 14                       & 2                                   & 5.71              & 14                       & 3                                   & 1.45              \\ \cline{3-10} 
                                       &                                     & \multirow{5}{*}{\textbf{R}} & 1             & 14                       & 2                                   & 6.57              & 21                       & 6                                   & 2.97              \\ \cline{4-10} 
                                       &                                     &                             & 2             & 14                       & 2                                   & 2.86              & 14                       & 5                                   & 1.18              \\ \cline{4-10} 
                                       &                                     &                             & 3             & 14                       & 2                                   & 10.22             & 14                       & 3                                   & 6.43              \\ \cline{4-10} 
                                       &                                     &                             & 4             & 14                       & 2                                   & 1.07              & 28                       & 5                                   & 0.43              \\ \cline{4-10} 
                                       &                                     &                             & 5             & 14                       & 2                                   & 4                 & 14                       & 3                                   & 0.74              \\ \hline
\end{tabular}
\end{table}

% Please add the following required packages to your document preamble:
% \usepackage{multirow}
\begin{table}[htb!]
\scriptsize
\caption{Number of time-varying parameters, interval lengths, and MAEs across different test folds using the expanding-window cross-validation (CV) method and the optimal configuration (OC) for each noise setting in the SIR model simulations.}
\label{tab:SIR}
\begin{tabular}{|m{1cm}|m{1cm}|m{1cm}|m{1cm}|m{1cm}|m{1cm}|m{1cm}|m{1cm}|m{1cm}|m{1cm}|}
\hline
\textbf{Dataset}              & \textbf{Noise}                 & \textbf{Variable}           & \textbf{Fold} & \textbf{CV: Interval length} & \textbf{CV: Num. of time varying params.} & \textbf{CV: MAE test} & \textbf{OC: Interval length} & \textbf{OC: Num. of time varying params.} & \textbf{OC: MAE test} \\ \hline
\multirow{50}{*}{\textbf{SIR}} & \multirow{10}{*}{\textbf{0.0}} & \multirow{5}{*}{\textbf{I}} & 1             & 14                       & 6                                   & 7.81              & 28                       & 5                                   & 3.42              \\ \cline{4-10} 
                              &                                &                             & 2             & 14                       & 6                                   & 16.78             & 28                       & 6                                   & 1.67              \\ \cline{4-10} 
                              &                                &                             & 3             & 28                       & 6                                   & 11.43             & 28                       & 1                                   & 1.78              \\ \cline{4-10} 
                              &                                &                             & 4             & 28                       & 6                                   & 0.66              & 28                       & 6                                   & 0.66              \\ \cline{4-10} 
                              &                                &                             & 5             & 28                       & 6                                   & 6.49              & 28                       & 6                                   & 6.49              \\ \cline{3-10} 
                              &                                & \multirow{5}{*}{\textbf{S}} & 1             & 14                       & 6                                   & 8.66              & 28                       & 6                                   & 3.61              \\ \cline{4-10} 
                              &                                &                             & 2             & 14                       & 6                                   & 17.54             & 28                       & 6                                   & 1.26              \\ \cline{4-10} 
                              &                                &                             & 3             & 28                       & 6                                   & 11                & 28                       & 1                                   & 1.8               \\ \cline{4-10} 
                              &                                &                             & 4             & 28                       & 6                                   & 1.41              & 7                        & 1                                   & 0.74              \\ \cline{4-10} 
                              &                                &                             & 5             & 28                       & 6                                   & 8.47              & 21                       & 4                                   & 4.03              \\ \cline{2-10} 
                              & \multirow{10}{*}{\textbf{1.0}} & \multirow{5}{*}{\textbf{I}} & 1             & 21                       & 6                                   & 7.83              & 28                       & 5                                   & 3.2               \\ \cline{4-10} 
                              &                                &                             & 2             & 21                       & 6                                   & 10.75             & 28                       & 6                                   & 1.57              \\ \cline{4-10} 
                              &                                &                             & 3             & 28                       & 6                                   & 13.98             & 14                       & 6                                   & 2.24              \\ \cline{4-10} 
                              &                                &                             & 4             & 21                       & 6                                   & 25.82             & 21                       & 1                                   & 0.48              \\ \cline{4-10} 
                              &                                &                             & 5             & 28                       & 6                                   & 5.08              & 14                       & 5                                   & 1.08              \\ \cline{3-10} 
                              &                                & \multirow{5}{*}{\textbf{S}} & 1             & 21                       & 6                                   & 8.08              & 14                       & 2                                   & 1.75              \\ \cline{4-10} 
                              &                                &                             & 2             & 21                       & 6                                   & 11.88             & 28                       & 6                                   & 0.73              \\ \cline{4-10} 
                              &                                &                             & 3             & 28                       & 6                                   & 13.2              & 28                       & 1                                   & 2.94              \\ \cline{4-10} 
                              &                                &                             & 4             & 21                       & 6                                   & 25.76             & 21                       & 1                                   & 0.47              \\ \cline{4-10} 
                              &                                &                             & 5             & 28                       & 6                                   & 3.69              & 21                       & 6                                   & 3.62              \\ \cline{2-10} 
                              & \multirow{10}{*}{\textbf{2.0}} & \multirow{5}{*}{\textbf{I}} & 1             & 21                       & 6                                   & 11.43             & 14                       & 3                                   & 4.16              \\ \cline{4-10} 
                              &                                &                             & 2             & 14                       & 3                                   & 17.39             & 28                       & 5                                   & 4.06              \\ \cline{4-10} 
                              &                                &                             & 3             & 14                       & 3                                   & 9.05              & 28                       & 1                                   & 1.28              \\ \cline{4-10} 
                              &                                &                             & 4             & 14                       & 3                                   & 2.45              & 7                        & 6                                   & 0.26              \\ \cline{4-10} 
                              &                                &                             & 5             & 14                       & 3                                   & 9.33              & 21                       & 5                                   & 1.42              \\ \cline{3-10} 
                              &                                & \multirow{5}{*}{\textbf{S}} & 1             & 21                       & 6                                   & 10.75             & 7                        & 2                                   & 3.16              \\ \cline{4-10} 
                              &                                &                             & 2             & 14                       & 3                                   & 13.36             & 21                       & 5                                   & 3.11              \\ \cline{4-10} 
                              &                                &                             & 3             & 14                       & 3                                   & 2.02              & 28                       & 2                                   & 0.79              \\ \cline{4-10} 
                              &                                &                             & 4             & 14                       & 3                                   & 14.36             & 7                        & 6                                   & 0.31              \\ \cline{4-10} 
                              &                                &                             & 5             & 14                       & 3                                   & 10.89             & 21                       & 6                                   & 4.29              \\ \cline{2-10} 
                              & \multirow{10}{*}{\textbf{3.0}} & \multirow{5}{*}{\textbf{I}} & 1             & 21                       & 6                                   & 12.46             & 14                       & 4                                   & 2.87              \\ \cline{4-10} 
                              &                                &                             & 2             & 21                       & 6                                   & 12.66             & 28                       & 4                                   & 1.13              \\ \cline{4-10} 
                              &                                &                             & 3             & 28                       & 6                                   & 7.75              & 28                       & 1                                   & 1.31              \\ \cline{4-10} 
                              &                                &                             & 4             & 28                       & 6                                   & 2.62              & 28                       & 5                                   & 1.13              \\ \cline{4-10} 
                              &                                &                             & 5             & 28                       & 6                                   & 8.35              & 7                        & 6                                   & 4.47              \\ \cline{3-10} 
                              &                                & \multirow{5}{*}{\textbf{S}} & 1             & 21                       & 6                                   & 10.3              & 28                       & 4                                   & 1.74              \\ \cline{4-10} 
                              &                                &                             & 2             & 21                       & 6                                   & 12.48             & 14                       & 4                                   & 0.57              \\ \cline{4-10} 
                              &                                &                             & 3             & 28                       & 6                                   & 8                 & 28                       & 1                                   & 1.88              \\ \cline{4-10} 
                              &                                &                             & 4             & 28                       & 6                                   & 2.48              & 14                       & 6                                   & 0.7               \\ \cline{4-10} 
                              &                                &                             & 5             & 28                       & 6                                   & 16.5              & 28                       & 4                                   & 5.17              \\ \cline{2-10} 
                              & \multirow{10}{*}{\textbf{4.0}} & \multirow{5}{*}{\textbf{I}} & 1             & 14                       & 1                                   & 10.01             & 14                       & 3                                   & 3.05              \\ \cline{4-10} 
                              &                                &                             & 2             & 21                       & 6                                   & 9.14              & 28                       & 5                                   & 8.06              \\ \cline{4-10} 
                              &                                &                             & 3             & 21                       & 6                                   & 6.09              & 28                       & 1                                   & 2.71              \\ \cline{4-10} 
                              &                                &                             & 4             & 21                       & 6                                   & 6.67              & 28                       & 2                                   & 1.32              \\ \cline{4-10} 
                              &                                &                             & 5             & 21                       & 6                                   & 14.84             & 14                       & 6                                   & 2.46              \\ \cline{3-10} 
                              &                                & \multirow{5}{*}{\textbf{S}} & 1             & 14                       & 1                                   & 10.13             & 21                       & 2                                   & 2.48              \\ \cline{4-10} 
                              &                                &                             & 2             & 21                       & 6                                   & 5.88              & 21                       & 3                                   & 1.19              \\ \cline{4-10} 
                              &                                &                             & 3             & 21                       & 6                                   & 3.8               & 21                       & 5                                   & 3.07              \\ \cline{4-10} 
                              &                                &                             & 4             & 21                       & 6                                   & 6.15              & 21                       & 5                                   & 2.38              \\ \cline{4-10} 
                              &                                &                             & 5             & 21                       & 6                                   & 13.35             & 14                       & 5                                   & 3.97              \\ \hline
\end{tabular}
\end{table}

% Please add the following required packages to your document preamble:
% \usepackage{multirow}
\begin{table}[htb!]
\scriptsize
\caption{Number of time-varying parameters, interval lengths, and MAEs across different test folds using the expanding-window cross-validation (CV) method and the optimal configuration (OC) for each station in the gases dataset.}
\label{tab:gases}
\begin{tabular}{|m{1cm}|m{1cm}|m{1cm}|m{1cm}|m{1cm}|m{1cm}|m{1cm}|m{1cm}|m{1cm}|m{1cm}|}
\hline
\textbf{Dataset}              & \textbf{Station}                 & \textbf{Variable}           & \textbf{Fold} & \textbf{CV: Interval length} & \textbf{CV: Num. of time varying params.} & \textbf{CV: MAE test} & \textbf{OC: Interval length} & \textbf{OC: Num. of time varying params.} & \textbf{OC: MAE test} \\ \hline
\multirow{20}{*}{gases}       & \multirow{10}{*}{AMS1}       & \multirow{5}{*}{CH\textsubscript{4}} & 1    & 7               & 3                          & 5.68     & 21              & 4                          & 1.38     \\ \cline{4-10} 
                              &                              &                      & 2    & 7               & 3                          & 37       & 14              & 3                          & 1.52     \\ \cline{4-10} 
                              &                              &                      & 3    & 14              & 3                          & 5.07     & 14              & 15                         & 1.13     \\ \cline{4-10} 
                              &                              &                      & 4    & 14              & 3                          & 10.58    & 7               & 36                         & 1.02     \\ \cline{4-10} 
                              &                              &                      & 5    & 14              & 3                          & 4.64     & 14              & 81                         & 1.27     \\ \cline{3-10} 
                              &                              & \multirow{5}{*}{CO\textsubscript{2}} & 1    & 7               & 3                          & 12.56    & 7               & 28                         & 3.02     \\ \cline{4-10} 
                              &                              &                      & 2    & 7               & 3                          & 10.91    & 14              & 8                          & 2.61     \\ \cline{4-10} 
                              &                              &                      & 3    & 14              & 3                          & 7.99     & 7               & 3                          & 2.48     \\ \cline{4-10} 
                              &                              &                      & 4    & 14              & 3                          & 6.47     & 14              & 5                          & 2.89     \\ \cline{4-10} 
                              &                              &                      & 5    & 14              & 3                          & 12.54    & 21              & 23                         & 2.46     \\ \cline{2-10} 
                              & \multirow{10}{*}{AMS18}      & \multirow{5}{*}{CH\textsubscript{4}} & 1    & 7               & 96                         & 24.6     & 7               & 81                         & 23.78    \\ \cline{4-10} 
                              &                              &                      & 2    & 7               & 99                         & 4.94     & 7               & 5                          & 4.51     \\ \cline{4-10} 
                              &                              &                      & 3    & 7               & 99                         & 15.28    & 14              & 95                         & 14.08    \\ \cline{4-10} 
                              &                              &                      & 4    & 7               & 99                         & 6.56     & 28              & 2                          & 4.61     \\ \cline{4-10} 
                              &                              &                      & 5    & 7               & 99                         & 5.57     & 7               & 75                         & 4        \\ \cline{3-10} 
                              &                              & \multirow{5}{*}{CO\textsubscript{2}} & 1    & 7               & 96                         & 11.16    & 7               & 3                          & 2.65     \\ \cline{4-10} 
                              &                              &                      & 2    & 7               & 99                         & 10.9     & 7               & 7                          & 1.63     \\ \cline{4-10} 
                              &                              &                      & 3    & 7               & 99                         & 5.75     & 14              & 12                         & 3.43     \\ \cline{4-10} 
                              &                              &                      & 4    & 7               & 99                         & 12.2     & 7               & 6                          & 3.05     \\ \cline{4-10} 
                              &                              &                      & 5    & 7               & 99                         & 13.86    & 7               & 40                         & 1.47     \\ \hline
\end{tabular}
\end{table}

% Please add the following required packages to your document preamble:
% \usepackage{multirow}
\begin{table}[htb!]
\scriptsize
\caption{Number of time-varying parameters, interval lengths, and MAEs across different test folds using the expanding-window cross-validation (CV) method and the optimal configuration (OC) for each station in the Cyanobacteria dataset.}
\label{tab:CB}
\begin{tabular}{|m{1cm}|m{1cm}|m{1cm}|m{1cm}|m{1cm}|m{1cm}|m{1cm}|m{1cm}|m{1cm}|m{1cm}|}
\hline
\textbf{Dataset}              & \textbf{Region}                 & \textbf{Variable}           & \textbf{Fold} & \textbf{CV: Interval length} & \textbf{CV: Num. of time varying params.} & \textbf{CV: MAE test} & \textbf{OC: Interval length} & \textbf{OC: Num. of time varying params.} & \textbf{OC: MAE test} \\ \hline
\multirow{5}{*}{CB}           & \multirow{5}{*}{AB}        & \multirow{5}{*}{CB\textsubscript{cell}} & 1    & 20              & 103                        & 2.26     & 15              & 82                         & 0.07     \\ \cline{4-10} 
                              &                              &                      & 2    & 15              & 55                         & 0.91     & 25              & 65                         & 0.29     \\ \cline{4-10} 
                              &                              &                      & 3    & 25              & 64                         & 5.46     & 25              & 106                        & 1.74     \\ \cline{4-10} 
                              &                              &                      & 4    & 25              & 64                         & 3.5      & 15              & 58                         & 0.26     \\ \cline{4-10} 
                              &                              &                      & 5    & 25              & 60                         & 2.88     & 10              & 69                         & 0.15     \\ \hline
\end{tabular}
\end{table}

\end{document}